\documentclass[11pt]{article}
\usepackage{amsmath,amssymb,amsfonts}

\usepackage{natbib}
 \bibpunct[, ]{(}{)}{,}{a}{}{,}%

\usepackage{rotating}
\usepackage{fancyvrb}
\usepackage{url}
\usepackage{hyperref}

\usepackage{amsthm}
\usepackage{mathtools}
\usepackage{algorithm}
\usepackage{algpseudocode}
\usepackage{tikz}
\usetikzlibrary{positioning, arrows.meta, patterns}
\usetikzlibrary{matrix, positioning, arrows.meta, calc, shapes.geometric, backgrounds, fit}
\usepackage{adjustbox}

\usepackage{caption}
\usepackage{array}
\usepackage{pgfplots}
\usepackage{float}
\pgfplotsset{compat=newest}
\usepackage{graphicx}
\usepackage{subcaption}
\usepackage{stmaryrd}

\newtheorem{lemma}{Lemma}
\newtheorem{theorem}{Theorem}
\newtheorem{corollary}{Corollary}

\newtheorem{proposition}{Proposition}

\newtheorem{remark}{Remark}

\DeclareMathOperator*{\argmax}{arg\,max}

\newcommand{\E}{\mathbb{E}}

\newcommand{\Prob}{\mathbb{P}}

\newcommand{\Var}{\operatorname{Var}}

\newcommand{\thetahat}[1]{\widehat{\theta}_{\text{#1}}}

\begin{document}
\begin{center}
    \Large \bf Design Experiments to Compare Multi-armed Bandit Algorithms
\end{center}
\begin{center}
    {Huiling Meng}\,\footnote{Department of Systems Engineering and Engineering Management, The Chinese University of Hong Kong, Hong Kong, China. Email: \url{hmeng@se.cuhk.edu.hk.} },
    {Ningyuan Chen}\,\footnote{Rotman School of Management, University of Toronto, Toronto, Canada. Email: \url{ningyuan.chen@utoronto.ca.}},
    {Xuefeng Gao}\,\footnote{Department of Systems Engineering and Engineering Management, The Chinese University of Hong Kong, Hong Kong, China. Email: \url{xfgao@se.cuhk.edu.hk.} }
\end{center}

	\begin{center}
		\today
	\end{center}

\begin{abstract}
Online platforms routinely compare multi-armed bandit algorithms, such as UCB and Thompson Sampling, to select the best-performing policy.  Unlike standard A/B tests for static treatments, each run of a bandit algorithm over $T$ users produces only one dependent trajectory, because the algorithm's decisions depend on all past interactions. Reliable inference therefore demands many independent restarts of the algorithm, making experimentation costly and delaying deployment decisions. We propose \emph{Artificial Replay} (AR) as a new experimental design for this problem. AR first runs one policy and records its trajectory. When the second policy is executed, it reuses a recorded reward whenever it selects an action the first policy already took, and queries the real environment only otherwise. We develop a new analytical framework for this design and prove three key properties of the resulting estimator: it is unbiased; it requires only $T + o(T)$ user interactions instead of $2T$ for a run of the treatment and control policies, nearly halving the experimental cost when both policies have sub-linear regret; and its variance grows sub-linearly in $T$, whereas the estimator from a na\"ive design has a linearly-growing variance. Numerical experiments with UCB, Thompson Sampling, and $\epsilon$-greedy policies confirm these theoretical gains.
\end{abstract}

\section{Introduction}
Online learning algorithms, in particular multi-armed bandit algorithms, are used widely in online platforms.
For example, new products are continuously launched by sellers on Walmart Marketplace, and these products need to be shown to users before sufficient interaction data are available, to collect such data at the first place.
Likewise, new users arrive without reliable historical behavior.
In both cases, the platform has to recommend products to users, learn from early feedback, and control the cost of exploration.
This is the cold-start regime where online learning algorithms are most needed.
Figure~\ref{fig:walmart-new-products} illustrates such a case on the Walmart website. When a user searches for a pair of socks, one that categorized as ``New Arrivals'' also appear in the regular search result, thanks to the exploratory phase of the online learning algorithm.

To ground the discussion, we consider a simplified but representative setting.
There are $K$ new items (or arms), and users arrive sequentially over a horizon of length $T$.
In each period, the learning algorithm recommends an item, observes a random engagement outcome, and updates its internal memory such as the learned ``reward'' of each item.
In real systems, the action can be richer than a single item (e.g., a ranking), contextual information (of both the item and the user) can be incorporated, and new items may arrive over time.
Nevertheless, this setting is realistic and useful to discuss the key challenge in this paper: comparing two online learning policies in an environment and identifying the better one with statistical confidence.

\begin{figure}[htbp]
    \centering
    \caption{Walmart search results for ``socks''. The left panel displays the default ``All results'' interface, while the right panel displays the filtered view triggered by the ``New Arrivals'' button.
    The red boxes highlight the new arrival item appearing in the regular search result.}
    \label{fig:walmart-new-products}
    \vspace{0.5em}
    \begin{subfigure}[b]{0.43\textwidth}
        \centering
        \caption*{All results}
        \vspace{-0.5em}
        \fbox{\includegraphics[width=\linewidth]{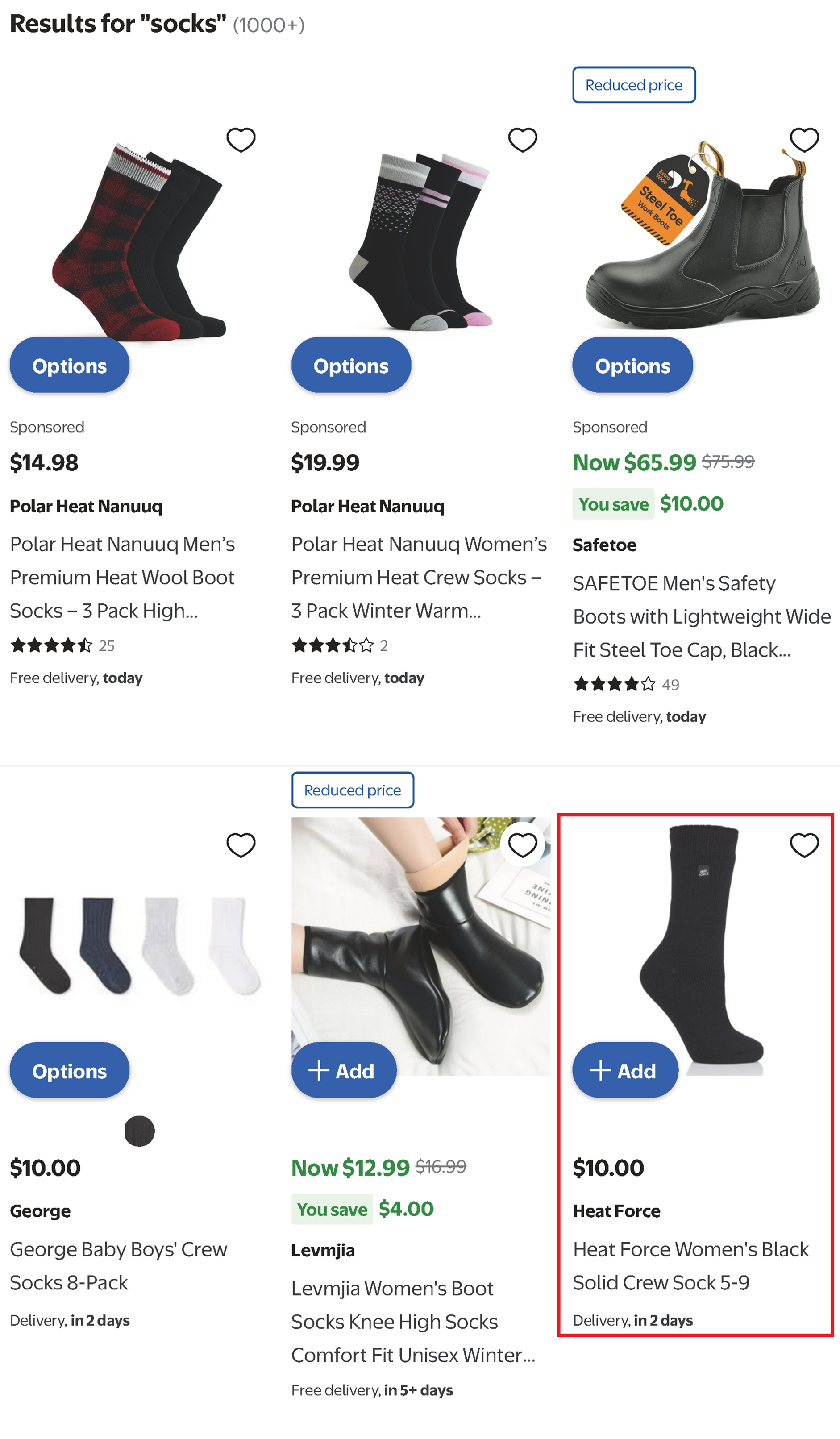}}
    \end{subfigure}
    \hfill
    \begin{subfigure}[b]{0.43\textwidth}
        \centering
        \caption*{New Arrivals}
        \vspace{-0.5em}
        \fbox{\includegraphics[width=\linewidth]{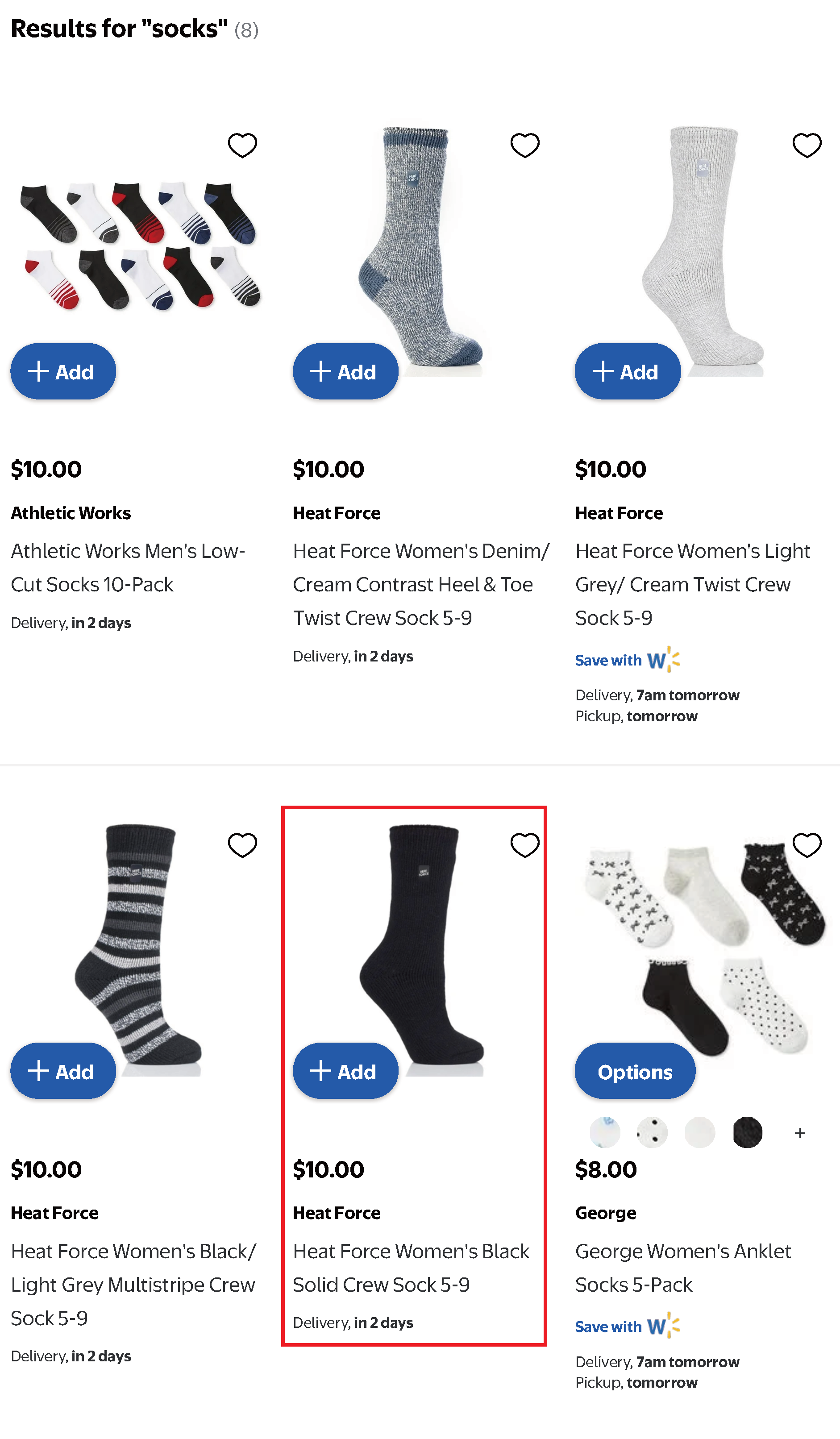}}
    \end{subfigure}
\end{figure}

\begin{figure}[tbp]
\centering
\caption{Na\"{\i}ve design for comparing two multi-armed bandit policies.}
\label{fig:intro-naive-design}
\vspace{0.5em}
\begin{tikzpicture}[
    >=Stealth,
    node distance=1.8cm,
    every node/.style={font=\small},
    box/.style={draw, rounded corners, align=center, text width=3.3cm, minimum height=0.95cm},
    arr/.style={->, thick}
]
\node[box] (users) {Arriving users ($2T$)};
\node[box, below left=0.8cm and 2cm of users] (control) {Control policy $\pi_0$\\separate memory};
\node[box, below right=0.8cm and 2cm of users] (treat) {Treatment policy $\pi_1$\\separte memory};
\node[box, below=2.5cm of users, text width=4.3cm] (est) {Na\"{\i}ve estimator\\$\sum_{t=1}^T \text{reward}_t^{\pi_1}-\sum_{t=1}^T \text{reward}_t^{\pi_0}$};

\draw[arr] (users) -- node[midway, above, sloped] {$T$ users} (control);
\draw[arr] (users) -- node[midway, above, sloped] {$T$ users} (treat);

\draw[arr] (control) -- (est);
\draw[arr] (treat) -- (est);
\end{tikzpicture}
\end{figure}
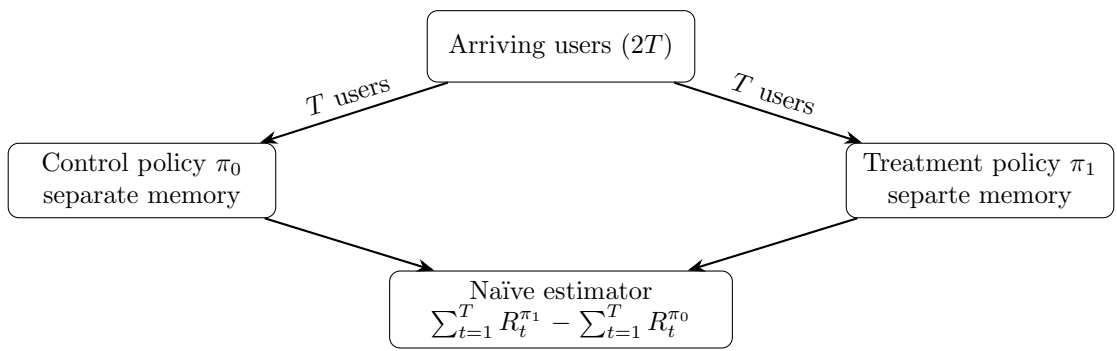

As one of the most consequential choices in platform operations, selecting a good online learning algorithm is usually carried out through online experiments.
In practice, a standard approach is to run a control policy $\pi_0$ and a treatment policy $\pi_1$ on two independent streams of users.
Different from other static treatments suitable for A/B testing, online learning policies are dynamic in nature and have system memory (e.g., the learned reward of each arm).
Therefore, the two streams of users maintain separate memory.
More precisely, the treatment (control) policies only use the interactions from the users in the treatment (control) stream to update the learned rewards of each item.
The resulting estimator for the average treatment effect is based on two independent bandit trajectories, each with horizon $T$.
Figure~\ref{fig:intro-naive-design} summarizes this na\"{\i}ve design at a high level.

This na\"{\i}ve design in Figure~\ref{fig:intro-naive-design} maintains strict independence between the treatment and control policies.
If we treat the trajectory of the $T$ users in the treatment (control) group as a single run of the treatment (control) online learning algorithm and record the aggregate reward as a ``sample'', then it fits the perspective of A/B testing. 
However, it is sample-inefficient, because the outcomes of $T$ users assigned to an algorithm only generate a single sample instead of $T$ independent observations.
Because of temporal dependence in learning dynamics, this aggregate reward of the trajectory may have a high variance.
Hence, to obtain stable inference, platforms typically need many repeated batches and policy restarts, which translate into large experimentation cost and delayed deployment decisions.
This is a central bottleneck when the objective is to compare online learning algorithms rather than static treatments.

In this paper, we propose \emph{Artificial Replay} (AR), a new experimental design that addresses this bottleneck.
At a high level, AR first runs one policy and records its action-reward history.
When the second policy is deployed, AR reuses (replays) historical rewards from the same arm by the first policy in the same order whenever possible, and queries the real environment only when relevant historical rewards have been used up.
Therefore, AR deliberately introduces coupling between the two policy trajectories.
This coupling is the key mechanism behind both improved sample efficiency and variance reduction.
Figure~\ref{fig:intro-ar-design} provides a conceptual view of this design.

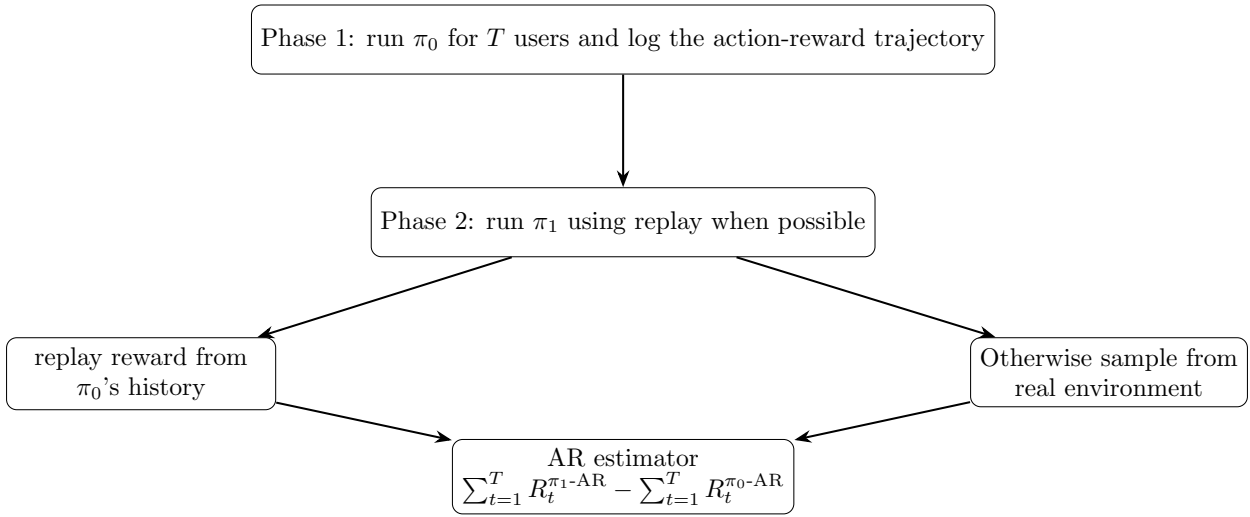
\begin{figure}[htbp]
\centering
\caption{Artificial Replay: run one policy, then replay compatible rewards for the other policy.}
\label{fig:intro-ar-design}
\vspace{0.5em}
\begin{adjustbox}{max width=\linewidth}
\begin{tikzpicture}[
    >=Stealth,
    node distance=1.55cm,
    every node/.style={font=\small},
    box/.style={draw, rounded corners, align=center, minimum width=4.5cm, minimum height=0.95cm},
    arr/.style={->, thick}
]
\node[box] (p1) {Phase 1: run $\pi_0$ for $T$ users and log the action-reward trajectory};
\node[box, below=of p1] (p2) {Phase 2: run $\pi_1$ using replay when possible};
\node[box, below left=1.1cm and 1.3cm of p2, minimum width=3.7cm] (replay) {replay reward from \\ $\pi_0$'s history};
\node[box, below right=1.1cm and 1.3cm of p2, minimum width=3.7cm] (real) {Otherwise sample from \\ real environment};
\node[box, below=2.5cm of p2, minimum width=4.6cm] (out) {AR estimator\\
$\sum_{t=1}^T \text{reward}_t^{\pi_1} \text{(AR)}-\text{reward}_t^{\pi_0} \text{(AR)}$};

\draw[arr] (p1) -- (p2);
\draw[arr] (p2) -- (replay);
\draw[arr] (p2) -- (real);
\draw[arr] (replay) -- (out);
\draw[arr] (real) -- (out);
\end{tikzpicture}
\end{adjustbox}
\end{figure}

Our contributions are threefold:
\begin{itemize}
    \item A new experimental design for comparing multi-armed bandit policies:
    We formulate the AR design and its corresponding estimator for the average treatment effect between two bandit policies.
    Compared with the na\"{\i}ve approach based on two independent runs, AR uses controlled replay to leverage previously observed rewards.

    \item A novel analytical framework:
    We introduce a shared-reward-stack probability model and establish its distributional equivalence to the canonical model underlying the AR experiment (Theorem~\ref{thm:AR-stack-same-dist}).
    This step is crucial because the canonical representation of AR exhibits strong path-dependent coupling between the two adaptive policy trajectories, making direct analysis particularly challenging.
    The shared-reward-stack model provides a tractable representation while preserving the exact joint distribution of action-reward sequences, and it reveals the stopping-time and martingale structure needed for the analysis.
    As a result, this framework is the backbone of all subsequent theoretical results, including sample-efficiency, unbiasedness, and variance reduction.

    \item Theoretical guarantees:
    We prove that the AR design is symmetric and much more sample-efficient than the na\"{\i}ve design, and the AR estimator is unbiased and variance-reduced relative to the na\"{\i}ve estimator.
    These properties have direct implications for online experimentation: symmetry ensures that the comparison is flexible with respect to the policy deployment order; sample efficiency reduces the amount of real-environment interaction needed to run the experiment; and unbiasedness together with asymptotic variance reduction ensures that the estimated performance gap is both centered at the correct target and statistically more precise.
    In particular, Theorem~\ref{thm:AR-pull-times} bounds the expected number of real-environment interactions, Theorem~\ref{thm:unbias} establishes unbiasedness, and Theorem~\ref{thm:asy_var} shows asymptotic variance reduction of the AR estimator.
    Together, these results support more reliable deployment decisions, especially in long-horizon experiments where both cost and statistical precision are critical.
    Numerical experiments further validate these effects across representative policy pairs.
\end{itemize}

\subsection{Literature Review}
Our study is related to the \emph{policy evaluation} of multi-armed bandit algorithms.
The multi-armed bandit framework has been extensively studied for its efficiency in online decision-making (see \citealt{lattimore2020bandit} for a comprehensive introduction).
The na\"ive design in our problem can be seen as (on-policy) evaluation and comparison of multi-armed bandit algorithms.
A natural question is whether the comparison of two multi-armed bandit algorithms can be conducted using off-policy evaluation, e.g, the treatment policy is evaluated using the data of the control policy.
There is extensive literature on the off-policy evaluation (OPE) of multi-armed bandit policies or contextual bandits (see, e.g., \citealt{wang2017optimal,farajtabar2018more,zhan2021off}).
For commonly used multi-armed bandit policies, it is a challenging problem even when the target policy, whose performance is to be evaluated, is as simple as always pulling a single arm \citep{nie2018adaptively,chen2022debiasing}.
One key challenge is that the logging policy, i.e., the multi-armed bandit algorithm that is used to generate the data, can be nonstationary and deterministic, such as the Upper Confidence Bound (UCB) algorithm. 
Deterministic policies violate a crucial assumption in OPE and render inverse-propensity-type estimators unusable.
Another challenge is that systems under multi-armed bandit algorithms cannot be modeled as a simple Markov chain and the policy is necessarily history-dependent rather than simply state-dependent as in Markov decision processes.
While a few studies investigate this topic \citep{bennett2024proximal,zhang2025statistical,zhan2024policy}, 
it is understood that OPE of history-dependent policies can lead to high variances.
In our problem, the two algorithms are not restricted to be evaluated in an off-policy manner. 
We can thus carefully couple the two algorithms using AR and achieve a lower variance than on-policy evaluation, i.e., the na\"ive design, not to mention off-policy evaluation.
In summary, our approach is aimed at a different problem setting from OPE, in which both policies can be implemented online.

This study is also broadly related to \emph{experimental design} and causal inference.
We refer to \cite{wager2024causal} or \cite{imbens2015causal} for an introduction to causal inference and \cite{zhao2024experimental} for a review of recent developments in experimental design and optimization.
In general, the experimental design literature has been focusing on the comparison of static treatments, i.e., one-time decisions made for independent units, and individual-level randomization, or A/B testing
\citep{kohavi2020trustworthy}.
We investigate the treatment effect of multi-bandit algorithms that are applied dynamically.
Therefore, the problem we study can be seen as a more structured form of dynamic treatment regimes \citep{chakraborty2014dynamic}.
The structure of multi-armed bandit algorithms allows us to derive a richer set of analytical results.
Recently, there is growing interest in the \emph{interference} between units under standard A/B testing, which violates the stable unit treatment value assumption (SUTVA), thanks to the rising prevalence of experimentation in online marketplaces \citep{johari2022experimental, si2023tackling, ye2023cold, holtz2025reducing}. 
If the interference is temporal, i.e., a treatment on a unit earlier may influence the treatment effect of later units, then the multi-armed bandit algorithm can be seen as a special case. 
However, the temporal interference in the literature usually has a special structure, such as one-period carryover \citep{bojinov2023design} or geometric decaying \citep{hu2022switchback}, for the problem to be tractable.
We specifically focus on multi-armed bandit algorithms, which do not satisfy these assumptions. 
More importantly, the policy itself, rather than simply the treatment effect, is history-dependent for multi-armed bandit algorithms.
Therefore, our conceptual framework significantly deviates from that of the temporal interference literature \citep{xiong2024data}.
Similarly, nonstationary treatment allocation and treatment effect studied in \cite{wu2025nonstationary} do not cover the problem in this study.
In this literature, our setting is also related to \cite{farias2022markovian}, which investigate the treatment effect when two policies are randomly chosen in each period in a Markov decision process.
Like us, the system dynamics call for a different analytical framework from the A/B testing paradigm for static treatments. 
\cite{farias2022markovian} use the $Q$-function to help with the estimation of the treatment effect.
Our major difference is that most multi-armed bandit algorithms are not Markovian with a compact state space.

There are three recent studies closely related to our work. Firstly, \cite{banerjee2025artificial} introduce the concept of artificial replay and use it as a meta-algorithm to efficiently incorporate historical data into a bandit algorithm. Their approach enables the use of only a fraction of historical data while achieving the same regret to a full warm-start algorithm under certain conditions. In contrast, our work focuses on evaluating and estimating the treatment effect between two given bandit policies through online experiments, re-purposing artificial replay as an experimental design tool in this setting. As a result, our theoretical framework differs substantially from \cite{banerjee2025artificial}.
Secondly, \cite{li2025choosing} investigate the problem of choosing the better bandit algorithm under data sharing, where each algorithm selects actions based on a shared history generated by both. This data sharing mechanism introduces interference and leads to ``symbiosis bias" \citep{brennan2025reducing}. The authors theoretically characterize the impact of such bias on the sign of expected estimates of the global treatment effect. In contrast, our work considers a distinct setting where two bandit algorithms are compared without being constrained to share the same history. We propose a novel design, artificial replay, to couple rewards from the two policies and construct an unbiased estimator for the performance gap with reduced variance. Consequently, our analysis framework and theoretical results differ significantly from \cite{li2025choosing}. For example, our design leads to an unbiased estimator of the treatment effect.
Third, \cite{simester2020efficiently} propose an experimental design leveraging data reuse to compare static policies, an approach later extended to Markov decision processes by \cite{chen2024improving}.
Because \cite{simester2020efficiently} focus on static policies,
their estimator constructed through customer-level randomization naturally achieves unbiasedness and variance reduction without requiring additional theoretical analysis.
While our work builds upon a similar intuition by allowing two candidate policies to share a fraction of reward samples, our adaptive policy setting differs fundamentally from their static one.
In fact, the sequential dependence in bandit policies prevents unbiased evaluation via simple customer-level randomization, and thus precludes data reuse at the customer level as in \cite{simester2020efficiently}. 
Therefore, we propose a different experimental design to reuse data across sequential steps within the trajectories of two adaptive policies. More importantly, we develop a novel analytical approach based on the stopping time and martingale techniques to handle the adaptive nature of bandit policies.
 
Finally, our analysis of the variance of the AR estimator has connections to the recent literature on multi-armed bandits that explores the distributional properties of regret under a single bandit policy. For instance, studies such as \cite{kalvit2021closer} and \cite{kuang2024weak} utilize diffusion approximations to examine the asymptotic distribution of regret in regimes where the mean reward gaps between arms decrease as the horizon increases. Additionally, \cite{fan2024fragility} demonstrates that optimized algorithms can inherently induce heavy-tailed regret distributions, while \cite{simchi2023regret} characterizes the optimal trade-off between expected regret and tail risk for regret distribution. Furthermore, \cite{chen2026bandit} establishes a fundamental lower bound governing the trade-off between worst-case regret and allocation variability (i.e., the largest standard deviation of pull counts across arms). In contrast to these studies, our variance analysis of the AR estimator is instance-dependent and does not rely on diffusion approximations. Instead, it focuses on analyzing the variance of cumulative reward (or random regret) under a single bandit policy, as well as the covariance of cumulative rewards under two distinct bandit policies.

The remainder of the paper is organized as follows.
Section~\ref{sec:problem_formulation} formulates the problem and the na\"{\i}ve baseline experiment.
Section~\ref{sec:AR_Design} presents the AR design and estimator.
Section~\ref{sec:analytical_framework} develops the shared-reward-stack framework and the stopping-time/martingale tools.
Section~\ref{sec:theoretical_properties} establishes the main theoretical properties of the AR estimator.
Section~\ref{sec:numerical} reports numerical experiments, and finally Section~\ref{sec:conclusion} concludes. The proofs of all technical results are provided in the appendix.

\section{Problem Formulation}\label{sec:problem_formulation}
Throughout the paper, we use the notation $[n]$ to denote the set $\{1, 2, \ldots, n\}$ for any positive integer $n$, and $\mathbb{N}$ to denote the set of natural numbers $\{0, 1, 2, \dots\}$.
For any measurable set $E \subseteq \mathbb{R}^d$, we denote by $\mathcal{B}(E)$ the Borel $\sigma$‑algebra on $E$,  i.e., $\mathcal{B}(E) = \{E \cap B : B \in \mathcal{B}(\mathbb{R}^d)\}$. 
The indicator of an event $B$ is denoted by $\mathbb{I}\{B\}$, which equals 1 if $B$ occurs and 0 otherwise. 
We write $U(0, 1)$ for the uniform distribution on $[0, 1]$.
The notation $X \stackrel{\text{d}}{=} Y$ means that the random variables $X$ and $Y$ have the same distribution. 
The notation $f(T) = O(g(T))$ means that there exists a constant $C > 0$ such that $\vert f(T) \vert \leq C g(T)$ for all $T \geq 1$, and $f(T) = o(g(T))$ means that $\lim_{T \rightarrow \infty} \vert f(T) / g(T) \vert = 0$.

\subsection{Preliminary: Multi-Armed Bandit}
Consider a stochastic multi-armed bandit problem with $K$ arms, indexed by the set $[K] = \{1, 2, \ldots, K\}$.
Each arm $a \in [K]$ is associated with an unknown reward distribution $P_a$ with mean $\mu_a$ and variance $\sigma_a^2$.
In the context of recommender systems, $K$ is the total number of candidate items to be recommended to arriving users, and the mean reward represents the ``value'' of the item, such as the purchase probability or the click-through rate.
We assume there exists a unique optimal arm $a^* \triangleq \argmax_{a \in [K]} \mu_a$, with its corresponding mean reward denoted by $\mu^* \triangleq \mu_{a^*}$.
The sub-optimality gap for each arm $a \in [K]$ is then defined as $\Delta_a \triangleq \mu^* - \mu_a$.

In each period $t \in [T]$, where $T$ is the horizon, a multi-armed bandit policy (or algorithm) $\pi$ selects an arm $A_t^{\pi} \in [K]$ based on the history of past interactions $\mathcal{H}_{t-1}^{\pi} \triangleq \{(A_1^{\pi}, R_1^{\pi}), \dots, (A_{t-1}^{\pi}, R_{t-1}^{\pi})\}$, where $R_s^{\pi}$ is the realized reward in period $s$.
Upon playing arm $A_t^{\pi}$, the policy observes the reward $R_t^{\pi}$, drawn from the corresponding distribution $P_{A_t^{\pi}}$, independent of everything else.

Formally, a policy $\pi$ is defined as a sequence of stochastic kernels $\{\pi_t\}_{t=1}^T$, where for each $t \in [T]$ and history $h_{t-1} \in \mathcal{H}_{t-1}^{\pi}$, $\pi_t(\cdot\mid h_{t-1})$ specifies a probability distribution over the arm set $[K]$. 
The arm chosen at period $t$ is sampled according to
$A_t^\pi \sim \pi_t(\cdot \mid \mathcal H_{t-1}^\pi)$.

The performance of a policy $\pi$ is evaluated by its expected regret, which measures the expected total reward lost by not always playing the optimal arm $a^*$. 
Formally, the random regret of a policy $\pi$ over a horizon $T$ is defined as
\begin{align*}
    Reg(\pi; T) \triangleq T \mu^* - \sum_{t=1}^T R_t^{\pi}.
\end{align*}
It should be noted that the cumulative reward $\sum_{t=1}^T R_t^{\pi}$ is an empirically observable quantity. 
The expected regret $\E[Reg(\pi; T)]$ is defined by taking the expectation of the random regret over both the policy's internal randomness and the stochastic rewards from the environment.
Let 
\begin{equation*}
    N_a^{\pi}(T) \triangleq \sum_{t=1}^T \mathbb{I}\{A_t^{\pi} = a\}
\end{equation*}
denote the random variable representing the total number of times arm $a$ is pulled by the policy up to period $T$. The expected regret can further be expressed as a weighted sum of the expected number of suboptimal-arm pulls, that is, $\E[Reg(\pi; T)] =\sum_{a = 1}^K \Delta_a  \E[N_a^{\pi}(T)]$.

\subsection{Empirical Comparison of Two Bandit Policies}
When applying multi-armed bandit algorithms to real-world problems, their performance depends on the choice of hyper-parameters and the characteristics of the reward distributions.
The worst-case or instance-dependent regret bounds provided in the theoretical literature may not provide a sharp characterization of their empirical performance.
This motivates the need for a rigorous framework for empirical comparison of two bandit policies in online experiments.

Consider a fixed MAB problem instance $\{P_a,\, a \in [K]\}$, with unknown parameters (i.e., the true mean rewards $\mu_a$ of each arm $a \in [K]$).
Let $\pi_0$ and $\pi_1$ denote two bandit policies under consideration. 
We designate $\pi_0$ as the \textit{control} policy and $\pi_1$ as the \textit{treatment} policy.
Our central objective is to quantify the performance gap within this environment, thereby identifying the superior policy.
To this end, consider the multi-armed bandit policy $\pi_i$ applied to the instance for $T$ periods.
Let $\{(A_{t}^{\pi_i},R_{t}^{\pi_i})\}_{t=1}^T$ denote the random action-reward trajectory.
Define the realized cumulative reward over $T$ periods by $Y_{n}^{\pi_i} \triangleq \sum_{t=1}^T R_{n,t}^{\pi_i}$.
Our objective is to estimate the average treatment effect (ATE), defined as the difference in the expected total rewards over a horizon of $T$ periods between the treatment policy $\pi_1$ and the control policy $\pi_0$:
\begin{align*}
    \theta(T) \triangleq \E\left[ Y^{\pi_1} - Y^{\pi_0}\right] = \E\left[\sum_{t=1}^T R_t^{\pi_1}\right] - \E\left[\sum_{t=1}^T R_t^{\pi_0}\right].
\end{align*}
The expectation is taken with respect to the random rewards and the internal randomization of the algorithm, while the MAB instance is fixed.
Equivalently, $\theta(T)$
can be interpreted as the difference in expected regret between the control and treatment policies over horizon $T$:
\begin{align*}
    \theta(T) = \E[Reg(\pi_0; T)] - \E[Reg(\pi_1; T)].
\end{align*}
The sign of $\theta(T)$ indicates the superior policy in expectation; specifically, $\theta(T) > 0$ implies that the treatment policy $\pi_1$ is preferable to the control policy $\pi_0$, and vice versa.

\paragraph{The Na\"{\i}ve Design}
Motivated by the application in the introduction, we rigorously formulate the na\"ve design that has been widely used in practice,
for the estimation of the treatment effect $\theta(T)$.
Given a batch of $2T$ arriving users, the na\"{\i}ve design allocates $T$ of them to the treatment and control groups, respectively.
Because the two groups are exposed to two multi-armed policies and keep track of their own system states (such as the learned mean reward), we abstract away the (random) allocation scheme and focus on two independent processes.
This is illustrated in Algorithm~\ref{alg:benchmark}.
Given the sample paths generated in Algorithm~\ref{alg:benchmark}, a na\"{\i}ve estimator for $\theta(T)$ is defined as
\begin{align*}
    \thetahat{b}(T) \triangleq \sum_{t=1}^T R_t^{\pi_1} - \sum_{t=1}^T R_t^{\pi_0},
\end{align*}
where $\text{b}$ stands for the baseline.
In practice, the estimator $\thetahat{b}(T)$ is recorded from a number of batches, each with $2T$ users, and averaged to lower the standard error.
For the analysis, because the batches are considered as i.i.d. copies and we focus on $\thetahat{b}(T)$.

\begin{algorithm}
\caption{The Na\"{\i}ve Design: Deploy Two Bandit Policies $\pi_0$ and $\pi_1$ Independently}
\label{alg:benchmark} 
\begin{algorithmic}[1]
    \Require horizon $T$, two bandit policies $\pi_0$ and $\pi_1$
    
    \vspace{0.5em}
    \Statex \textbf{Phase 1: Policy $\pi_0$ interacts with the environment}
    \State \textbf{Initialize:} $\mathcal{H}_0^{\pi_0} \gets \emptyset$
    \For{$t = 1,2,\dots,T$}
        \State Sample arm $A_t^{\pi_0} \sim (\pi_0)_t(\cdot \mid \mathcal{H}_{t-1}^{\pi_0})$
        \State Pull arm $A_t^{\pi_0}$ 
               and observe reward $R_t^{\pi_0} \sim P_{A_t^{\pi_0}}$
        \State Update histories $\mathcal{H}_t^{\pi_0} \gets \mathcal{H}_{t-1}^{\pi_0} \,\cup \, \{(A_t^{\pi_0}, R_t^{\pi_0})\}$
    \EndFor
    
    \vspace{0.5em}
    \Statex \textbf{Phase 2: Policy $\pi_1$ interacts with the environment}
    \State \textbf{Initialize:}  $\mathcal{H}_0^{\pi_1} \gets \emptyset$
    \For{$t = 1,2,\dots,T$}
    \State Sample arm $A_t^{\pi_1} \sim (\pi_1)_t(\cdot \mid \mathcal{H}_{t-1}^{\pi_1})$
    \State Pull arm $A_t^{\pi_1}$ 
           and observe reward $R_t^{\pi_1} \sim P_{A_t^{\pi_1}}$
    \State Update histories $\mathcal{H}_t^{\pi_1} \gets \mathcal{H}_{t-1}^{\pi_1} \, \cup \, \{(A_t^{\pi_1}, R_t^{\pi_1})\}$
    \EndFor
\end{algorithmic}
\end{algorithm}

By construction, it is immediate that the estimator $\thetahat{b}(T)$ is unbiased for $\theta(T)$.
However, its variance is given by
\begin{align*}
    \Var(\thetahat{b}(T)) = \Var(Reg(\pi_0; T)) + \Var(Reg(\pi_1; T)) 
\end{align*}
where for $i \in \{0, 1\}$, $Reg(\pi_i; T) = T \mu^* - \sum_{t=1}^T R_t^{\pi_i}$ is the random regret of policy $\pi_i$.
Due to the stochastic and adaptive nature of bandit algorithms, the random regret of a single run is inherently high-variance. In particular, the variance can exhibit asymptotic linear growth with respect to the horizon $T$.
For instance, as shown in Section~\ref{sec:theoretical_properties}, for certain UCB policies, the variance of the cumulative rewards---equivalently, the variance of the random regret---has a leading linear term $\sigma_{a^*}^2 T$, where $\sigma_{a^*}^2$ denotes the reward variance associated with the optimal arm $a^*$.

The high‑variance property renders the na\"{\i}ve estimator $\thetahat{b}(T)$ sample-inefficient, particularly for long horizons. 
Obtaining a reliable estimate would necessitate an an excessive number of experimental replications, making the na\"{\i}ve design prohibitively expensive for meaningful comparison. This critical drawback motivates the development of a variance reduction technique.

\section{A New Design Using Artificial Replay}\label{sec:AR_Design}
To mitigate the aforementioned issue of the na\"{\i}ve estimator, we propose the artificial replay (AR) design. 
The concept of ``artificial replay'' is first proposed by \cite{banerjee2025artificial}.
The core idea of using the AR design in our context is to break the strict independence between the two experimental runs, i.e., Phase 1 and Phase 2 in Algorithm~\ref{alg:benchmark}.
Instead of having the two policies interact with separate, independent environments, 
our method couples them by sharing a portion of reward samples.
By inducing strong positive correlation between the two sample paths, the AR method achieves a significant reduction of variance, leading to a far more stable and efficient estimator. 

The AR design consists of two phases.
In the first phase, we deploy the control policy $\pi_0$, allowing it to interact with the environment for $T$ periods. 
We record the full sequence of actions taken and rewards received by $\pi_0$ during this process, which constitutes a historical trajectory, denoted as $\mathcal{H}_T^{\pi_0\text{-AR}}$.
In the second phase, the treatment policy $\pi_1$ is deployed.
In each period $t \in [T]$, $\pi_1$ selects an action $A_t^{\pi_1\text{-AR}}$ based on its own history. We then check if this same action exists within $\pi_0$'s trajectory and has not yet been ``replayed":
\begin{itemize}
    \item If such an unused historical action is found, we forgo a new interaction with the environment. Instead, we directly replay the reward associated with the earliest matching action in $\pi_0$'s history and mark that entry as used/replayed. In this case, the replayed reward is recorded as $R_t^{\pi_1\text{-AR}}$ in $\pi_1$'s trajectory; 
    \item If no such action is found (e.g., because the action has not been played by $\pi_0$ at all or all matching historical interactions have already been replayed), $\pi_1$ performs a real interaction with the environment to observe a new reward. In this case, the new reward is recorded as $R_t^{\pi_1\text{-AR}}$ in $\pi_1$'s trajectory.
\end{itemize}
The formal steps can be found in Algorithm~\ref{alg:AR}. For each $i \in \{0, 1\}$, $a \in [K]$, and $t \in [T]$, let
\begin{align*}
   \quad N_a^{\pi_i\text{-AR}}(t) \triangleq \sum_{r=1}^t \mathbb{I}\{A_r^{\pi_i\text{-AR}} = a\}
\end{align*}
denote the cumulative number of times arm $a$ is selected by policy $\pi_i$ up to period $t$ in the AR experiment.
We introduce the random variables $N^{\text{e-AR}}(T)$ and $N^{\text{r-AR}}(T)$ to denote the \emph{number of real-environment interactions} and the \emph{number of replayed rewards} in the AR experiment, respectively.
Mathematically, $N^{\text{e-AR}}(T)$ and $N^{\text{r-AR}}(T)$ can be expressed as
\begin{align}
    N^{\text{e-AR}}(T) &= \sum_{a \in [K]} \max\{ N_a^{\pi_0\text{-AR}}(T), N_a^{\pi_1\text{-AR}}(T)\}, \label{eq:def-num-env-interactions} \\
    N^{\text{r-AR}}(T) &= \sum_{a \in [K]} \min\{ N_a^{\pi_0\text{-AR}}(T), N_a^{\pi_1\text{-AR}}(T)\}.
    \label{eq:def-num-replay}
\end{align}
We further denote their corresponding expected values by 
\begin{align}\label{eq:def-exp-num-env-interactions-replay}
    n^{\mathrm{e\text{-}AR}}(T) \triangleq \mathbb{E}[ N^{\mathrm{e\text{-}AR}}(T)] \quad \text{ and } \quad n^{\mathrm{r\text{-}AR}}(T) \triangleq \mathbb{E}[ N^{\mathrm{r\text{-}AR}}(T) ].
\end{align}
From the design of the AR experiment, one can anticipate that the number of real-environment interactions may be substantially smaller than $2T$.
In Figure~\ref{fig:AR-arrow-diagram}, we illustrate such an instance with $T=10$ and three arms. 
With artificial replay, the experiment involves only $12$ real-environment interactions for the two policies combined.

\begin{figure}[H]
\caption{An illustration of artificial replay with $K=3$ and $T=10$. The top row of each policy is the index of the pulled arm and the bottom row shows the realized reward.}
\label{fig:AR-arrow-diagram}
\vspace{0.5em}
    \centering
    \begin{tikzpicture}[scale=0.9,   
        pullbox/.style={rectangle,fill=blue!30,minimum width=6mm,minimum height=6mm,align=center},
        replaybox/.style={rectangle,fill=orange!40,minimum width=6mm,minimum height=6mm,align=center},
        replayarrow/.style={orange!60,thick,-{Stealth[length=2mm]}},
        line/.style={thick,black},
    ]

    \node[pullbox] (a11) at (1.0,1.3) {1};
    \node[pullbox] (r11) at (1.0,0.7) {1};
    \node[above=0.5em,font=\small] at (a11.north) {$t=1$};
    \node[pullbox] (a12) at (2.1,1.3) {2};
    \node[pullbox] (r12) at (2.1,0.7) {0};
    \node[above=0.5em,font=\small] at (a12.north) {$t=2$};
    \node[pullbox] (a13) at (3.2,1.3) {3};
    \node[pullbox] (r13) at (3.2,0.7) {1};
    \node[above=0.5em,font=\small] at (a13.north) {$t=3$};
    \node[pullbox] (a14) at (4.3,1.3) {1};
    \node[pullbox] (r14) at (4.3,0.7) {0};
    \node[above=0.5em,font=\small] at (a14.north) {$t=4$};
    \node[pullbox] (a15) at (5.4,1.3) {2};
    \node[pullbox] (r15) at (5.4,0.7) {0};
    \node[above=0.5em,font=\small] at (a15.north) {$t=5$};
    \node[pullbox] (a16) at (6.5,1.3) {3};
    \node[pullbox] (r16) at (6.5,0.7) {1};
    \node[above=0.5em,font=\small] at (a16.north) {$t=6$};
    \node[pullbox] (a17) at (7.6,1.3) {3};
    \node[pullbox] (r17) at (7.6,0.7) {1};
    \node[above=0.5em,font=\small] at (a17.north) {$t=7$};
    \node[pullbox] (a18) at (8.7,1.3) {3};
    \node[pullbox] (r18) at (8.7,0.7) {0};
    \node[above=0.5em,font=\small] at (a18.north) {$t=8$};
    \node[pullbox] (a19) at (9.8,1.3) {2};
    \node[pullbox] (r19) at (9.8,0.7) {0};
    \node[above=0.5em,font=\small] at (a19.north) {$t=9$};
    \node[pullbox] (a110) at (10.9,1.3) {3};
    \node[pullbox] (r110) at (10.9,0.7) {0};
    \node[above=0.5em,font=\small] at (a110.north) {$t=10$};

    \node[replaybox] (a21) at (1.0,-0.7) {1};
    \node[replaybox] (r21) at (1.0,-1.3) {1};
    \node[replaybox] (a22) at (2.1,-0.7) {2};
    \node[replaybox] (r22) at (2.1,-1.3) {0};
    \node[replaybox] (a23) at (3.2,-0.7) {3};
    \node[replaybox] (r23) at (3.2,-1.3) {1};
    \node[replaybox] (a24) at (4.3,-0.7) {3};
    \node[replaybox] (r24) at (4.3,-1.3) {1};
    \node[replaybox] (a25) at (5.4,-0.7) {1};
    \node[replaybox] (r25) at (5.4,-1.3) {0};
    \node[replaybox] (a26) at (6.5,-0.7) {2};
    \node[replaybox] (r26) at (6.5,-1.3) {0};
    \node[pullbox] (a27) at (7.6,-0.7) {1};
    \node[pullbox] (r27) at (7.6,-1.3) {0};
    \node[replaybox] (a28) at (8.7,-0.7) {2};
    \node[replaybox] (r28) at (8.7,-1.3) {0};
    \node[replaybox] (a29) at (9.8,-0.7) {3};
    \node[replaybox] (r29) at (9.8,-1.3) {1};
    \node[pullbox] (a210) at (10.9,-0.7) {2};
    \node[pullbox] (r210) at (10.9,-1.3) {1};

    \draw[line] (0.67,1) -- (11.23,1);    
    \draw[line] (0.67,-1.0) -- (11.23,-1.0);
    \node[left=1.2em,font=\small] at (a11.west) {$A_t^{\pi_0\text{-AR}}$};
    \node[left=1.2em,font=\small] at (r11.west) {$R_t^{\pi_0\text{-AR}}$};
    \node[left=1.2em,font=\small] at (a21.west) {$A_t^{\pi_1\text{-AR}}$};
    \node[left=1.2em,font=\small] at (r21.west) {$R_t^{\pi_1\text{-AR}}$};
    
    \draw[replayarrow] (r11.south) to (a21.north);
    \draw[replayarrow] (r12.south) to (a22.north);
    \draw[replayarrow] (r13.south) to (a23.north);
    \draw[replayarrow] (r16.south) to (a24.north);
    \draw[replayarrow] (r14.south) to (a25.north);
    \draw[replayarrow] (r15.south) to (a26.north);
    \draw[replayarrow] (r19.south) to (a28.north);
    \draw[replayarrow] (r17.south) to (a29.north);
    \node[pullbox,label={[align=left, anchor=west]right:{Sample from real\\arm pulls}}] at (13,0.6) {};
    \node[replaybox,label=right:{Replayed sample}] at (13,-0.6) {};
    \end{tikzpicture}
\end{figure}
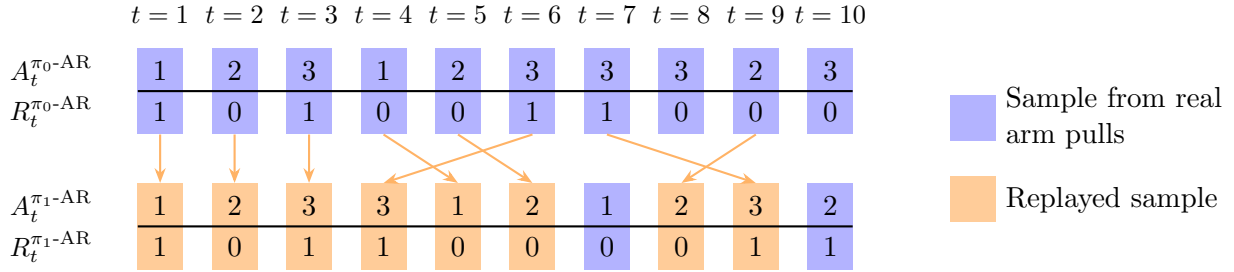

\begin{algorithm}[H]
\caption{Artificial Replay}
\label{alg:AR}
\begin{algorithmic}[1]
    \Require horizon $T$, two bandit policies $\pi_0$ and $\pi_1$
    \vspace{0.5em}
    \Statex \textbf{Phase 1: Policy $\pi_0$ interacts with the environment}
    \State \textbf{Initialize:} $\mathcal{H}_0^{\pi_0\text{-AR}} \gets \emptyset$
    \For{$t = 1, 2, \dots, T$}
        \State Sample arm $A_t^{\pi_0\text{-AR}} \sim (\pi_0)_t(\cdot \mid \mathcal{H}_{t-1}^{\pi_0\text{-AR}})$
        \State Pull arm $A_t^{\pi_0\text{-AR}}$ and observe reward $R_t^{\pi_0\text{-AR}} \sim P_{A_t^{\pi_0\text{-AR}}}$
        \State Update history $\mathcal{H}_t^{\pi_0\text{-AR}} \gets \mathcal{H}_{t-1}^{\pi_0\text{-AR}} \,\cup\, \{(A_t^{\pi_0\text{-AR}}, R_t^{\pi_0\text{-AR}})\}$
    \EndFor
    
    \vspace{0.5em}
    \Statex \textbf{Phase 2: Policy $\pi_1$ operates under artificial replay}
    \State \textbf{Initialize:}  $\mathcal{H}_0^{\pi_1\text{-AR}} \gets \emptyset$, $\mathcal{I}_{\text{used}} \gets \emptyset$, and $\mathcal{I} \gets \{1, 2, \cdots, T\}$
    \For{$t = 1, 2, \dots, T$}
        \State Sample arm $A_t^{\pi_1\text{-AR}} \sim (\pi_1)_t(\cdot \mid \mathcal{H}_{t-1}^{\pi_1\text{-AR}})$
        \If{there exists $s \in \mathcal{I} \setminus \mathcal{I}_{\text{used}}$ such that $A_s^{\pi_0\text{-AR}} = A_t^{\pi_1\text{-AR}}$}
            \State Set $s^* \gets \min\{s \in \mathcal{I} \setminus \mathcal{I}_{\text{used}} : A_s^{\pi_0\text{-AR}} = A_t^{\pi_1\text{-AR}}\}$
            \State Store $R_t^{\pi_1\text{-AR}} \gets R_{s^*}^{\pi_0\text{-AR}}$
            \State Update $\mathcal{I}_{\text{used}} \gets \mathcal{I}_{\text{used}} \cup \{s^*\}$
        \Else
            \State Pull arm $A_t^{\pi_1\text{-AR}}$ and observe reward $R_t^{\pi_1\text{-AR}} \sim P_{A_t^{\pi_1\text{-AR}}}$
        \EndIf
        \State Update history $\mathcal{H}_t^{\pi_1\text{-AR}} \gets \mathcal{H}_{t-1}^{\pi_1\text{-AR}} \,\cup\, \{(A_t^{\pi_1\text{-AR}}, R_t^{\pi_1\text{-AR}})\}$
    \EndFor
\end{algorithmic}
\end{algorithm}

Based on Algorithm~\ref{alg:AR}, our proposed AR estimator is defined as
\begin{equation}\label{eq:def-AR-estimator}
    \thetahat{AR}(T) \triangleq 
    \sum_{t=1}^{T} R_{t}^{\pi_1\text{-AR}} - \sum_{t=1}^{T} R_{t}^{\pi_0\text{-AR}}.
\end{equation}
The AR estimator possesses several desirable properties, for which we will provide formal proofs in the subsequent sections.
First, the AR experimental design is \emph{symmetric}.
Specifically, if we exchange the roles of policies $\pi_0$ and $\pi_1$---that is, letting $\pi_1$ first interact with the environment and then having $\pi_0$ replay on its historical data---the newly constructed estimator is identically distributed to the original estimator $\thetahat{AR}(T)$.
This ensures that the comparison is fair regardless of which policy is deployed first.
Second, the expected number of interactions with the environment during the AR experiment can be significantly less than the $2T$ interactions required by the na\"{\i}ve method. 
A more general theoretical result is provided in Theorem~\ref{thm:AR-pull-times}.
Third, similar to the na\"{\i}ve estimator, the AR estimator is also unbiased for the true parameter $\theta(T)$. This is rigorously proven in Theorem~\ref{thm:unbias}.
Most importantly, the AR estimator achieves significant asymptotic variance reduction compared to the na\"{\i}ve estimator.
Theorem~\ref{thm:asy_var} shows that, under mild conditions, the $T$-normalized variance of the AR estimator converges to zero. This contrasts sharply with the na\"{\i}ve estimator whose variance typically scales with $T$, highlighting the superior statistical property of the AR method.

However, establishing the above statistical properties of the AR estimator poses unique theoretical challenges.
Standard bandit analysis primarily focuses on the performance (e.g., regret bounds) of a single policy in isolation. 
In contrast, the AR experiment inherently couples the action-reward trajectories of the control and treatment policies via a reward replay mechanism.
This coupling creates complex, history-dependent correlations, which lies outside the scope of standard single-policy analytical techniques.
To overcome this challenge, we develop an analytical framework tailored to the AR experiment.

\section{Theoretical Foundations for the Analysis}
\label{sec:analytical_framework}
In this section, we establish the probablistic foundations for analyzing the statistical properties of the AR estimator. 
Specifically, we construct two alternative probability spaces that support the random actions and rewards generated by two candidate policies in the AR experiment: the \emph{canonical model} and the \emph{shared-reward-stack model}.

Before detailing these two constructions, we first review their counterparts in the single-policy bandit experiment.
\citet[Section~4.6]{lattimore2020bandit} introduces three distinct probability spaces that support the random action-reward sequence in the standard bandit experiment, termed the canonical model, the reward-stack model, and the random table model.
Since the first two are directly related to the probability models we propose for the AR experiment, we summarize them below.

\paragraph{Canonical model of classical multi-armed bandit.} In the canonical model, the sample space is the set of action-reward trajectories, $\Omega \triangleq ([K] \times \mathbb{R})^{T}$, equipped with the Borel $\sigma$-algebra $\mathcal{F} \triangleq \mathcal{B}(\Omega)$.
    For any sample path $\omega = (a_1, r_1, \dots, a_T, r_T) \in \Omega$, the action and reward random variables are defined by the coordinate projections $A_t^{\pi}(\omega) = a_t$ and $R_t^{\pi}(\omega) = r_t$ for all $t \in [T]$.
    Here, the superscript $\pi$ in $A_t^{\pi}$ and $R_t^{\pi}$ is used to indicate that their distributions rely on a specific probability measure on $(\Omega, \mathcal{F})$ induced by $\pi$, while ensuring
    consistency with the notation introduced in Algorithm~\ref{alg:benchmark}.    Indeed, there exists a probability measure $\mathbb{P}^{\pi}$ on $(\Omega,\mathcal F)$ such that for each $t \in [T]$,
    \begin{equation}\label{eq:A-R-single-conditional-dist}
        A_t^{\pi} \mid \mathcal H_{t-1}^{\pi} \sim \pi_t(\cdot \mid \mathcal H_{t-1}^{\pi}),
        \qquad R_t^{\pi} \mid (A_t^{\pi},\mathcal H_{t-1}^{\pi}) \sim P_{A_t^{\pi}}.
    \end{equation}
    
\paragraph{Reward-stack model of classical multi-armed bandit.} In the reward-stack model, the sample space is taken to be $\Omega_{\text{r}} \triangleq \mathbb R^{(K+1) \times T}$, equipped with the Borel $\sigma$-algebra $\mathcal F_{\text{r}} \triangleq \mathcal B(\Omega_{\text{r}})$. 
    A generic element of the sample space is denoted by $\omega_{\text{r}} = [(x_{a,t})_{a,t}, (y_t)_t]$. For each $a \in [K]$ and $t \in [T]$, let $X_{a,t}: \Omega_{\text{r}} \mapsto \mathbb{R}$ and $\eta_t: \Omega_{\text{r}} \mapsto \mathbb{R}$ be the coordinate projections defined by $X_{a, t}(\omega_{\text{r}}) = x_{a, t}$ and $\eta_t(\omega_{\text{r}}) = y_t$, respectively.
    The probability measure $\mathbb{P}_{\text{r}}$ on $(\Omega_{\text{r}},\mathcal{F}_{\text{r}})$ is defined such that the collection of random variables $\{(X_{a,t}, \eta_t): a \in [K], t \in [T]\}$ are independent, with $X_{a,t} \sim P_a$ and $\eta_t \sim U(0, 1)$ for all $a \in [K]$ and $t \in [T]$.
    Here, the sequence $(X_{a,1},\ldots,X_{a,T})$ is interpreted as the \emph{reward stack} associated with arm $a$, while $(\eta_t)_{t=1}^T$ provides the internal randomization of policy $\pi$. 
    The sequence of action and reward random variables is then defined recursively. Let $\mathcal{H}_{t-1}^{\pi\text{-r}} \triangleq (A_1^{\pi\text{-r}}, R_1^{\pi\text{-r}}, \dots, A_{t-1}^{\pi\text{-r}}, R_{t-1}^{\pi\text{-r}})$ denote the history up to period $t-1$, and let $N_a^{\pi_i\text{-r}}(t) \triangleq \sum_{r=1}^t \mathbb{I}\{A_r^{\pi_i\text{-r}} = a\}$ denote the cumulative count of pulls for arm $a$ up to period $t$. For each period $t \in [T]$, the action and reward are defined as
    \begin{align*}
        A_t^{\pi\text{-r}} = F^{-1}_{\pi_t(\cdot \mid \mathcal{H}_{t-1}^{\pi\text{-r}})} (\eta_t), \qquad R_t^{\pi\text{-r}} = X_{A_t^{\pi\text{-r}}, N_{A_t^{\pi\text{-r}}}^{\pi\text{-r}}(t)},
    \end{align*}
    where $F^{-1}_{\nu}: [0,1] \to [K]$ denotes the generalized inverse distribution function (or quantile function) associated with the probability distribution $\nu$. Intuitively, each period the learner $\pi$ chooses an arm, it receives the reward on top of the
    corresponding stack that remains unused by this learner.

To illustrate the correspondence to the canonical model, we consider the sample path of policy $\pi_0$ shown in Figure~\ref{fig:AR-arrow-diagram} and provide its reward-stack representation in Figure~\ref{fig:reward-stack}. 
In this figure, each row corresponds to the pre-generated reward stack of a specific arm, with the central number in each cell indicating the potential reward.
The light-gray cells mark the rewards actually observed by the policy, where the small blue number in the top-left corner denotes the period in which the reward is used.
The used rewards are required to coincide with the per-round rewards realized along the trajectory of policy $\pi_0$ in Figure~\ref{fig:AR-arrow-diagram},
while the remaining (unused) rewards can be chosen arbitrarily.
For instance, in Figure~\ref{fig:AR-arrow-diagram}, $\pi_0$ selects arm~1 at $t=1$ and $t=4$, yielding rewards $1$ and $0$, respectively. Accordingly, in Figure~\ref{fig:reward-stack}, the first two cells in the reward stack of arm~1 (with center values $1$ and $0$) are marked as used and labeled by the period indices $1$ and $4$.
\begin{figure}[htbp] 
\caption{An illustration of the sample path of $\pi_0$ in Figure~\ref{fig:AR-arrow-diagram} using the reward-stack model. The integer at the top left corner indicates the period the reward is used.}
\label{fig:reward-stack}
\vspace{0.5em}
\centering
\begin{tikzpicture}[
    pullbox/.style={rectangle,draw,thick,
    fill=black!20,
                minimum width=6mm,minimum height=6mm,
                align=center},
    replaybox/.style={rectangle,draw,thick,fill=orange!40,
                minimum width=6mm,minimum height=6mm,
                align=center},
    box/.style={rectangle,draw,thick,fill=white,
                minimum width=6mm,minimum height=6mm,
                align=center},
    line/.style={thick,black},
    stepnum/.style=
    {anchor=north west, font=\tiny, text=blue!70, inner sep=1pt},
]

\node at (-1,1.3) {$\pi_0$};

\foreach \i/\style/\text/\step in {
    0/pullbox/1/1,
    1/pullbox/0/4,
    2/box/0/,
    3/box/0/,
    4/box/0/,
    5/box/0/,
    6/box/1/,
    7/box/1/,
    8/box/1/,
    9/box/0/
}{
   \node[\style] (r1\i) at (\i*0.6,2.0) {\text};
   \node[stepnum] at (r1\i.north west) {\step};
}

\node[right=0.5em] at (r19.east) {arm 1};

\foreach \i/\style/\text\step in {
    0/pullbox/0/2,
    1/pullbox/0/5,
    2/pullbox/0/9,
    3/box/1/,
    4/box/0/,
    5/box/0/,
    6/box/0/,
    7/box/0/,
    8/box/0/,
    9/box/1/
}{
   \node[\style] (r2\i) at (\i*0.6,1.3) {\text};
   \node[stepnum] at (r2\i.north west) {\step};
}
\node[right=0.5em] at (r29.east) {arm 2};

\foreach \i/\style/\text/\step in {
    0/pullbox/1/3,
    1/pullbox/1/6,
    2/pullbox/1/7,
    3/pullbox/0/8,
    4/pullbox/0/10,
    5/box/1/,
    6/box/1/,
    7/box/0/,
    8/box/0/,
    9/box/0/ 
}{
   \node[\style] (r3\i) at (\i*0.6,0.6) {\text};
   \node[stepnum] at (r3\i.north west) {\step};
}
\node[right=0.5em] at (r39.east) {arm 3};

\node[pullbox,label=right:{Used reward}] at (8.5, 1.8) {};
\node[box,label=right:{Unused reward}] at (8.5, 1.0) {};
\end{tikzpicture}
\end{figure}

In Section~4.6 of \cite{lattimore2020bandit}, the authors assert the equivalence between the reward-stack model and the canonical model (with the proof left as an exercise), in the sense that the sequence $(A_1^{\pi}, R_1^{\pi}, \ldots, A_T^{\pi}, R_T^{\pi})$ on $(\Omega, \mathcal{F}, \mathbb{P}^{\pi})$ has the same distribution as $(A_1^{\pi\text{-r}}, R_1^{\pi\text{-r}}, \ldots, A_T^{\pi\text{-r}}, R_T^{\pi\text{-r}})$ on $(\Omega_{\text{r}}, \mathcal{F}_{\text{r}}, \mathbb{P}_{\text{r}})$. 
Therefore, the reward-stack model can serve as an alternative framework for analyzing bandit policies.
Moreover, a key advantage of the reward-stack model is that it decouples environmental stochasticity from action selection by treating rewards as pre-generated sequences, where the policy merely reveals the realization of these fixed outcomes. This perspective often renders the theoretical analysis more tractable.

\paragraph{Canonical model of the AR design.} Having reviewed the canonical model and reward-stack model for the standard bandit experiment, we now turn to our AR experiment.
It is a natural choice to establish a corresponding canonical model to accommodate the joint sequence 
$(\mathcal{H}_T^{\pi_0\text{-AR}}, \mathcal{H}_T^{\pi_1\text{-AR}})$.
We take the sample space to be $\Omega' \triangleq ([K]\times\mathbb R)^{2T}$, 
equipped with the Borel $\sigma$-algebra $\mathcal F' \triangleq \mathcal{B}(\Omega')$. 
For any sample path $\omega' = (a_1^0, r_1^0, \ldots, a_T^0, r_T^0, a_1^1, r_1^1, \ldots, a_T^1, r_T^1) \in \Omega'$, we define the action and reward random variables via the coordinate projections:
\begin{equation*}
    A_t^{\pi_i\text{-AR}}(\omega') =a_t^i \quad \text{ and }  \quad R_t^{\pi_i\text{-AR}}(\omega') =r_t^i,
\end{equation*}
for all $i \in \{0, 1\}$ and $t\in [T]$.
Here, the superscript $\pi_i\text{-AR}$ in $A_t^{\pi_i\text{-AR}}$ and $R_t^{\pi_i\text{-AR}}$ is used to  indicate that their distributions rely on a specific probability measure on $(\Omega', \mathcal{F}')$ induced by $\pi_0$ and $\pi_1$, while ensuring consistency with the notation
introduced in Algorithm~\ref{alg:AR}.
Indeed, there exists a probability measure $\mathbb{P}^{\pi_{01}}$ on $(\Omega',\mathcal F')$ such that for each $t \in [T]$, the following statements hold:
\begin{itemize}
    \item [(i)] The action $A_t^{\pi_0\text{-AR}}$ and reward $R_t^{\pi_0\text{-AR}}$ collected in Phase~1 of Algorithm~\ref{alg:AR} satisfy
    \begin{equation}\label{eq:A-R-1-AR-conditional-dist}
        A_t^{\pi_0\text{-AR}} \mid \mathcal H_{t-1}^{\pi_0\text{-AR}} \sim (\pi_0)_t(\cdot \mid \mathcal H_{t-1}^{\pi_0\text{-AR}}) \quad \text{ and } \quad R_t^{\pi_0\text{-AR}} \mid (\mathcal H_{t-1}^{\pi_0\text{-AR}}, A_t^{\pi_0\text{-AR}}) \sim P_{A_t^{\pi_0\text{-AR}}}.
    \end{equation}
    \item [(ii)] The action $A_t^{\pi_1\text{-AR}}$ collected in Phase~2 of Algorithm~\ref{alg:AR} satisfies
    \begin{equation}\label{eq:A-2-AR-conditional-dist}
        A_t^{\pi_1\text{-AR}} \mid \mathcal (\mathcal{H}_T^{\pi_0\text{-AR}}, \mathcal{H}_{t-1}^{\pi_1\text{-AR}}) \sim (\pi_1)_t(\cdot \mid \mathcal H_{t-1}^{\pi_1\text{-AR}}).
    \end{equation}
    \item [(iii)]  
    Given realizations $h_T^0=(a_1^0, r_1^0, \ldots, a_T^0, r_T^0)$ of $\mathcal H_T^{\pi_0\text{-AR}}$ and $h_{t-1}^1=(a_1^1, r_1^1, \ldots, a_{t-1}^1, r_{t-1}^1)$ of $\mathcal H_{t-1}^{\pi_1\text{-AR}}$, and given the action $A_t^{\pi_1\text{-AR}} = a$, let $n_a^i(s) \triangleq \sum_{r=1}^{s} \mathbb{I}\{a_r^i = a\}$ be the number of times arm $a$ has been selected by policy $\pi_i$ up to period $s$.
    The conditional distribution of the reward $R_t^{\pi_1\text{-AR}}$ collected in Phase~2 of Algorithm \ref{alg:AR} satisfies
    \begin{align}\label{eq:R-2-AR-conditional-dist}
        R_t^{\pi_1\text{-AR}} \mid (\mathcal H_T^{\pi_0\text{-AR}}, \mathcal H_{t-1}^{\pi_1\text{-AR}}, A_t^{\pi_1\text{-AR}})=(h_T^0, h_{t-1}^1, a) \ \sim \
        \begin{cases}
        \delta(r_{s^*}^0) & \text{if } n_a^1(t-1) < n_a^0(T), \\
        P_{a} & \text{otherwise},
        \end{cases}
    \end{align}
    where $s^* \triangleq \inf\{s\in [T]: n_a^0(s) = n_a^1(t-1) + 1\}$, and $\delta(r_{s^*}^0)$ denotes the Dirac distribution concentrated at $r_{s^*}^0$, i.e., $R_t^{\pi_1\text{-AR}}=r_{s^*}^0$ almost surely under this case.
\end{itemize}
We refer to the probability space $(\Omega', \mathcal F', \mathbb P^{\pi_{01}})$ as the \emph{canonical model} for the AR experiment.

\paragraph{Shared-reward-stack model of the AR design.} Note that in the AR experiment, the action and reward generated by the phase-1 policy in each period shape the replayed history, which in turn affects the future actions of the policy deployed in Phase 2.
This intricate, path-dependent coupling renders a direct analysis of the joint action-reward process based on the canonical model challenging.
To address this, we draw inspiration from the decoupling principles of the reward-stack model for the standard bandit experiment and propose a novel \emph{shared-reward-stack model} tailored for the AR experiment.
Given a fixed horizon $T \geq 1$, the shared‑reward-stack model is defined as follows.
\begin{itemize}
    \item First, we take the sample space to be $\Omega_{\text{s}} \triangleq \mathbb{R}^{(K+2) \times T}$, equipped with the Borel $\sigma$-algebra $\mathcal F_{\text{s}} \triangleq \mathcal{B}(\Omega_{\text{s}})$. 
    A generic element of the sample space is denoted by $\omega_{\text{s}} = [(x_{a,t})_{a,t}, (y_t^0)_t, (y_t^1)_t]$. For each $a \in [K]$ and $t \in [T]$, we define the coordinate projections $X_{a,t}: \Omega_{\text{s}} \mapsto \mathbb{R}$, $\eta_t^0: \Omega_{\text{s}} \mapsto \mathbb{R}$ and $\eta_t^1: \Omega_{\text{s}} \mapsto \mathbb{R}$ by $X_{a, t}(\omega_{\text{s}}) = x_{a, t}$, $\eta_t^0(\omega_{\text{s}}) = y_t^0$ and $\eta_t^1(\omega_{\text{s}}) = y_t^1$, respectively.
    The probability measure $\mathbb{P}_{\text{s}}$ on $(\Omega_{\text{s}},\mathcal{F}_{\text{s}})$ is defined such that the collection of random variables $\{(X_{a,t}, \eta_t^0, \eta_t^1): a \in [K], t \in [T]\}$ are independent, with $X_{a,t} \sim P_a$ and $\eta_t^0, \eta_t^1 \sim U(0, 1)$ for all $a \in [K]$ and $t \in [T]$.
    Here, the collection $\{X_{a, t} : a \in [K],\, t \in [T]\}$ serves as the shared reward stacks that are jointly accessible to both policies, while $(\eta_t^0)_{t=1}^T$ and $(\eta_t^1)_{t=1}^T$ govern the internal randomization for policies $\pi_0$ and $\pi_1$, respectively.

    \item 
    Next, we let the two policies interact through the shared reward stacks.
    For each policy $\pi_i$ with $i \in \{0, 1\}$ and each period $t \in [T]$, we denote $\mathcal{H}_{t-1}^{\pi_i\text{-s}} \triangleq \{(A_1^{\pi_i\text{-s}}, R_1^{\pi_i\text{-s}}), \ldots, (A_{t-1}^{\pi_i\text{-s}}, R_{t-1}^{\pi_i\text{-s}})\}$ as the interaction history of policy $\pi_i$ up to period $t-1$
    , and define
    \begin{equation}\label{eq:N_a-s}
        N_a^{\pi_i\text{-s}}(t) \triangleq \sum_{r=1}^t \mathbb{I}\{A_r^{\pi_i\text{-s}} = a\}
    \end{equation}
    as the number of times arm $a$ has been pulled by policy $\pi_i$ up to period $t$. Then,
    the action and corresponding reward for policy $\pi_i$ ($i \in \{0, 1\}$) in period $t \in [T]$ are defined as
    \begin{align}\label{eq:action-reward-stack}
        A_t^{\pi_i\text{-s}} = F^{-1}_{(\pi_i)_t(\cdot \mid \mathcal{H}_{t-1}^{\pi_i\text{-s}})}(\eta_t^{i}), \qquad R_t^{\pi_i\text{-s}} = X_{A_t^{\pi_i\text{-s}}, N_{A_t^{\pi_i\text{-s}}}^{\pi_i\text{-s}}(t)}.
    \end{align}
    That is, given the common underlying reward stacks, each policy reveals rewards from the stacks according to its own action-selection rule, without interfering with each other.

    Now, we revisit the example discussed in Figure \ref{fig:AR-arrow-diagram} from the perspective of the shared-reward-stack model. As illustrated in Figure \ref{fig:shared-reward-stack}, the two policies $\pi_0$ and $\pi_1$ interact with the same underlying reward stacks.
    Within each cell, the blue index in the top-left corner indicates the period in which $\pi_0$ uses that reward, while the orange index in the bottom-left corner indicates the corresponding period for $\pi_1$.
    Dark-gray cells mark rewards used by both policies, light-gray cells mark rewards used by exactly one policy, and unshaded cells remain unused. 
    Therefore, the number of dark-gray cells and the total number of shaded cells matches the numbers of replayed rewards and real-environment interactions in Figure~\ref{fig:AR-arrow-diagram}, respectively.
    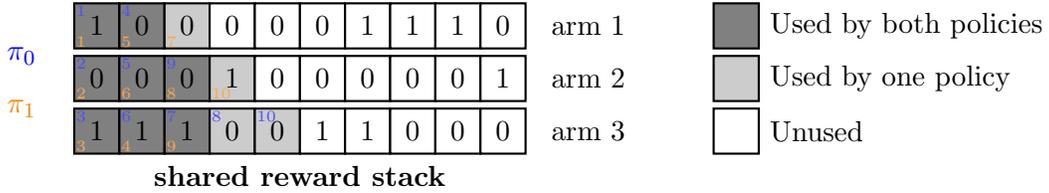
\begin{figure}[htbp] 
    \caption{An illustration of the sample path in Figure~\ref{fig:AR-arrow-diagram} using in the shared-reward-stack model. The integer at the top-left (bottom-right) corner indicates the period in which policy $\pi_0$ ($\pi_1$) uses the reward.}
    \label{fig:shared-reward-stack}
    \vspace{0.5em}
    \centering
    \begin{tikzpicture}[
        pullbox/.style={rectangle,draw,thick,
        fill=black!20,
                    minimum width=6mm,minimum height=6mm,
                    align=center},
        replaybox/.style={rectangle,draw,thick,
        fill=black!50,
                    minimum width=6mm,minimum height=6mm,
                    align=center},
        box/.style={rectangle,draw,thick,fill=white,
                    minimum width=6mm,minimum height=6mm,
                    align=center},
        line/.style={thick,black},
        step1num/.style=
        {anchor=north west, font=\tiny, text=blue!70, inner sep=1pt},
        step2num/.style=
        {anchor=south west, font=\tiny, text=orange!70, inner sep=1pt},
    ]
    
    \node at (-1,1.6) {\textcolor{blue}{$\pi_0$}};
    \node at (-1,0.9) {\textcolor{orange}{$\pi_1$}};
    \node at (2.7,0) {\textbf{shared reward stack}};
    
    \foreach \i/\style/\text/\ustep/\lstep in {
        0/replaybox/1/1/1,
        1/replaybox/0/4/5,
        2/pullbox/0/ /7,
        3/box/0/ /,
        4/box/0/ /,
        5/box/0/ /,
        6/box/1/ /,
        7/box/1/ /,
        8/box/1/ /,
        9/box/0/ /
    }{
       \node[\style] (r1\i) at (\i*0.6,2.0) {\text};
       \node[step1num] at (r1\i.north west) {\ustep};
       \node[step2num] at (r1\i.south west) {\lstep};
       }
    
    \node[right=0.5em] at (r19.east) {arm 1};
    
    \foreach \i/\style/\text/\ustep/\lstep in {
        0/replaybox/0/2/2,
        1/replaybox/0/5/6,
        2/replaybox/0/9/8,
        3/pullbox/1/ /10,
        4/box/0/ /,
        5/box/0/ /,
        6/box/0/ /,
        7/box/0/ /,
        8/box/0/ /,
        9/box/1/ /
    }{
       \node[\style] (r2\i) at (\i*0.6,1.3) {\text};
       \node[step1num] at (r2\i.north west) {\ustep};
       \node[step2num] at (r2\i.south west) {\lstep};
    }
    \node[right=0.5em] at (r29.east) {arm 2};
    
    \foreach \i/\style/\text/\ustep/\lstep in {
        0/replaybox/1/3/3,
        1/replaybox/1/6/4,
        2/replaybox/1/7/9,
        3/pullbox/0/8/,
        4/pullbox/0/10/,
        5/box/1/ /,
        6/box/1/ /,
        7/box/0/ /,
        8/box/0/ /,
        9/box/0/ /
    }{
       \node[\style] (r3\i) at (\i*0.6,0.6) {\text};
       \node[step1num] at (r3\i.north west) {\ustep};
       \node[step2num] at (r3\i.south west) {\lstep};
    }
    \node[right=0.5em] at (r39.east) {arm 3};
    
    \node[replaybox,label=right:{Used by both policies}] at (8.5,2.0) {};
    \node[pullbox,label=right:{Used by one policy}] at (8.5,1.3) {};
    \node[box,label=right:{Unused}] at (8.5,0.6) {};
    \end{tikzpicture}
    \end{figure}
\end{itemize}

To rigorously justify the shared-reward-stack model as an alternative analytical tool to the canonical model for the AR experiment, we establish the distributional equivalence results in Lemma~\ref{lem:bnk-stack-same-dist} and Theorem~\ref{thm:AR-stack-same-dist}.
Lemma~\ref{lem:bnk-stack-same-dist} focuses on the behavior of a single policy in the AR design and the shared‑reward‑stack model.
It states that, for either the control or the treatment policy considered in isolation, the historical action-reward sequence generated in Algorithm~\ref{alg:benchmark} is identically distributed to that constructed in the shared‑reward‑stack model.
In contrast, Theorem~\ref{thm:AR-stack-same-dist} concerns the coupled behavior of two policies under the canonical model and the shared‑reward‑stack model.
It establishes that the joint distribution of the $2T$-period trajectory (including replayed rewards) obtained in the AR experiment (Algorithm~\ref{alg:AR}) is identical to that of the action-reward trajectory constructed in the shared-reward-stack model.

\begin{lemma}\label{lem:bnk-stack-same-dist}
    Fix a horizon $T \geq 1$. {For any $i \in \{0, 1\}$, let $\{(A_t^{\pi_i}, R_t^{\pi_i})\}_{t=1}^T$ be the action-reward trajectory  defined in \eqref{eq:A-R-single-conditional-dist} under the canonical model of the na\"{\i}ve design, and let $\{(A_t^{\pi_i\text{-s}}, R_t^{\pi_i\text{-s}})\}_{t=1}^T$ be the action-reward trajectory constructed in \eqref{eq:action-reward-stack} under the shared-reward-stack model of the AR design.}
    Then, for each policy $\pi_i$, the two trajectories have the same joint distribution:
    \begin{align}\label{eq:bnk-stack-same-dist}
        (A_1^{\pi_i}, R_1^{\pi_i}, \dots, A_T^{\pi_i}, R_T^{\pi_i}) \stackrel{\text{d}}{=} (A_1^{\pi_i\text{-s}}, R_1^{\pi_i\text{-s}}, \dots, A_T^{\pi_i\text{-s}}, R_T^{\pi_i\text{-s}}).
    \end{align}
\end{lemma}

\begin{theorem}\label{thm:AR-stack-same-dist}
    Fix a horizon $T \geq 1$. 
    {For any $i \in \{0, 1\}$,
    let $\{(A_t^{\pi_i\text{-AR}}, R_t^{\pi_i\text{-AR}})\}_{t=1}^T$ be the action-reward trajectory defined in \eqref{eq:A-R-1-AR-conditional-dist}--\eqref{eq:R-2-AR-conditional-dist} under the canonical model of the AR design, and let $\{(A_t^{\pi_i\text{-s}}, R_t^{\pi_i\text{-s}})\}_{t=1}^T$ be the action-reward trajectory constructed in \eqref{eq:action-reward-stack} under the shared-reward-stack model of the AR design.}
    Then, the joint distribution of the action-reward trajectories of both policies in the AR design is equal under the canonical model and the shared-reward-stack model:
    \begin{equation}\label{eq:AR-stack-same-dist}
    \begin{aligned}
     &(A_1^{\pi_0\text{-AR}}, R_1^{\pi_0\text{-AR}}, \dots, A_T^{\pi_0\text{-AR}}, R_T^{\pi_0\text{-AR}}, A_1^{\pi_1\text{-AR}}, R_1^{\pi_1\text{-AR}}, \dots, A_T^{\pi_1\text{-AR}}, R_T^{\pi_1\text{-AR}}) \\
     \stackrel{\text{d}}{=} 
     {}& (A_1^{\pi_0\text{-s}}, R_1^{\pi_0\text{-s}}, \dots, A_T^{\pi_0\text{-s}}, R_T^{\pi_0\text{-s}}, A_1^{\pi_1\text{-s}}, R_1^{\pi_1\text{-s}}, \dots, A_T^{\pi_1\text{-s}}, R_T^{\pi_1\text{-s}}). 
    \end{aligned}
    \end{equation}
\end{theorem}
While the results of Lemma~\ref{lem:bnk-stack-same-dist} and Theorem~\ref{thm:AR-stack-same-dist} align with intuition, their proofs are not straightforward.
Indeed, they rely on the stopping time and martingale properties established later in Proposition~\ref{prop:stopping-time-martingale} for the shared-reward-stack model.
The detailed proofs of Lemma~\ref{lem:bnk-stack-same-dist} and Theorem~\ref{thm:AR-stack-same-dist} are provided in Appendix~\ref{app:lem-bnk-stack-same-dist} and Appendix~\ref{app:thm-AR-stack-same-dist}, respectively.

The following Corollary~\ref{cor:mean-var-equal} is an immediate consequence of Lemma~\ref{lem:bnk-stack-same-dist} and Theorem~\ref{thm:AR-stack-same-dist}.
Since it will be repeatedly used in proving the statistical properties of the AR estimator in Section~\ref{sec:theoretical_properties}, we state it here for convenience, whose proof is provided in Appendix~\ref{app:cor-mean-var-equal}. 
\begin{corollary}\label{cor:mean-var-equal}
    Fix a horizon $T \geq 1$. Then, for any $i \in \{0, 1\}$ and $a \in [K]$, 
    \begin{align*}
        \E[N_a^{\pi_i\text{-AR}}(T)] &= \E[N_a^{\pi_i\text{-s}}(T)] =\E[N_a^{\pi_i}(T)], \\
        \Var(N_a^{\pi_i\text{-AR}}(T)) &= \Var(N_a^{\pi_i\text{-s}}(T)) =\Var(N_a^{\pi_i}(T)).
    \end{align*}
    Moreover, 
    \begin{align*}
        \E\left[\sum_{t=1}^T R_t^{\pi_i\text{-AR}} \right] = \E\left[\sum_{t=1}^T R_t^{\pi_i\text{-s}} \right] = \E\left[\sum_{t=1}^T R_t^{\pi_i} \right].
    \end{align*}
\end{corollary}

In light of the distributional equivalence,
we can conduct our theoretical analysis of the AR experiment through the lens of the shared-reward-stack model.
Recalling the four statistical properties of the AR design outlined in Section~\ref{sec:AR_Design},
we note that under the shared-reward-stack model, the symmetry and sample-efficiency properties of the AR experiment, as well as the unbiasedness of the AR estimator, are amenable to standard analytical arguments.
The most technical aspect, however, lies in quantifying the variance-reduction effect of the AR estimator relative to the na\"{\i}ve estimator. 
To this end, we develop a novel analytical approach, which involves the systematic use of stopping time and martingale theory for analyzing intertwined bandit policies---a significant methodological innovation for bandit problems.
We lay the basis of this framework by identifying the fundamental stopping‑time and martingale structures within the shared‑reward‑stack model. 
A key technical point in this process lies in 
the construction of \emph{an appropriate filtration such that 
$N_a^{\pi_i\text{-s}}(T)$ is a stopping time}---a property that does not hold under the experiment's natural filtration $\{\sigma(\mathcal{H}_t^{\pi_i\text{-s}}): t \in [T]\}$.
Indeed, the event $N_a^{\pi_i\text{-s}}(T) \leq t$ is certainly not based on the information available at $t$ in the natural filtration associated with a multi-armed bandit process.

To give the intuition of this novel filtration,
suppose the internal randomness of policy $\pi_i$ for the whole process and the complete reward stacks of all the arms \emph{except arm $a$} are given. 
Then once the first $t$ rewards in the stack of arm $a$ have been revealed, one can determine whether $N_a^{\pi_i\text{-s}}(T) \leq t$.
For instance, Figure~\ref{fig:stopping-time} presents a partial reward-stack realization in which the stack of arm $2$ is masked from its 4th entry onward.
This partial reward-stack realization together with the realized internal randomness $\{y_t^i: t\in[T]\}$ suffices to determine whether $N_2^{\pi_i\text{-s}}(10)\leq 3$. 
Indeed, there are two possible scenarios after running $\pi_i$ for $T=10$ periods.
For the first scenario, policy $\pi_i$ selects arm 2 at most three times within the 10 periods, in which case one can simulate the entire trajectory using the revealed information in Figure~\ref{fig:stopping-time} and thus conclude that the event $\{N_2^{\pi_i\text{-s}}(10)\leq 3\}$ occurs. 
For the second scenario, policy $\pi_i$ attempts to select arm 2 for the 4th time in some period.
At this point, the masked reward is required, and one can conclude immediately that $N_2^{\pi_i\text{-s}}(10)> 3$, \emph{without actually revealing the masked rewards}.
\begin{figure}[htbp] 
\caption{An illustration of a partially revealed reward stack. The masked cells indicate rewards that are treated as unknown.}
\label{fig:stopping-time}
\vspace{0.5em}
\centering
\begin{tikzpicture}[
    unknownbox/.style={rectangle,draw,thick,
    pattern=north east lines, pattern color=black!90,
                minimum width=6mm,minimum height=6mm,
                align=center},
    replaybox/.style={rectangle,draw,thick,fill=orange!40,
                minimum width=6mm,minimum height=6mm,
                align=center},
    box/.style={rectangle,draw,thick,fill=white,
                minimum width=6mm,minimum height=6mm,
                align=center},
    line/.style={thick,black},
    stepnum/.style=
    {anchor=north west, font=\tiny, text=blue!70, inner sep=1pt},
]

\foreach \i/\style/\text/\step in {
    0/box/1/,
    1/box/0/,
    2/box/0/,
    3/box/0/,
    4/box/0/,
    5/box/0/,
    6/box/1/,
    7/box/1/,
    8/box/1/,
    9/box/0/
}{
   \node[\style] (r1\i) at (\i*0.6,2.0) {\text};
   \node[stepnum] at (r1\i.north west) {\step};
}

\node[right=0.5em] at (r19.east) {arm 1};

\foreach \i/\style/\text\step in {
    0/box/0/,
    1/box/0/,
    2/box/0/,
    3/unknownbox//,
    4/unknownbox//,
    5/unknownbox//,
    6/unknownbox//,
    7/unknownbox//,
    8/unknownbox//,
    9/unknownbox//
}{
   \node[\style] (r2\i) at (\i*0.6,1.3) {\text};
   \node[stepnum] at (r2\i.north west) {\step};
}
\node[right=0.5em] at (r29.east) {arm 2};

\foreach \i/\style/\text/\step in {
    0/box/1/,
    1/box/1/,
    2/box/1/,
    3/box/0/,
    4/box/0/,
    5/box/1/,
    6/box/1/,
    7/box/0/,
    8/box/0/,
    9/box/0/ 
}{
   \node[\style] (r3\i) at (\i*0.6,0.6) {\text};
   \node[stepnum] at (r3\i.north west) {\step};
}
\node[right=0.5em] at (r39.east) {arm 3};

\end{tikzpicture}
\end{figure}

Therefore, we define the filtration $\{\mathcal{F}_t^{T,a}\}_{t \in \mathbb{N}}$ specific to arm $a$ such that $\mathcal{F}_t^{T,a}$ captures all the information available after the first 
$t$ rewards in the stack of arm $a$ have been observed, together with the complete reward stacks of other arms and the internal randomness of both policies. Formally, we define (recall that $X_{a,t}$ is the $t$-th random reward of arm $a$ in the stack):
\begin{equation}\label{eq:def-filtration}
\begin{aligned}
    &\mathcal{F}_0^{T, a}
    \triangleq \sigma(X_{b,r},\, \eta_r^0, \,\eta_r^1,\, b \neq a,\, r \in [T]), \\
    &\mathcal{F}_t^{T, a}
    \triangleq \mathcal{F}_0^{T, a} \vee \sigma(X_{a,r},\, r \in [t])
    \quad \text{ for } t \in [T], \\
    &\mathcal{F}_t^{T, a}
    \triangleq \mathcal{F}_T^{T, a} \quad \text{ for } t > T.
\end{aligned}
\end{equation}
In fact, for $i\in \{0, 1\}$ and $a \in [K]$, the natural filtration at period $t$, i.e., $\sigma(\mathcal{H}_t^{\pi_i\text{-s}})$ is contained in $\mathcal{F}_{N_a^{\pi_i\text{-s}}(t)}^{T, a}$, the $\sigma$-algebra of events prior to the stopping time $N_a^{\pi_i\text{-s}}(t)$, for all $t \in [T]$.

{Based on this new filtration, we now introduce three stochastic processes $\{I_t^{T, a}(B),\,t \in \mathbb{N}\}$, $\{S_t^{T, a},\,t \in \mathbb{N}\}$ and $\{V_t^{T, a},\,t \in \mathbb{N}\}$ that constitute the main objects of our analysis.
For any $B \in \mathcal{B}(\mathbb{R})$, set $I_0^{T, a}(B) = S_0^{T, a} = V_0^{T, a} = 0$, and define for $t \in [T]$,
\begin{align}\label{eq:def-I-S-V}
    I_t^{T, a}(B) \triangleq \sum_{r=1}^t (\mathbb{I}\{X_{a, r} \in B\} - P_a(B)), \quad S_t^{T, a} \triangleq \sum_{r=1}^t (X_{a, r} - \mu_a), \quad V_t^{T, a} \triangleq (S_t^{T, a})^2 - t \sigma_a^2.
\end{align}
To facilitate the subsequent stopping-time arguments, we extend the time index set of these processes to $\mathbb{N}$
by setting $I_t^{T, a}(B) = I_T^{T, a}(B)$, $S_t^{T, a} = S_T^{T, a}$, and $V_t^{T, a} = V_T^{T, a}$ for all $t > T$.
{It is worth noting that these three processes are adapted to the new filtration defined in \eqref{eq:def-filtration}, but not adapted to the experiment's natural filtration $\{\sigma(\mathcal{H}_t^{\pi_i\text{-s}}): t \in [T]\}$.}
Moreover, $I_t^{T,a}(B)$ is a centered empirical count, which tracks the deviation between the number of samples among $\{X_{a,1},\ldots,X_{a,t}\}$ that fall in $B$ and its expectation $tP_a(B)$. The process $\{I_t^{T, a}(B),\,t \in \mathbb{N}\}$ arises in our proof of the distributional equivalence in Theorem~\ref{thm:AR-stack-same-dist}. In particular, a key step in this proof is to establish
\begin{align*}
    \E[I_{N_a^{\pi_1\text{-s}}(t) \vee N_a^{\pi_0\text{-s}}(T)}^{T, a}(B) \mid \mathcal{F}_{N_a^{\pi_1\text{-s}}(t-1) \vee N_a^{\pi_0\text{-s}}(T)}^{T, a}] = I_{N_a^{\pi_1\text{-s}}(t-1) \vee N_a^{\pi_0\text{-s}}(T)}^{T, a}(B).
\end{align*} 
This result necessitates the martingale property of $\{I_t^{T, a}(B),\,t \in \mathbb{N}\}$, which enables the application of the optional sampling theorem.
The process $\{S_t^{T, a},\,t \in \mathbb{N}\}$ is the zero-mean random walk induced by the reward stack of arm $a$, and $\{V_t^{T, a},\,t \in \mathbb{N}\}$ is defined as the associated quadratic correction. 
These two processes play a crucial role in establishing the variance-reduction property of the AR estimator. Specifically, in this analysis, it is essential to evaluate the variance of the arm-wise centered cumulative rewards, $\operatorname{Var}(S_{N_a^{\pi_i\text{-s}}(T)}^{T, a})$.
By leveraging the martingale structure of $\{S_t^{T, a},\,t \in \mathbb{N}\}$ and $\{V_t^{T, a},\,t \in \mathbb{N}\}$, we can express this quantity explicitly in terms of the expected pull count $\E[N_a^{\pi_i\text{-s}}(T)]$, thereby rendering the subsequent analysis tractable.
}

We state the stopping-time and martingale results in Proposition~\ref{prop:stopping-time-martingale}. The proof is provided in Appendix~\ref{app:lem-stopping-time-martingale}.
\begin{proposition}\label{prop:stopping-time-martingale}
Fix a horizon $T \ge 1$ and an arm $a \in [K]$. The following statements hold:
\begin{itemize}
    \item [(i)] For each policy $\pi_i$ with $i \in \{0, 1\}$ and any $r \in [T]$, $N_a^{\pi_i\text{-s}}(r)$ defined in \eqref{eq:N_a-s} is a bounded stopping time with respect to the filtration $\{\mathcal{F}_t^{T, a}\}_{t \in \mathbb{N}}$ defined in \eqref{eq:def-filtration}.
    \item [(ii)]
    The processes $\{I_t^{T, a}(B),\,t \in \mathbb{N}\}$, $\{S_t^{T, a},\,t \in \mathbb{N}\}$ and $\{V_t^{T, a},\,t \in \mathbb{N}\}$, defined in \eqref{eq:def-I-S-V}, 
    are $\{\mathcal{F}_t^{T,a}\}_{t \in \mathbb{N}}$-martingales.
\end{itemize}
\end{proposition}

\section{Statistical Properties of the AR Design}\label{sec:theoretical_properties}

In this section, we establish the statistical properties of the AR design and the associated estimator, as outlined at the end of Section~\ref{sec:AR_Design}. We will prove that the AR design is symmetric, highly sample-efficient, and that the AR estimator is unbiased and achieves significant asymptotic variance reduction compared to the na\"{\i}ve estimator.

\subsection*{Symmetry}
A key feature of a fair comparison framework is symmetry. The choice of which policy to deploy first in the real environment should not affect the final outcome. The AR design satisfies this property. 
This property arises naturally from the shared-reward-stack model, in which the two policies play perfectly symmetric and interchangeable roles. 
Moreover, as established in Theorem~\ref{thm:AR-stack-same-dist}, the canonical model underlying the AR experiment is distributionally equivalent to the shared-reward-stack model. 
Therefore, the symmetry property of the AR design follows immediately. We state this result in Theorem~\ref{thm:symmetry}, the proof of which is omitted.
\begin{theorem}[Symmetry]\label{thm:symmetry}
    Let $\thetahat{AR}(T)$ be the AR estimator defined in \eqref{eq:def-AR-estimator} where $\pi_0$ is deployed first in Algorithm~\ref{alg:AR}. Let $\thetahat{AR}'(T)$ be the estimator constructed by applying Algorithm~\ref{alg:AR} after swapping the roles of $\pi_0$ and $\pi_1$--that is, deploying $\pi_1$ first and having $\pi_0$ replay on its data. Then, $\thetahat{AR}(T)$ and $\thetahat{AR}'(T)$ are identically distributed.
\end{theorem}
This symmetry ensures that the AR experiment is robust and the comparison is fair, irrespective of the deployment order.

\subsection*{Sample Efficiency}
The AR method is designed to minimize costly interactions with the real environment. Unlike a na\"{\i}ve design that requires $2T$ interactions, the AR design reuses data, leading to substantial savings. The following theorem quantifies the expected number of real-world arm pulls.

\begin{theorem}\label{thm:AR-pull-times}
    For a fixed horizon $T \geq 1$, let $n^{\pi_i}(T) \triangleq \sum_{a \neq a^*, \, a \in [K]} \mathbb{E}[N_a^{\pi_i}(T)]$ denote the expected number of suboptimal arm pulls by policy $\pi_i$ for each $i \in \{0, 1\}$. Furthermore, let $n^{\text{e-AR}}(T)$ and $n^{\text{r-AR}}(T)$ be the expected number of real-environment interactions and replayed rewards in Algorithm~\ref{alg:AR}, respectively, as defined in \eqref{eq:def-exp-num-env-interactions-replay}. Then, $n^{\text{e-AR}}(T)$ satisfies
    \begin{align*}
        n^{\text{e-AR}}(T) \leq \min\{T + n^{\pi_0}(T) + n^{\pi_1}(T),\, 2T-n^{\text{r-AR}}(T)\}.
    \end{align*}
\end{theorem}
\emph{The proof of Theorem~\ref{thm:AR-pull-times} is provided in Appendix~\ref{app:thm-AR-pull-times}.}

\begin{remark}
The expected number of suboptimal-arm pulls, $n^{\pi_i}(T)$, is directly related to the expected regret $\E[Reg(\pi_i; T)]$ of policy $\pi_i$. Specifically,
\begin{align*}
    \frac{1}{\max_{a \in [K]} \Delta_a} \E[Reg(\pi_i; T)] \leq n^{\pi_i}(T) \leq \frac{1}{\min_{a \in [K] \setminus \{a^*\}} \Delta_a} \E[Reg(\pi_i; T)].
\end{align*}
This implies that if both bandit policies under comparison achieve $O(\log T)$ regret, which is common for efficient algorithms like UCB and Thompson Sampling, then $n^{\pi_i}(T) = O(\log T)$. Consequently, the expected number of real-environment interactions in the AR design is $T + O(\log T)$, a clear improvement compared with the $2T$ required by the na\"{\i}ve design.
\end{remark}

\subsection*{Unbiasedness}
A fundamental requirement for any reliable estimator is unbiasedness. We demonstrate that the AR estimator, despite its complex data-sharing mechanism, correctly estimates the true ATE on average, which is formally stated as Theorem~\ref{thm:unbias} below. It is a direct consequence of Corollary~\ref{cor:mean-var-equal}, and the proof is therefore omitted.
\begin{theorem}[Unbiasedness]\label{thm:unbias}
    For a fixed horizon $T \geq 1$, the estimator $\thetahat{AR}(T)$ defined in \eqref{eq:def-AR-estimator} is unbiased for the true parameter $\theta(T)$. That is, $\E[\thetahat{AR}(T)] = \theta(T)$.
\end{theorem}

\subsection*{Asymptotic Variance Reduction}
The main advantage of the AR method lies in its superior statistical properties. 
We now present our main result: the AR design achieves significant asymptotic variance reduction relative to the na\"{\i}ve design. 
The key insight is that the shared reward stack induces strong positive correlation between the two cumulative reward estimators, thereby offsetting a large portion of the variance.

To formalize this argument, we first analyze the variance structure of both the na\"{\i}ve and AR estimators. 
For the na\"{\i}ve estimator $\thetahat{b}(T)$, since the two policies $\pi_0$ and $\pi_1$ are executed independently, its variance is therefore given by
\begin{equation}\label{eq:var-naive-estimator}
    \Var(\thetahat{b}(T)) = \Var\bigg(\sum_{t=1}^{T} R_t^{\pi_0}\bigg) + \Var\bigg(\sum_{t=1}^{T} R_t^{\pi_1}\bigg).
\end{equation}
The variance of the proposed AR estimator $\thetahat{AR}(T)$, which incorporates data sharing between the two policies, can be written as
\begin{equation}\label{eq:var-AR-estimator}
    \Var(\thetahat{AR}(T)) = \Var\bigg(\sum_{t=1}^{T} R_t^{\pi_0\text{-AR}}\bigg) + \Var\bigg(\sum_{t=1}^{T} R_t^{\pi_1\text{-AR}}\bigg) - 2 \operatorname{Cov}\bigg(\sum_{t=1}^{T} R_t^{\pi_0\text{-AR}},\,\sum_{t=1}^{T} R_t^{\pi_1\text{-AR}}\bigg).
\end{equation}
To facilitate a direct and rigorous comparison between \eqref{eq:var-naive-estimator} and \eqref{eq:var-AR-estimator}, we require a common analytical framework. The crucial bridge that connects these two formulations is the distributional equivalence to the shared-reward-stack model, established in Lemma~\ref{lem:bnk-stack-same-dist} and Theorem~\ref{thm:AR-stack-same-dist}.
When we denote $C_{ij}(T) = \operatorname{Cov}(\sum_{t=1}^{T} R_t^{\pi_i\text{-s}},\,\sum_{t=1}^{T} R_t^{\pi_j\text{-s}})$ for $i,\, j \in \{0, 1\}$, it follows from Lemma~\ref{lem:bnk-stack-same-dist} and Theorem~\ref{thm:AR-stack-same-dist} that
\begin{gather*}
    \Var(\thetahat{b}(T)) = C_{00}(T) + C_{11}(T), \\
    \Var(\thetahat{AR}(T)) = C_{00}(T) + C_{11}(T) - 2C_{01}(T).
\end{gather*}
Once can see that the variance reduction can be achieved when $C_{01}(T)>0$.
In general, the analysis of $C_{01}(T)$ for a given finite $T$ for arbitrary multi-armed bandit policies can be challenging. 
Instead, our subsequent analysis focuses on the asymptotic behavior of the covariance. We show that under certain assumptions on the two compared policies, the covariance $C_{01}(T)$ grows linearly with the horizon $T$, which implies that the covariance will be positive when $T$ is sufficiently large. 
In fact, we establish a more general result: 
\begin{align}\label{eq:C_ij/T}
    \lim_{T \rightarrow \infty} \frac{1}{T} C_{00} = \lim_{T \rightarrow \infty} \frac{1}{T} C_{11} = \lim_{T \rightarrow \infty} \frac{1}{T} C_{01}=\sigma_{a^*}^2.
\end{align}
Building on this result, we further characterize the asymptotic variances of $\thetahat{b}(T)$ and $\thetahat{AR}(T)$ themselves, as stated in Theorem~\ref{thm:asy_var}.
\begin{theorem}[Asymptotic Variance Reduction]
\label{thm:asy_var}
Let $\pi_0$ and $\pi_1$ be two bandit policies such that, for each $i \in \{0, 1\}$ and every suboptimal arm $a$,
\begin{equation}\label{eq:var-assump}
    \E[N_a^{\pi_i}(T)] = o(T) \quad \text{and} \quad
    \Var(N_a^{\pi_i}(T)) = o(T).
\end{equation}
Then, the variances of the estimators $\thetahat{b}(T)$ and $\thetahat{AR}(T)$ satisfy
\begin{align*}
    \lim_{T \rightarrow \infty} \frac{1}{T}\Var(\thetahat{b}(T)) = 2 \sigma_{a^*}^2, \qquad
    \lim_{T \rightarrow \infty} \frac{1}{T}\Var(\thetahat{AR}(T)) = 0,
\end{align*}
where $\sigma_{a^*}^2$ denotes the reward variance associated with the optimal arm $a^*$.
\end{theorem}

\begin{remark}
Theorem~\ref{thm:asy_var} reveals that, the variance of the na\"{\i}ve estimator $\thetahat{b}(T)$ admits a leading term of order $2\sigma_{a^*}^2 T$.
In contrast, our proposed AR estimator $\thetahat{AR}(T)$ effectively eliminates this leading term, resulting in a variance resulting in a variance that grows at a sub-linear rate. 
This constitutes an order-of-magnitude improvement, which is critical for achieving higher statistical precision in comparing the performances of two bandit policies.
\end{remark}

\begin{remark}[Extension to Bayesian Bandits]
    While our theoretical analysis focuses on a fixed, unknown multi-armed bandit instance, the AR framework extends naturally to a Bayesian setting, where the bandit instance $\xi$ is assumed to be drawn from a prior distribution $\mathcal{P}$.
    This formulation is particularly relevant in recommender systems, where policies are often evaluated across heterogeneous environments such as different user segments or newly arriving item sets, each of which can be viewed as an instance sample.
    In this setting, the goal is to estimate the expected ATE over the prior: $\theta_{\text{Bayes}}(T) \triangleq \E_{\xi \sim \mathcal{P}}[ \theta(T; \xi)]$.
    Given $M$ i.i.d instances $\xi_1, \ldots, \xi_{M} \sim \mathcal{P}$, one can independently apply the AR experiment to each realized instance $\xi_i \sim \mathcal{P}$ to obtain the instance-specific estimate $\widehat{\theta}_{\text{AR}}(T; \xi_i)$, and then define the  Bayesian AR estimator as $\widehat{\theta}_{\text{Bayes-AR}}(T) = \frac{1}{M} \sum_{i=1}^{M} \widehat{\theta}_{\text{AR}}(T; \xi_i)$.
    The statistical properties of the AR design established in this section for a fixed instance can be adapted to the Bayesian setting through the law of total expectation and the law of total variance.
\end{remark}

\begin{proof}{Proof Sketch of Theorem~\ref{thm:asy_var}}
The complete proof of Theorem~\ref{thm:asy_var} is provided in Appendix~\ref{app:thm-asy_var}. 
The core of the argument is to establish the asymptotic covariance result in~\eqref{eq:C_ij/T}. Once this is established, the results of Theorem~\ref{thm:asy_var} follow directly by substituting this limit into the variance formulas \eqref{eq:var-naive-estimator} and \eqref{eq:var-AR-estimator}.
Our proof for \eqref{eq:C_ij/T} proceeds in two main steps.

\paragraph{Step 1: variance-covariance decomposition.}
We decompose $C_{ij}(T)$ for $i, j \in \{0, 1\}$ into a sum of covariances or variances involving the number of arm pulls $N_a^{\pi_i\text{-s}}(T)$ and the centered cumulative rewards $S_{N_a^{\pi_i\text{-s}}(T)}^{T, a}$ for each arm, where $S_t^{T, a}$ is defined as in \eqref{eq:def-I-S-V}.
Specifically,
\begin{equation}\label{eq:pf-sketch-decomp}
\begin{aligned}
    C_{ij}(T)
    &= \sum_{a=1}^K \sum_{b=1}^K \operatorname{Cov}( S_{N_a^{\pi_i\text{-s}}(T)}^{T, a} ,\,S_{N_b^{\pi_j\text{-s}}(T)}^{T, b})  + \sum_{a=1}^K \sum_{b=1}^K \mu_b \operatorname{Cov}(N_a^{\pi_i\text{-s}}(T),\,N_b^{\pi_j\text{-s}}(T)) \\
    &\hspace{1em} + \sum_{a=1}^K \sum_{b=1}^K \mu_a \operatorname{Cov}( N_a^{\pi_i\text{-s}}(T),\,S_{N_b^{\pi_j\text{-s}}(T)}^{T, b}) + \sum_{a=1}^K \sum_{b=1}^K \mu_b \operatorname{Cov}(S_{N_a^{\pi_i\text{-s}}(T)}^{T, a},\,N_b^{\pi_j\text{-s}}(T)).
\end{aligned}
\end{equation}

\paragraph{Step 2: asymptotic analysis.} We analyze the asymptotic order of each term in the decomposition \eqref{eq:pf-sketch-decomp} and show that the only term that contributes to the linear growth is the covariance of the centered rewards from the optimal arm $a^*$. That is,
\begin{align}\label{eq:pf-sketch-optimal-arm}
    \lim_{T \rightarrow \infty }\frac{1}{T}\operatorname{Cov}(S_{N_{a^*}^{\pi_i\text{-s}}(T)}^{T, a^*} ,\,S_{N_{a^*}^{\pi_j\text{-s}}(T)}^{T, a^*}) = \sigma_{a^*}^2.
\end{align}
This equation highlights the benefit of coupling brought by the shared reward stacks: for the treatment and control policies ($i=0$ and $j=1$), their total rewards from the optimal arm are not independent as in the na\"ve design; rather, their correlation converges to one.
All other covariance terms on the right-hand side of \eqref{eq:pf-sketch-decomp} are shown to be of order $o(T)$.
This argument relies critically on a corollary of the stopping time and martingale results established in Proposition~\ref{prop:stopping-time-martingale}. This result is stated as Corollary~\ref{cor:moment-eq}, and its proof is provided in Appendix~\ref{app:cor-moment-eq}. The establishment of \eqref{eq:pf-sketch-optimal-arm} follows from repeated applications of the Doob's optional stopping theorem to the martingale $\{S_t^{T, a^*},\, t \in \mathbb{N}\}$, thanks to the novel filtration leveraging the shared-reward-stack model constructed before Proposition~\ref{prop:stopping-time-martingale},
together with the result of Corollary~\ref{cor:moment-eq}. 
The remaining terms in \eqref{eq:pf-sketch-decomp} are shown to be of order $o(T)$ by combining Corollary~\ref{cor:moment-eq} with the Cauchy–Schwarz inequality.
\end{proof}
\begin{corollary}\label{cor:moment-eq}
    For a fixed horizon $T \geq 1$ and any arm \(a \in [K]\), the following equalities hold:
    \begin{align*}
    \E[ S_{N_a^{\pi_i\text{-s}}(T)}^{T, a} ] = 0, \quad &\Var(S_{ N_a^{\pi_i\text{-s}}(T)}^{T, a}) = \sigma_a^2 \E[N_a^{\pi_i\text{-s}}(T)] \quad \text{ for } i \in \{0, 1\}, \\
    \E[ S_{N_a^{\pi_0\text{-s}}(T) \wedge N_a^{\pi_1\text{-s}}(T)}^{T, a}] = 0, \quad &\Var( S_{N_a^{\pi_0\text{-s}}(T) \wedge N_a^{\pi_1\text{-s}}(T)}^{T, a}) = \sigma_a^2 \E[N_a^{\pi_0\text{-s}}(T) \wedge N_a^{\pi_1\text{-s}}(T)].
    \end{align*}
\end{corollary}
The proof of Theorem~\ref{thm:asy_var} represents the ultimate application of the analytical framework established in Section~\ref{sec:analytical_framework}, fully exploiting the distributional equivalence to the shared-reward-stack model and the martingale theory. 
The novel use of the filtration constructed specifically for the shared-reward-stack model may be of independent interest in analyzing other multi-armed bandit problems.

Next, we examine the technical assumption in Theorem~\ref{thm:asy_var}, which consists of two conditions.
The first condition, $\E[N_a^{\pi_i}(T)] = o(T)$ for all suboptimal arm $a$, is rather mild, as it is automatically satisfied by any bandit algorithm with sub-linear regret.
The second condition, $\Var(N_a^{\pi_i}(T)) = o(T)$ for all suboptimal arm $a$, holds for a wide class of but not all commonly used policies. 
On the one hand, we show that the second condition holds for the classical
UCB1 policy with bounded rewards \cite[Figure~1]{auer2002finite} and UCB-$\delta$ policy with subgaussian reward distributions \cite[Algorithm~3, p.~103]{lattimore2020bandit}. 
This finding is formalized in Proposition~\ref{prop:UCB-validity}, with the proof provided in Appendix~\ref{app:prop-UCB-validity}.
On the other hand, for optimized algorithms with discrimination-equivalent rewards---such as Thompson sampling with Gaussian rewards, {Corollary~4 in \cite{fan2024fragility} suggests that for any $\epsilon \in (0, 1)$, $\liminf_{T \rightarrow \infty} \E[(N_{a^{(2)}}^{\pi}(T))^2]/ T^{\epsilon} \geq 1$ holds for the second-optimal arm $a^{(2)}$.}
While their statement does not strictly imply a violation of the second condition of Theorem~\ref{thm:asy_var}, it nevertheless conveys the possibility that the second moment $\E[(N_{a^{(2)}}^{\pi}(T))^2]$ may grow linearly (or even superlinearly) with $T$.
Given the potential concern about the variance-reduction performance of our AR design when the assumption of Theorem~\ref{thm:asy_var} fails, we use an example in Section~\ref{ssec:example3} to 
provide empirical evidence that the variance-reduction effect can remain substantial even when the assumption does not hold.

\begin{proposition}\label{prop:UCB-validity}
Consider a $K$-armed bandit problem.
\begin{itemize}
    \item [(i)] Suppose the reward distributions are supported on $[0, 1]$.
    Let $\pi$ be the UCB1 policy with exploration parameter $\alpha > 0$. At each period $t \in [T]$, the policy selects $A_t^{\pi} = \argmax_{a \in [K]} \text{UCB}_{a}(t-1)$, where the index is defined as
    \begin{align*}
        \text{UCB}_{a}(t-1) = \begin{cases}
        \infty & \text{ if }  N_a^{\pi}(t-1) = 0 \\
        \hat{\mu}_a(t-1) + \sqrt{\frac{\alpha \log t}{N_a^{\pi}(t-1)}} & \text{otherwise}
    \end{cases},
    \end{align*}
    with $\hat{\mu}_a(t-1)$ is the empirical mean reward of arm $a$ after $t-1$ periods.  
    If $\alpha \geq 2$, then for any suboptimal arm~$a$, {$\Var(N_a^{\pi}(T)) = O((\log(T))^2).$}

    \item [(ii)] Suppose the reward distributions are 1-subgaussian. 
    Let $\pi$ be the UCB‑$\delta$ policy with confidence parameter $\delta \in (0, 1)$.
    At each period $t \in [T]$, the policy selects $A_t^{\pi} = \argmax_{a \in [K]} \text{UCB}_{a}(t-1, \delta)$, where the index is defined as
    \begin{align*}
        \text{UCB}_{a}(t-1, \delta) = \begin{cases}
        \infty & \text{ if }  N_a^{\pi}(t-1) = 0 \\
        \hat{\mu}_a(t-1) + \sqrt{\frac{2 \log(1/\delta)}{N_a^{\pi}(t-1)}} & \text{otherwise}
    \end{cases},
    \end{align*}
    with $\hat{\mu}_a(t-1)$ is the empirical mean reward of arm~$a$ after $t-1$ periods. 
    Given a known horizon~$T$, if the confidence parameter is set to $\delta = T^{-d}$ for some $d \geq 2$,   
    then for any suboptimal arm~$a$, $\Var(N_a^{\pi}(T)) = O((\log(T))^2).$
\end{itemize}
\end{proposition}

\begin{remark}
    {In the classical UCB1 algorithm, the exploration parameter $\alpha$ is typically set to $\alpha = 2$. In fact, the algorithm achieves an $O(\log T)$ regret bound for any $\alpha \ge 2$ \citep{auer2002finite}.  Similarly, in the classical UCB‑$\delta$ algorithm, the confidence parameter is usually parameterized as $\delta = T^{-2}$, corresponding to $d = 2$. The algorithm also enjoys an $O(\log T)$ regret bound for any $d \geq 2$ \citep{lattimore2020bandit}. }
\end{remark}

\section{Numerical Experiment}\label{sec:numerical}

In this section, we conduct a series of numerical experiments to empirically validate the theoretical properties of our proposed AR design. Before presenting individual examples, we first describe the general experimental setup.
\subsection{General Experimental Setup}

{To facilitate visualization, we estimate the means of the AR and na\"ive estimators, along with their corresponding confidence intervals, based on a small sample of $M=10$ independent runs for each experiment.} We define $\thetahat{AR}^{(M)}(T)$ and $\thetahat{b}^{(M)}(T)$ as the sample means of the AR and na\"{\i}ve estimators, respectively.
That is,
\begin{align*}
    \thetahat{AR}^M(T) = \frac{1}{M} \sum_{m=1}^M [\thetahat{AR}(T)]^{(m)}, \qquad \thetahat{b}^M(T) = \frac{1}{M} \sum_{m=1}^M [\thetahat{b}(T)]^{(m)},
\end{align*}
where $[\thetahat{AR}(T)]^{(m)}$ and $[\thetahat{b}(T)]^{(m)}$ denote the estimates obtained from the $m$-th independent run of the AR and na\"{\i}ve baseline experiments, respectively.
Based on the empirical AR and na\"{\i}ve estimators defined above, the corresponding asymptotic $(1 - \alpha)$-level confidence intervals for the ATE, $\theta(T)$, can be respectively constructed as
\begin{align*}
\text{CI}_{\text{AR}}^{\alpha}(T)
= \bigg[
\thetahat{AR}^{(M)}(T)
\, \pm \,
t_{1-\frac{\alpha}{2}, M-1}\frac{s_{\text{AR}}^{(M)}}{\sqrt{M}}\bigg], \qquad \text{CI}_{\text{b}}^{\alpha}(T) = \bigg[
\thetahat{b}^{(M)}(T)
\, \pm \,
t_{1-\frac{\alpha}{2}, M-1}\frac{s_{\text{b}}^{(M)}}{\sqrt{M}}\bigg].
\end{align*}
Here, $t_{1-\frac{\alpha}{2}, M-1}$ denotes the upper $(1-\frac{\alpha}{2})$-quantile of the Student's $t$-distribution with $M-1$ degrees of freedom, and the quantities $s_{\text{AR}}^{(M)}$ and $s_{\text{b}}^{(M)}$ are the sample standard deviations, computed as
\begin{align*}
    s_{\text{AR}}^{(M)} = \sqrt{\frac{1}{M-1} \sum_{m=1}^M \Big([\thetahat{AR}(T)]^{(m)} - \thetahat{AR}^M(T)\Big)^2}, \qquad s_{\text{b}}^{(M)} = \sqrt{\frac{1}{M-1} \sum_{m=1}^M \Big([\thetahat{b}(T)]^{(m)} - \thetahat{b}^M(T)\Big)^2}.
\end{align*}

We estimate other relevant quantities, including the variance of suboptimal-arm pull counts for a given policy, the variances of the AR and na\"ive estimators, and the expected number of real-environment interactions under both AR and na\"ive experiments, using 10,000 independent runs to ensure accurate estimates. We report all the corresponding estimates computed at selected horizons to reduce the computational cost.

\subsection{Example 1: Comparison of Two UCB1 Policies in a Bernoulli Bandit}
We consider a 5-armed Bernoulli bandit with mean rewards $[0.9,\,0.7,\,0.5,\,0.3,\,0.1]$.
The two policies under comparison are variants of the UCB1 policy: $\pi_0$ is the UCB1 policy with $\alpha=2.5$, and $\pi_1$ is the UCB1 policy with $\alpha=3$.

We first examine the expected number of real‑environment interactions during the AR experiment (Algorithm~\ref{alg:AR}).  Table~\ref{tab:Example1_env_interactions} reports that, for each selected horizon $T \in \{10,\,10^2,\,10^3,\,10^4\}$,  
AR experiment requires only slightly more than $T$ environment interaction.
In contrast to the na\"{\i}ve design, which requires 
$2T$ environment interactions, our AR design exhibits markedly higher sample efficiency.
This behavior is consistent with the statement in Theorem~\ref{thm:AR-pull-times}, given that the regret of UCB1 with $\alpha \geq 2$ grows on the order of $\log T$.

Next, Figure~\ref{fig:Example1_confidence_interval} illustrates the 99\% confidence intervals constructed from the AR and na\"{\i}ve estimators, where the dashed lines denote the sample means and solid lines indicate the corresponding confidence bounds.
The sample mean of the AR estimator closely aligns with that of the na\"{\i}ve estimator across all selected horizons and exhibits no systematic deviation,
providing empirical support for the unbiasedness of the AR estimator. 
In addition, the smoother curve of the AR sample mean also suggests that the AR estimator has a lower variance.
This observation is further supported by the fact that the AR estimator exhibits a notably narrower and more stable confidence band than the na\"{\i}ve estimator.
More importantly, the 99\% confidence region constructed from the AR estimator lies entirely below the zero level, implying that we can be 99\% confident that the UCB1 policy with $\alpha=3$ underperforms the UCB1 policy with $\alpha=2.5$ in this instance.
In contrast, the 99\% confidence region constructed from the na\"{\i}ve estimator spans both positive and negative values, indicating that it cannot provide a statistically consistent conclusion about the relative performance of the two policies.
\begin{table}[htbp]
    \centering
    \caption{Average number of real-environment interactions under AR and na\"{\i}ve designs (Example~1).}
    \label{tab:Example1_env_interactions}
    \vspace{0.5em}
    \begin{tabular}{ccccc}
    \hline
    Horizon & $T=10$ & $T=10^2$ & $T=10^3$ & $T=10^4$ \\
    \hline
    Na\"{\i}ve & 20.00 & 200.00 & 2,000.00 & 20,000.00 \\
    AR & 10.46 & 102.44 & 1,027.08 & 10,092.58\\
    \hline
    \end{tabular}   
\end{table}

\begin{figure}[htbp]
\centering
\begin{minipage}[b]{0.45\textwidth}
    \centering
    \captionof{figure}{99\% confidence intervals for AR and na\"{\i}ve estimators (Example~1).}
    \label{fig:Example1_confidence_interval}
    \vspace{0.5em}
    \begin{tikzpicture}
    \begin{axis}[
        font=\small,
        width=\linewidth, height=0.75\linewidth,
        xlabel={Horizon $T$},
        ylabel={99\% Confidence Interval},
        grid=both,
        legend style={
            at={(0.01, 0.02)},
            anchor=south west,
        },
        xmin=0, xmax=10000,
        ymin=-150, ymax=75,
        xtick={0, 2000, 4000, 6000, 8000, 10000},
        ytick={-150, -75, 0, 75},
        xticklabels={0, 2, 4, 6, 8, 10},
        xtick scale label code/.code={$\times 10^3$},
        every x tick scale label/.style={
            at={(rel axis cs:1,0)},  
            anchor=north east,     
            yshift=-1.5em,
            inner sep=1pt},
        ]
    
    \addplot[blue, dash pattern=on 2pt off 1pt, forget plot] table[col sep=comma, x=horizon, y=baseline_mean] {data/Example1.csv};
    
    \addplot[blue, thick, forget plot] table[col sep=comma, x=horizon, y=baseline_lb] {data/Example1.csv};
    \addplot[blue, thick, forget plot] table[col sep=comma, x=horizon, y={baseline_ub}] {data/Example1.csv};
    
    \addplot[orange, dash pattern=on 2pt off 1pt, forget plot] table[col sep=comma, x=horizon, y=AR_mean] {data/Example1.csv};
    
    \addplot[orange, thick, forget plot] table[col sep=comma, x=horizon, y=AR_lb] {data/Example1.csv};
    \addplot[orange, thick, forget plot] table[col sep=comma, x=horizon, y=AR_ub] {data/Example1.csv};
    
    \addlegendimage{blue, thick} 
    \addlegendentry{Na\"{\i}ve}
    
    \addlegendimage{orange, thick}
    \addlegendentry{AR}
        
    \end{axis}
    \end{tikzpicture}
\end{minipage}
\hfill
\begin{minipage}[b]{0.45\textwidth}
    \centering
    \captionof{figure}{99\% confidence intervals for AR and na\"{\i}ve estimators (Example~2).}
    \label{fig:Example2_confidence_interval}
    \vspace{0.5em}
    \begin{tikzpicture}
    \begin{axis}[
        font=\small,
        width=\linewidth, height=0.75\linewidth,
        xlabel={Horizon $T$},
        ylabel={99\% Confidence Interval},
        grid=both,
        legend style={
            at={(0.01, 0.98)},
            anchor=north west,
        },
        xmin=0, xmax=10000,
        ymin=-100, ymax=200,
        xtick={0, 2000, 4000, 6000, 8000, 10000},
        ytick={-100, 0, 100, 200},
        xticklabels={0, 2, 4, 6, 8, 10},
        xtick scale label code/.code={$\times 10^3$},
        every x tick scale label/.style={
            at={(rel axis cs:1,0)},  
            anchor=north east,     
            yshift=-1.5em,
            inner sep=1pt},
        ]
    
    \addplot[blue, dash pattern=on 2pt off 1pt, forget plot] table[col sep=comma, x=horizon, y=baseline_mean] {data/Example2.csv};
    
    \addplot[blue, thick, forget plot] table[col sep=comma, x=horizon, y=baseline_lb] {data/Example2.csv};
    \addplot[blue, thick, forget plot] table[col sep=comma, x=horizon, y=baseline_ub] {data/Example2.csv};
    
    \addplot[orange, dash pattern=on 2pt off 1pt, forget plot] table[col sep=comma, x=horizon, y=AR_mean] {data/Example2.csv};
    
    \addplot[orange, thick, forget plot] table[col sep=comma, x=horizon, y=AR_lb] {data/Example2.csv};
    \addplot[orange, thick, forget plot] table[col sep=comma, x=horizon, y=AR_ub] {data/Example2.csv};
    
    \addlegendimage{blue, thick} 
    \addlegendentry{Na\"{\i}ve}
    
    \addlegendimage{orange, thick}
    \addlegendentry{AR}
    \end{axis}
    \end{tikzpicture}
\end{minipage}
\end{figure}

\subsection{Example 2: UCB1 Versus Thompson Sampling in a Bernoulli Bandit}
We consider a 2-armed Bernoulli bandit with mean rewards $[0.7,\,0.3]$.
The two policies under comparison are as follows: $\pi_0$ is the UCB1 policy with $\alpha=2$, and $\pi_1$ is the Thompson sampling policy with independent $\mathrm{Beta}(1, 1)$ priors on the mean reward of each arm.

Although we do not provide a theoretical proof regarding whether, in a Bernoulli bandit, Thompson sampling policy satisfy sublinear growth in the variances of sub‑optimal arm pull counts, for the instance considered here, the experimental results shown in Figure~\ref{fig:Example2_var_suboptimal_arm_pulls} indicate that this property holds for Thompson sampling policy.
Considering that the two policies under comparison both achieve $O(\log T)$ regret,
Table~\ref{tab:Example2_env_interactions} shows, as expected, that the average number of interactions with the real environment of the AR experiment is slightly above $T$ and remains substantially below $2T$, highlighting its improved sample efficiency relative to the na\"{\i}ve estimator.

Moreover, Figure~\ref{fig:Example2_confidence_interval} depicts the 99\% confidence interval constructed from the AR and na\"{\i}ve estimator. 
The AR estimator yields a noticeably narrower confidence region, demonstrating a substantial reduction in variance compared to the na\"{\i}ve estimator.
The confidence band of our AR estimator lies predominantly above the zero level, whereas that of the na\"{\i}ve estimator spans both sides of the zero level, with substantial mass on each side.
Therefore, based on the AR estimator, we can make the inference with 99\% confidence that the TS policy outperforms the UCB1 policy with $\alpha=2$ in this instance,  whereas the na\"{\i}ve estimator does not support such a conclusion at the same confidence level.

\begin{table}[htbp]
    \centering
    \caption{Average number of real-environment interactions under AR and na\"{\i}ve designs (Example~2).}
    \label{tab:Example2_env_interactions}
    \vspace{0.5em}
    \begin{tabular}{ccccc}
    \hline
    Horizon & $T=10$ & $T=10^2$ & $T=10^3$ & $T=10^4$ \\
    \hline
    Na\"{\i}ve & 20.00 & 200.00 & 2,000.00 & 20,000.00 \\
    AR & 11.05 & 109.52 & 1,038.62 & 10,076.03 \\
    \hline
    \end{tabular}
\end{table}
\begin{figure}[htbp]
\centering
\begin{minipage}[t]{0.45\textwidth}
    \centering
    \begin{tikzpicture}
    \begin{axis}[
        font=\small,
        width=\linewidth, height=0.75\linewidth,
        xlabel={Horizon $T$},
        ylabel={$\operatorname{Var}(N_2^{\pi}(T)) / T$},
        grid=both,
        legend style={
            at={(0.99, 0.98)},
            anchor=north east,
        },
        xmin=0, xmax=10000,
        ymin=0, ymax=0.15,
        xtick={0, 2000, 4000, 6000, 8000, 10000},
        ytick={0, 0.05, 0.10, 0.15},
        yticklabels={0,0.05,0.10,0.15},
        xticklabels={0, 2, 4, 6, 8, 10},
        xtick scale label code/.code={$\times 10^3$},
        every x tick scale label/.style={
            at={(rel axis cs:1,0)},  
            anchor=north east,     
            yshift=-1.5em,
            inner sep=1pt},
        ]

    \addplot[orange, thick, forget plot] table[col sep=comma, x=horizon, y=varT_TS] {data/Example2.csv};

    \addlegendimage{orange, thick}
    \addlegendentry{TS}
    \end{axis}
    \end{tikzpicture}
    \captionof{figure}{$T$-Normalized variance of pull counts for suboptimal arm~2 under  Thompson sampling (Example~2).}
    \label{fig:Example2_var_suboptimal_arm_pulls}
\end{minipage}
\hfill
\begin{minipage}[t]{0.45\textwidth}
    \centering
    \begin{tikzpicture}
    \begin{axis}[
        font=\small,
        width=\linewidth, height=0.75\linewidth,
        xlabel={Horizon $T$},
        ylabel={$\operatorname{Var}(N_2^{\pi}(T)) / T$},
        grid=both,
        legend style={
            at={(0.99, 0.98)},
            anchor=north east,
        },
        xmin=0, xmax=10000,
        ymin=0, ymax=30,
        xtick={0, 2000, 4000, 6000, 8000, 10000},
        ytick={0, 10, 20, 30},
        xticklabels={0, 2, 4, 6, 8, 10},
        xtick scale label code/.code={$\times 10^3$},
        every x tick scale label/.style={
            at={(rel axis cs:1,0)},  
            anchor=north east,     
            yshift=-1.5em,
            inner sep=1pt},
        ]
    
    \addplot[orange, thick, forget plot] table[col sep=comma, x=horizon, y={varT_TS}] {data/Example3.csv};
    
    \addlegendimage{orange, thick}
    \addlegendentry{TS}
    \end{axis}
    \end{tikzpicture}
    \captionof{figure}{$T$-Normalized variance of pull count for suboptimal arm~2 under Thompson sampling (Example~3).}
    \label{fig:Example3_var_suboptimal_arm_pulls}
\end{minipage}
\end{figure}

\begin{figure}
\centering
\begin{minipage}[t]{0.45\textwidth}
    \centering
    \begin{tikzpicture}
    \begin{axis}[
        font=\small,
        width=\linewidth, height=0.75\linewidth,
        xlabel={Horizon $T$},
        ylabel={Average Number of\\Real-Environment Interactions},
        ylabel style={align=center},
        grid=both,
        legend style={
            at={(0.01, 0.98)},
            anchor=north west,
        },
        xmin=0, xmax=10000,
        ymin=0, ymax=20000,
        xtick={0, 2000, 4000, 6000, 8000, 10000},
        ytick={0, 4000, 8000, 12000, 16000, 20000},
        xticklabels={0, 2, 4, 6, 8, 10},
        xtick scale label code/.code={$\times 10^3$},
        every x tick scale label/.style={
            at={(rel axis cs:1,0)},  
            anchor=north east,     
            yshift=-1.5em,
            inner sep=1pt},
        yticklabels={0, 4, 8, 12, 16, 20},
        ytick scale label code/.code={$\times 10^3$},
        every y tick scale label/.style={
            at={(rel axis cs:0,1)},  
            anchor=south west,     
            inner sep=1pt},
        ]
    \addplot[blue, thick, forget plot] table[col sep=comma, x=horizon, y=baseline_num_interactions] {data/Example3.csv};
    \addplot[orange, thick, forget plot] table[col sep=comma, x=horizon, y=AR_num_interactions] {data/Example3.csv};

    \addlegendimage{blue, thick}
    \addlegendentry{Na\"{\i}ve}
    \addlegendimage{orange, thick}
    \addlegendentry{AR}
    \end{axis}
    \end{tikzpicture}
    \captionof{figure}{Average number of real-environment interactions under AR and na\"{\i}ve designs (Example~3).}
    \label{fig:Example3_env_interactions}
\end{minipage}
\hfill
\begin{minipage}[t]{0.45\textwidth}
    \centering
    \begin{tikzpicture}
    \begin{axis}[
        font=\small,
        width=\linewidth, height=0.75\linewidth,
        xlabel={Horizon $T$},
        ylabel={Variance of Estimator},
        ylabel style={align=center},
        grid=both,
        legend style={
            at={(0.01, 0.98)},
            anchor=north west,
        },
        xmin=0, xmax=10000,
        ymin=0, ymax=60000,
        xtick={0, 2000, 4000, 6000, 8000, 10000},
        ytick={0, 20000, 40000, 60000},
        xticklabels={0, 2, 4, 6, 8, 10},
        xtick scale label code/.code={$\times 10^3$},
        every x tick scale label/.style={
            at={(rel axis cs:1,0)},  
            anchor=north east,     
            yshift=-1.5em,
            inner sep=1pt},
        yticklabels={0, 2, 4, 6},
        ytick scale label code/.code={$\times 10^4$},
        every y tick scale label/.style={
            at={(rel axis cs:0,1)},  
            anchor=south west,     
            inner sep=1pt},
        ]
    \addplot[blue, thick, forget plot] table[col sep=comma, x=horizon, y=baseline_var] {data/Example3.csv};
    \addplot[orange, thick, forget plot] table[col sep=comma, x=horizon, y=AR_var] {data/Example3.csv};

    \addlegendimage{blue, thick}
    \addlegendentry{Na\"{\i}ve}
    \addlegendimage{orange, thick}
    \addlegendentry{AR}
    \end{axis}
    \end{tikzpicture}
    \captionof{figure}{Variance of AR and na\"{\i}ve estimators (Example~3).}
    \label{fig:Example3_var_estimators}
\end{minipage}
\end{figure}
\subsection{Example 3: Thompson Sampling Versus $\epsilon$-Greedy in a Gaussian Bandit}\label{ssec:example3}
We consider a 5-armed Gaussion bandit with mean rewards $[1.0,\, 0.8,\,-2,\,-3,\,-4]$, and a common variance of $1$.
The two policies under comparison are as follows: $\pi_0$ is the Thompson sampling policy with independent Gaussian priors $\mathcal{N}(0, 1)$ on the mean reward of each arm, and $\pi_1$ is the $\epsilon$-greedy policy with $\epsilon=0.1$.

It is worth noting that the two policies considered in this example do not meet the assumptions required by Theorem~\ref{thm:asy_var}. 
As illustrated in Figure~\ref{fig:Example3_var_suboptimal_arm_pulls}, Thompson sampling does not satisfy $\Var(N_a^{\pi}(T)) = o(T)$ for suboptimal arm $2$.
In addition, $\epsilon$-greedy policy with a fixed exploration rate $\epsilon=0.1$ even fails to satisfy $\E(N_a^{\pi}(T)) = o(T)$ for some suboptimal arm $a$.
In Figure~\ref{fig:Example3_env_interactions}, we plot the average number of real-environment interactions under the AR and na\"{\i}ve designs
against the horizon.
We observe that, in the AR experiment, this quantity grows approximately linearly with $T$, with an asymptotic slope of about $1.09$, rather than $1$. 
This behavior arises from the fact that the $\epsilon$‑greedy policy with fixed exploration rate $\epsilon$ incurs $O(T)$ regret.
Nevertheless, the slope being much smaller than 2 still reflects the sample‑efficiency advantage of our AR design.

Although the two compared policies do not satisfy the assumptions of Theorem~\ref{thm:asy_var}, 
Figure~\ref{fig:Example3_var_estimators} shows that our AR estimator still exhibits visibly lower variance than the na\"{\i}ve estimator across all considered horizons.
While this reduction is less pronounced than in Examples~1 and 2, it nevertheless permits inference of comparable accuracy with fewer samples.

\section{Conclusion and Future Directions}\label{sec:conclusion}
In this study, we investigate artificial replay as a new design framework for experimenting and comparing multi-armed bandit algorithms.
We show it has promising theoretical and empirical benefits.
It opens up a number of venues for future research on the experimental design of adaptive and online learning policies.

For contextual bandit policies, one exciting problem is whether AR can be adapted to effectively handle high-dimensional contextual information.
Since the reward distribution is conditional on the context and exact matches in continuous or high-dimensional spaces are rare, an exact \emph{replay} may be infeasible.
To address this, future research could explore the use of clustering or smoothing techniques to create ``approximate'' matches. For instance, by partitioning the context space into clusters or using kernel-based smoothing, AR could reuse rewards from historically similar contexts rather than requiring identical ones. However, such approximations would introduce a bias-variance trade-off for the AR estimator: while clustering facilitates data reuse and further reduces the estimator's variance, it incurs potential estimation bias due to context aggregation. Rigorously characterizing this trade-off under these approximations remains a vital open question.

More broadly, for adaptive policies that are prevalent in reinforcement learning, it remains open whether coupling multiple policies in a certain way like AR can improve the estimation of the treatment effect.

\newpage
\bibliographystyle{chicago}
\bibliography{refs}	

\appendix
\section{Proofs of Statements in Section~\ref{sec:analytical_framework}}
In this appendix, we prove the results presented in Section~\ref{sec:analytical_framework}. 
We begin with the proof of Proposition~\ref{prop:stopping-time-martingale}. We then proceed to prove Lemma \ref{lem:bnk-stack-same-dist} and Theorem \ref{thm:AR-stack-same-dist}, leveraging the results established in Proposition~\ref{prop:stopping-time-martingale}.
Finally, we present the proof for Corollary~\ref{cor:mean-var-equal} of Theorem \ref{thm:AR-stack-same-dist}.

\subsection{Proof of Proposition~\ref{prop:stopping-time-martingale}}
\label{app:lem-stopping-time-martingale}
\begin{itemize}
    \item [(i)]
    For $i \in \{0, 1\}$, to demonstrate that $N_a^{\pi_i\text{-s}}(r)$ is a stopping time with respect to $\{\mathcal{F}_t^{T,a}\}_{\{t \in \mathbb{N}\}}$ for all $r \in [T]$, it suffices to show that $\{N_a^{\pi_i\text{-s}}(r) = t\} \in \mathcal{F}_t^{T,a}$ for all $t \in \mathbb{N}$ and $r \in [T]$.
    We prove this property by induction on $r$.
    
    For $r = 1$, note that $\{N_a^{\pi_i\text{-s}}(1) = 0\} = \{A_1^{\pi_i\text{-s}} \neq a\}$.
    Since the arm $A_1^{\pi_i\text{-s}}$ selected at the first period is determined by the internal randomness $\eta_1^i$ of policy $\pi_i$, it follows that $\{N_a^{\pi_i\text{-s}}(1) = 0\} \in \sigma(\eta_1^i)$, and hence $\{N_a^{\pi_i\text{-s}}(1) = 0\} \in \mathcal{F}_0^{T, a}$. 
    Similarly, $\{N_a^{\pi_i\text{-s}}(1) = 1\} = \{A_1^{\pi_i\text{-s}} = a\} \in \mathcal{F}_1^{T,a}$.
    Moreover, for $t > 1$, 
    $\{N_a^{\pi_i\text{-s}}(1) = t\}$ is empty, and hence belongs to $\mathcal{F}_t^{T,a}$.

    Assume that the claim holds for $r-1$ with $ 2 \leq r \leq T$; that is, $\{N_a^{\pi_i\text{-s}}(r-1) = t\} \in \mathcal{F}_t^{T,a}$ for all $t \in \mathbb{N}$.
    We now show that it also holds for $r$.
    Note that 
    \begin{equation}\label{eq:pf-induction-step-1}
        \begin{aligned}
            &\{N_a^{\pi_i\text{-s}}(r) = 0\} = \{N_a^{\pi_i\text{-s}}(r-1) = 0,\, A_r^{\pi_i\text{-s}} \neq a\}, \\
            &\{N_a^{\pi_i\text{-s}}(r) = r\} = \{N_a^{\pi_i\text{-s}}(r-1) = r-1,\, A_r^{\pi_i\text{-s}} = a\},
        \end{aligned}
    \end{equation}
    and for $t = 1, 2, \dots, r-1$, 
    \begin{align}\label{eq:pf-induction-step-2}
        \{N_a^{\pi_i\text{-s}}(r) = t\} = \{N_a^{\pi_i\text{-s}}(r-1) = t-1,\, A_r^{\pi_i\text{-s}} = a\} \cup \{N_a^{\pi_i\text{-s}}(r-1) = t,\, A_r^{\pi_i\text{-s}} \neq a\}.
    \end{align}
    From the recursive construction of $\{(A_t^{\pi_i\text{-s}}, R_t^{\pi_i\text{-s}})\}_{t\in [T]}$ in \eqref{eq:action-reward-stack}, it follows that for any $t = 0,1,\dots r-1$, there exists a measurable arm-valued function $g_{ t}^{\pi_i}$ such that
    \begin{align}
        & A_r^{\pi_i\text{-s}} \mathbb{I}\{N_a^{\pi_i\text{-s}}(r-1) = t\} \nonumber \\
        = {}&g_{t}^{\pi_i}(X_{a, 1}, \dots, X_{a, t},\, X_{b, r'},\, \eta_{r'}^i; b \neq a,\, r' \in [T]) \mathbb{I}\{N_a^{\pi_i\text{-s}}(r-1) = t\}. \label{eq:pf-A-t-expression}
    \end{align}
    Returning to equations \eqref{eq:pf-induction-step-1}-\eqref{eq:pf-induction-step-2} and using the result in \eqref{eq:pf-A-t-expression}, we can infer that for $t=0, 1, \dots, r-1$,
    \begin{align*}
        &\{N_a^{\pi_i\text{-s}}(r-1) = t,\, A_r^{\pi_i\text{-s}} \neq a\} \\
        ={}& \{N_a^{\pi_i\text{-s}}(r-1) = t,\, g_{t}^{\pi_i}(X_{a, 1}, \dots, X_{a, t},\, X_{b, r'},\, \eta_{r'}^i; b \neq a,\, r' \in [T]) \neq a\}.
    \end{align*}
    Because $g_{t}^{\pi_i}(X_{a, 1}, \dots, X_{a, t},\, X_{b, r'},\, \eta_{r'}^i; b \neq a,\, r' \in [T])$ is $\mathcal{F}_t^{T,a}$-measurable, and $\{N_a^{\pi_i\text{-s}}(r-1) = t\} \in \mathcal{F}_{t}^{T,a}$ by the induction hypothesis, it follows that
    \begin{equation}\label{eq:pf-N-a-neq-a}
        \{N_a^{\pi_i\text{-s}}(r-1) = t,\, A_r^{\pi_i\text{-s}} \neq a\} \in \mathcal{F}_{t}^{T,a} \quad \text{ for } t = 0, 1, \dots, r-1.
    \end{equation}
    On the other hand, we have for $t=1, 2, \dots, r$, 
    \begin{align*}
        &\{N_a^{\pi_i\text{-s}}(r-1) = t-1,\, A_r^{\pi_i\text{-s}} = a\} \\
        ={}& \{N_a^{\pi_i\text{-s}}(r-1) = t-1,\, g_{t-1}^{\pi_i}(X_{a, 1}, \dots, X_{a, t-1},\, X_{b, r'},\, \eta_{r'}^i; b \neq a,\, r' \in [T]) = a\}.
    \end{align*}
    Given that $g_{t-1}^{\pi_i}(X_{a, 1}, \dots, X_{a, t-1},\, X_{b, r'},\, \eta_{r'}^i; b \neq a,\, r' \in [T])$ is $\mathcal{F}_{t-1}^{T, a}$-measurable, and the event $\{N_a^{\pi_i\text{-s}}(r-1) = t-1\} \in \mathcal{F}_{t-1}^{T, a}$ by the induction hypothesis, thus we have
    \begin{equation}\label{eq:pf-N-a-eq-a}
        \{N_a^{\pi_i\text{-s}}(r-1) = t-1,\, A_r^{\pi_i\text{-s}} = a\} \in \mathcal{F}_{t-1}^{T, a} \quad \text{ for } t=1, 2, \dots, r.
    \end{equation}
    Combining the conclusions in \eqref{eq:pf-N-a-neq-a} and \eqref{eq:pf-N-a-eq-a} with equations \eqref{eq:pf-induction-step-1} and \eqref{eq:pf-induction-step-2}, we conclude 
    \begin{align*}
         &\{N_a^{\pi_i\text{-s}}(r) = 0\} = \{N_a^{\pi_i\text{-s}}(r-1) = 0,\, A_r^{\pi_i\text{-s}} \neq a\} \in \mathcal{F}_0^{T,a}, \\
        &\{N_a^{\pi_i\text{-s}}(r) = r\} = \{N_a^{\pi_i\text{-s}}(r-1) = r-1,\, A_r^{\pi_i\text{-s}} = a\} \in \mathcal{F}_{r-1}^{T,a} \subseteq \mathcal{F}_r^{T,a},
    \end{align*}
    and for $t = 1, 2, \dots, r-1$, 
    \begin{align*}
        \{N_a^{\pi_i\text{-s}}(r) = t\} = \{N_a^{\pi_i\text{-s}}(r-1) = t-1,\, A_r^{\pi_i\text{-s}} = a\} \cup \{N_a^{\pi_i\text{-s}}(r-1) = t,\, A_r^{\pi_i\text{-s}} \neq a\} \in \mathcal{F}_t^{T,a}.
    \end{align*}
    Moreover, for $t > r$, $\{N_a^{\pi_i\text{-s}}(r) = t\} = \emptyset \in \mathcal{F}_t^{T,a}$.
    Therefore, the statement holds for $r$, completing the induction. We conclude that $N_a^{\pi_i\text{-s}}(r)$ is a stopping time with respect to $\{\mathcal{F}_t^{T,a}\}_{t \in \mathbb{N}}$ for all $r \in [T]$.
    
    In addtion, since $N_a^{\pi_i\text{-s}}(r) \leq T$ for all $r \in [T]$, $N_a^{\pi_i\text{-s}}(r)$ is a bounded stopping time.
    
    \item [(ii)]
    The definition of $\{\mathcal{F}_t^{T,a}\}_{t \in \mathbb{N}}$ in \eqref{eq:def-filtration} implies directly that $I_t^{T, a}(B)$, $S_t^{T, a}$ and $V_t^{T, a}$ are $\mathcal{F}_t^{T,a}$-measurable for all $t \in \mathbb{N}$. We now verify the martingale property for the processes $\{I_t^{T, a}(B),\,t \in \mathbb{N}\}$, $\{S_t^{T, a},\,t \in \mathbb{N}\}$ and $\{V_t^{T, a},\,t \in \mathbb{N}\}$.
    
    First, consider $\{I_t^{T, a}(B),\,t \in \mathbb{N}\}$. 
    For any $0 \leq t \leq T-1$, we have
    \begin{align*}
        \E[I_{t+1}^{T, a}(B) \mid \mathcal{F}_t^{T,a}] &= \E[I_t^{T, a}(B) + (\mathbb{I}\{X_{a, t+1} \in B) - P_a(B)) \mid \mathcal{F}_t^{T,a}] \\
        &= I_t^{T, a}(B) + \E[\mathbb{I}\{X_{a, t+1} \in B) - P_a(B)] = I_t^{T,a}(B),
    \end{align*}
    where the second equality holds because $X_{a, t+1}$ is independent of $\mathcal{F}_t^{T,a}$, and the last equality follows from $\Prob(X_{a, t+1} \in B) = P_a(B)$. 
    Similarly, for $\{S_t^{T, a},\,t \in \mathbb{N}\}$,
    \begin{align*}
        \E[S_{t+1}^{T, a} \mid \mathcal{F}_t^{T,a}] = \E[S_t^{T, a} + (X_{a, t+1} - \mu_a) \mid \mathcal{F}_t^{T,a}] = S_t^{T, a} + \E[X_{a, t+1} - \mu_a] = S_t^{T, a},
    \end{align*}
    where we again use the independence of $X_{a, t+1}$ from $\mathcal{F}_t^{T,a}$ and the fact that $\E[X_{a, t+1}] = \mu_a$.
    
    Next, consider $\{V_t^{T, a},\,t \in \mathbb{N}\}$. For any $0 \leq t \leq T-1$, note that
    \begin{align*}
        V_{t+1}^{T, a} &= [S_t^{T, a} + (X_{a, t+1} - \mu_a)]^2 - (t+1) \sigma_a^2 \\
        &= V_t^{T, a} + (X_{a, t+1} - \mu_a)^2 + 2  S_t^{T, a} (X_{a, t+1} - \mu_a) - \sigma_a^2,
    \end{align*}
    then we have
    \begin{equation}\label{eq:lem-pf-V-parts}
        \begin{aligned}
            \E[V_{t+1}^{T, a} \mid \mathcal{F}_t^{T,a}] 
            &= V_t^{T, a} + \E[(X_{a, t+1} - \mu_a)^2 + 2 S_t^{T, a} (X_{a, t+1} - \mu_a)  \mid \mathcal{F}_t^{T,a}] - \sigma_a^2 \\
            &= V_t^{T, a} + \E[(X_{a, t+1} - \mu_a)^2 \mid \mathcal{F}_t^{T,a}] + 2  \E[S_t^{T, a} (X_{a, t+1} - \mu_a) \mid \mathcal{F}_t^{T,a}] - \sigma_a^2.
        \end{aligned}
    \end{equation}
    Since $X_{a, t+1} - \mu_a$ is independent of $\mathcal{F}_t^{T,a}$ and $S_t^{T, a}$ is $\mathcal{F}_t^{T,a}$-measurable, it follows that
    \begin{equation}\label{eq:lem-pf-conditional-expectation}
        \begin{aligned}
        \E[(X_{a, t+1} - \mu_a)^2 \mid \mathcal{F}_t^{T,a}] &= \E[(X_{a, t+1} - \mu_a)^2] = \sigma_a^2, \\
        \E[S_t^{T, a} (X_{a, t+1} - \mu_a) \mid \mathcal{F}_t^{T,a}] &= S_t^{T, a}\E[X_{a, t+1} - \mu_a] = 0.
    \end{aligned}
    \end{equation}
    Combining equations \eqref{eq:lem-pf-V-parts} and \eqref{eq:lem-pf-conditional-expectation}, we obtain $\E[V_{t+1}^{T, a} \mid \mathcal{F}_t^{T, a}] = V_t^{T, a}$.

    Moreover, since for all $t \geq T$, we have $I_t^{T, a}(B) = I_T^{T, a}(B)$, $S_t^{T, a} = S_T^{T, a}$, and $V_t^{T, a} = V_T^{T, a}$, it follows that for $t \geq T$,
    \begin{align*}
        \E[I_{t+1}^{T, a}(B) \mid \mathcal{F}_t^{T,a}] &= \E[I_t^{T, a}(B) \mid \mathcal{F}_t^{T,a}] = I_t^{T, a}(B), \\
        \E[S_{t+1}^{T, a} \mid \mathcal{F}_t^{T,a}] &= \E[S_t^{T, a} \mid \mathcal{F}_t^{T,a}] = S_t^{T, a}, \\
        \E[V_{t+1}^{T, a} \mid \mathcal{F}_t^{T,a}] &= \E[V_t^{T, a} \mid \mathcal{F}_t^{T,a}] = V_t^{T, a}.
    \end{align*}
    
    Therefore, we conclude that the processes $\{I_t^{T, a}(B),\,t \in \mathbb{N}\}$, $\{S_t^{T, a},\,t \in \mathbb{N}\}$ and $\{V_t^{T, a},\,t \in \mathbb{N}\}$ are all $\{\mathcal{F}_t^{T,a}\}_{t \in \mathbb{N}}$-martingales.
\end{itemize}


\subsection{Proof of Lemma~\ref{lem:bnk-stack-same-dist}}
\label{app:lem-bnk-stack-same-dist}
    First, from the definition of $A_t^{\pi_i\text{-s}}$ in \eqref{eq:action-reward-stack} and the fact $\eta_t^i \sim U(0, 1)$, we obtain
    \begin{align*}
    A_t^{\pi_i\text{-s}} \mid \mathcal{H}_{t-1}^{\pi_i\text{-s}} = h_{t-1}^i \ &\sim \ (\pi_i)_t(\cdot \mid h_{t-1}^i),
    \end{align*}
    which coincide with the conditional distribution of $A_t^{\pi_i}$ given $\mathcal{H}_{t-1}^{\pi_i} = h_{t-1}^i$ in \eqref{eq:A-R-single-conditional-dist}.
    
    To establish \eqref{eq:bnk-stack-same-dist}, it remains to verify that
    for all $t \in [T]$, the conditional distribution of $R_t^{\pi_i\text{-s}}$ given $(\mathcal{H}_{t-1}^{\pi_i\text{-s}}, A_t^{\pi_i\text{-s}}) = (h_{t-1}^i, a)$ coincides with that of $R_t^{\pi_i}$ given $(\mathcal{H}_{t-1}^{\pi_i}, A_t^{\pi_i}) = (h_{t-1}^i, a)$ ; the latter, by the definition in~\eqref{eq:A-R-single-conditional-dist}, is exactly the reward distribution  $P_a$. We next aim to prove that, for any arm $a \in [K]$ and any $B \in \mathcal{B}(\mathbb{R})$,
    \begin{align}\label{eq:pf-bnk-condition-dist-equal}
        \Prob(R_t^{\pi_i\text{-s}} \in B  \mid \mathcal{H}_{t-1}^{\pi_i\text{-s}} = h_{t-1}^i, A_t^{\pi_i\text{-s}} = a) = P_a(B).
    \end{align}
    Since $R_t^{\pi_i\text{-s}}\mathbb{I}\{A_t^{\pi_i\text{-s}} = a\} = X_{a, N_{a}^{\pi_i\text{-s}}(t)}$, we have
    \begin{equation}\label{eq:pf-bnk-P-E}
    \begin{aligned}
        &\Prob(R_t^{\pi_i\text{-s}} \in B  \mid \mathcal{H}_{t-1}^{\pi_i\text{-s}} = h_{t-1}^i, A_t^{\pi_i\text{-s}} = a) - P_a(B) \\
        = {}&\E[\mathbb{I}\{X_{a, N_a^{\pi_i\text{-s}}(t)} \in B\} - P_a(B) \mid \mathcal{H}_{t-1}^{\pi_i\text{-s}} = h_{t-1}^i, A_t^{\pi_i\text{-s}} = a].
    \end{aligned}
    \end{equation}
    Note that $N_a^{\pi_i\text{-s}}(t)$ and $N_a^{\pi_i\text{-s}}(t-1)$ are stopping times with respect to $\{\mathcal{F}_t^{T, a}\}_{t\in \mathbb{N}}$ by Proposition~\ref{prop:stopping-time-martingale}-(i), and 
    \begin{equation}
    \label{eq:pf-bnk-I-I}
    \begin{aligned}
        \big(\mathbb{I}\{X_{a, N_a^{\pi_i\text{-s}}(t)} \in B\} - P_a(B)\big) \mathbb{I}\{A_t^{\pi_i\text{-s}} = a\} = \big(I_{ N_a^{\pi_i\text{-s}}(t)}^{T, a}(B) - I_{N_a^{\pi_i\text{-s}}(t-1)}^{T, a}(B)\big) \mathbb{I}\{A_t^{\pi_i\text{-s}} = a\},
    \end{aligned}
    \end{equation}
    where $I_t^{T, a}(B)$ is defined as in \eqref{eq:def-I-S-V}.
    It then follows from \eqref{eq:pf-bnk-P-E} and \eqref{eq:pf-bnk-I-I} that
    \begin{align}
        &\Prob(R_t^{\pi_i\text{-s}} \in B  \mid \mathcal{H}_{t-1}^{\pi_i\text{-s}} = h_{t-1}^i, A_t^{\pi_i\text{-s}} = a) - P_a(B)  \nonumber\\
        = {}&\E[I_{N_a^{\pi_i\text{-s}}(t)}^{T, a}(B) - I_{N_a^{\pi_i\text{-s}}(t-1)}^{T, a}(B) \mid \mathcal{H}_{t-1}^{\pi_i\text{-s}} = h_{t-1}^i, A_t^{\pi_i\text{-s}} = a] \nonumber\\
        = {}&\E[\E[I_{N_a^{\pi_i\text{-s}}(t)}^{T, a}(B) - I_{ N_a^{\pi_i\text{-s}}(t-1)}^{T, a}(B) \mid \mathcal{F}_{N_a^{\pi_i\text{-s}}(t-1)}^{T, a}] \mid \mathcal{H}_{t-1}^{\pi_i\text{-s}} = h_{t-1}^i, A_t^{\pi_i\text{-s}} = a] \label{eq:pf-tower-rule},
    \end{align}
    where, in the last equality, we use the fact that $\sigma(\mathcal{H}_{t-1}^{\pi_i\text{-s}},\, A_t^{\pi_i\text{-s}}) \subseteq \mathcal{F}_{N_a^{\pi_i\text{-s}}(t-1)}^{T, a}$, and apply the tower rule of conditional expectation.
    Moreover, since $\{I_t^{T, a}(B),\,t \in \mathbb{N}\}$ is a $\{\mathcal{F}_t^{T,a}\}_{t \in \mathbb{N}}$-martingale by Proposition~\ref{prop:stopping-time-martingale}-(ii), and noting that $N_a^{\pi_i\text{-s}}(t-1) \leq N_a^{\pi_i\text{-s}}(t)$, the Doob's optional sampling theorem (see Theorem~6.1.2 in \cite{kallenberg1997foundations}) gives
    \begin{align}\label{eq:pf-I-optional-sampling}
        \E[I_{N_a^{\pi_i\text{-s}}(t)}^{T, a}(B) - I_{N_a^{\pi_i\text{-s}}(t-1)}^{T, a}(B) \mid \mathcal{F}_{N_a^{\pi_i\text{-s}}(t-1)}^{T, a}] = 0.
    \end{align}
    Combining \eqref{eq:pf-tower-rule} and \eqref{eq:pf-I-optional-sampling}, we can conclude that \eqref{eq:pf-bnk-condition-dist-equal} holds. This complete the proof.


\subsection{Proof of Theorem~\ref{thm:AR-stack-same-dist}}
\label{app:thm-AR-stack-same-dist}
    First, since the deployment of $\pi_0$ in Phase 1 of Algorithm~\ref{alg:AR} is identical to that in Algorithm~\ref{alg:benchmark}, we have
    \begin{align}\label{eq:pf-AR-1-same-dist}
        (A_1^{\pi_0\text{-AR}}, R_1^{\pi_0\text{-AR}}, \dots, A_T^{\pi_0\text{-AR}}, R_T^{\pi_0\text{-AR}}) \stackrel{\text{d}}{=} (A_1^{\pi_0}, R_1^{\pi_0}, \dots, A_T^{\pi_0}, R_T^{\pi_0}).
    \end{align}
    By combining \eqref{eq:pf-AR-1-same-dist} with the result of Lemma~\ref{lem:bnk-stack-same-dist}, we obtain that 
    \begin{align*}
        (A_1^{\pi_0\text{-AR}}, R_1^{\pi_0\text{-AR}}, \dots, A_T^{\pi_0\text{-AR}}, R_T^{\pi_0\text{-AR}}) \stackrel{\text{d}}{=} (A_1^{\pi_0\text{-s}}, R_1^{\pi_0\text{-s}}, \dots, A_T^{\pi_0\text{-s}}, R_T^{\pi_0\text{-s}}).
    \end{align*}
    
    Next, let $h_T^0=(a_1^0, r_1^0, \ldots, a_T^0, r_T^0) \in ([K] \times \mathbb{R})^T$ and $h_{t-1}^1=(a_1^1, r_1^1, \ldots, a_{t-1}^1, r_{t-1}^1) \in ([K] \times \mathbb{R})^{t-1}$. 
    By the definition of $A_t^{\pi_i\text{-s}}$ in \eqref{eq:action-reward-stack} and the fact $\eta_t^i \sim U(0, 1)$, it follows that
    \begin{align*}
    A_t^{\pi_1\text{-s}} \mid (\mathcal{H}_T^{\pi_0\text{-s}}, \mathcal{H}_{t-1}^{\pi_1\text{-s}}) = (h_T^0, h_{t-1}^1) \ &\sim \ (\pi_1)_t(\cdot \mid h_{t-1}^1),
    \end{align*}
    which coincide with the conditional distributions of $A_t^{\pi_1\text{-AR}}$ given $(\mathcal{H}_T^{\pi_0\text{-AR}}, \mathcal{H}_{t-1}^{\pi_1\text{-AR}}) = (h_T^0, h_{t-1}^1)$ in \eqref{eq:A-2-AR-conditional-dist}.
    In order to establish \eqref{eq:AR-stack-same-dist}, it suffices to verify that for all $t \in [T]$, the conditional distribution of $R_t^{\pi_1\text{-AR}}$ given $(\mathcal{H}_T^{\pi_0\text{-AR}}, \mathcal{H}_{t-1}^{\pi_1\text{-AR}}, A_{t}^{\pi_1\text{-AR}}) = (h_T^0, h_{t-1}^1, a)$ coincides with that of $R_t^{\pi_1\text{-s}}$ given $(\mathcal{H}_T^{\pi_0\text{-s}}, \mathcal{H}_{t-1}^{\pi_1\text{-s}}, A_{t}^{\pi_1\text{-s}}) = (h_T^0, h_{t-1}^1, a)$.
    Therefore, we proceed to prove that for any arm $a \in [K]$ and any $B \in \mathcal{B}(\mathbb{R})$, the following equality holds
    \begin{align}
        &\Prob( R_t^{\pi_1\text{-AR}} \in B \mid \mathcal{H}_T^{\pi_0\text{-AR}}=h_T^0, \mathcal{H}_{t-1}^{\pi_1\text{-AR}}=h_{t-1}^1,  A_t^{\pi_1\text{-AR}} = a) \label{eq:pf-AR-stack-same-dist-1} \\
        = {}&\Prob( R_t^{\pi_1\text{-s}} \in B \mid \mathcal{H}_T^{\pi_0\text{-s}}=h_T^0, \mathcal{H}_{t-1}^{\pi_1\text{-s}}=h_{t-1}^1,  A_t^{\pi_1\text{-s}} = a).\label{eq:pf-AR-stack-same-dist-2}
    \end{align}

    For \eqref{eq:pf-AR-stack-same-dist-1}, according the definition in \eqref{eq:R-2-AR-conditional-dist}, we have
     \begin{align}
        &\Prob( R_t^{\pi_1\text{-AR}} \in B \mid \mathcal{H}_T^{\pi_0\text{-AR}}=h_T^0, \mathcal{H}_{t-1}^{\pi_1\text{-AR}}=h_{t-1}^1,  A_t^{\pi_1\text{-AR}}=a) \nonumber \\
        ={}& 
        \begin{cases}
        \mathbb{I} \{r_{s^*}^0 \in B\} & \text{if } n_a^1(t-1) < n_a^0(T), \\
        P_{a}(B) & \text{otherwise},
        \end{cases} \label{eq:pf-first-term-form}
    \end{align}
    where $n_a^i(t) \triangleq \sum_{r=1}^{t} \mathbb{I}\{a_r^i = a\}$ and $s^* \triangleq \min\{s \in [T]: n_a^0(s) = n_a^1(t-1) + 1\}$.
    For \eqref{eq:pf-AR-stack-same-dist-2}, recall that
    $R_t^{\pi_1\text{-s}} = X_{a, n_a^1(t-1)+1}$ when $A_t^{\pi_1\text{-s}} = a$.
    Moreover, if $n_a^1(t-1) < n_a^0(T)$, the reward $r_{s^*}^0$ equals to $X_{a, n_a^1(t-1) + 1}$. Therefore, we have
    \begin{align}
        &\Prob( R_t^{\pi_1\text{-s}} \in B \mid \mathcal{H}_T^{\pi_0\text{-s}}=h_T^0, \mathcal{H}_{t-1}^{\pi_1\text{-s}}=h_{t-1}^1,  A_t^{\pi_1\text{-s}}=a) \nonumber \\
        = {}&
        \begin{cases}
        \mathbb{I} \{r_{s^*}^0 \in B\} & \text{if } n_a^1(t-1) < n_a^0(T), \\
        \Prob( X_{a, n_a^1(t-1) + 1} \in B \mid \mathcal{H}_T^{\pi_0\text{-s}}=h_T^0, \mathcal{H}_{t-1}^{\pi_1\text{-s}}=h_{t-1}^1,  A_t^{\pi_1\text{-s}}=a) & \text{otherwise}.
        \end{cases} \label{eq:pf-second-term-form}
    \end{align}
    By comparing the right-hand sides of equations \eqref{eq:pf-first-term-form} and \eqref{eq:pf-second-term-form}, it suffices to prove that if $n_a^1(t-1) \geq n_a^0(T)$, then 
    \begin{align}\label{eq:pf-AR-condition-dist-equal}
        \Prob( X_{a, N_a^{\pi_1\text{-s}}(t-1)+1} \in B \mid \mathcal{H}_T^{\pi_0\text{-s}} = h_T^0, \mathcal{H}_{t-1}^{\pi_1\text{-s}}=h_{t-1}^1, A_t^{\pi_1\text{-s}}=a) = P_a(B).
    \end{align}
    Indeed, when $n_a^1(t-1) \geq n_a^0(T)$, the left-hand side of \eqref{eq:pf-AR-condition-dist-equal} can be rewritten as
    \begin{align}\label{eq:pf-AR-P-E}
        \E\big[\big(\mathbb{I}\{X_{a, [N_a^{\pi_1\text{-s}}(t-1) + 1]\vee N_a^{\pi_0\text{-s}}(T)} \in B\} - P_a(B)\big) \mid \mathcal{H}_T^{\pi_0\text{-s}} = h_T^0, \mathcal{H}_{t-1}^{\pi_1\text{-s}}=h_{t-1}^1, A_t^{\pi_1\text{-s}}=a\big] = 0.
    \end{align}
    Note that when $n_a^1(t-1) \geq n_a^0(T)$, we have
    \begin{align}
        &\E\big[\big(\mathbb{I}\{X_{a, [N_a^{\pi_1\text{-s}}(t-1) + 1]\vee N_a^{\pi_0\text{-s}}(T)} \in B\} - P_a(B)\big) \mid \mathcal{H}_T^{\pi_0\text{-s}} = h_T^0, \mathcal{H}_{t-1}^{\pi_1\text{-s}}=h_{t-1}^1, A_t^{\pi_1\text{-s}}=a\big] \nonumber \\
        = {}&\E[I_{[N_a^{\pi_1\text{-s}}(t-1)+1] \vee N_a^{\pi_0\text{-s}}(T)}^{T, a}(B) - I_{N_a^{\pi_1\text{-s}}(t-1) \vee N_a^{\pi_0\text{-s}}(T)}^{T, a}(B) \mid \mathcal{H}_T^{\pi_0\text{-s}} = h_T^0, \mathcal{H}_{t-1}^{\pi_1\text{-s}}=h_{t-1}^1, A_t^{\pi_1\text{-s}}=a] \nonumber \\
        = {}&\E[I_{N_a^{\pi_1\text{-s}}(t) \vee N_a^{\pi_0\text{-s}}(T)}^{T, a}(B) - I_{N_a^{\pi_1\text{-s}}(t-1) \vee N_a^{\pi_0\text{-s}}(T)}^{T, a}(B) \mid \mathcal{H}_T^{\pi_0\text{-s}} = h_T^0, \mathcal{H}_{t-1}^{\pi_1\text{-s}}=h_{t-1}^1, A_t^{\pi_1\text{-s}}=a] \label{eq:pf-AR-I-I},
    \end{align}
    where $I_t^{T, a}(B)$ is defined as in \eqref{eq:def-I-S-V}.
    Since $\sigma(\mathcal{H}_T^{\pi_0\text{-s}},\, \mathcal{H}_{t-1}^{\pi_1\text{-s}},\,  A_t^{\pi_1\text{-s}}) \subseteq \mathcal{F}_{N_a^{\pi_1\text{-s}}(t-1) \vee N_a^{\pi_0\text{-s}}(T)}^{T, a}$, we apply the tower rule of conditional expectation to obtain
    \begin{equation}
    \label{eq:pf-AR-tower-rule}
    \begin{aligned}
        \eqref{eq:pf-AR-I-I}
        &= \E\big[\E[I_{N_a^{\pi_1\text{-s}}(t) \vee N_a^{\pi_0\text{-s}}(T)}^{T, a}(B) - I_{N_a^{\pi_1\text{-s}}(t-1) \vee N_a^{\pi_0\text{-s}}(T)}^{T, a}(B) \mid \mathcal{F}_{N_a^{\pi_1\text{-s}}(t-1) \vee N_a^{\pi_0\text{-s}}(T)}^{T, a}] \\
        &\hspace{6.5cm} \mid \mathcal{H}_T^{\pi_0\text{-s}} = h_T^0, \mathcal{H}_{t-1}^{\pi_1\text{-s}}=h_{t-1}^1, A_t^{\pi_1\text{-s}}=a\big] .
    \end{aligned}
    \end{equation}
    By Proposition~\ref{prop:stopping-time-martingale}-(ii), $\{I_t^{T, a}(B),\,t \in \mathbb{N}\}$ is a $\{\mathcal{F}_t^{T,a}\}_{t \in \mathbb{N}}$-martingale.
    In addition, note that $N_a^{\pi_1\text{-s}}(t-1)$, $N_a^{\pi_1\text{-s}}(t)$ and $N_a^{\pi_0\text{-s}}(T)$ are all bounded stopping times with respect to 
    $\{\mathcal{F}_t^{T, a}\}_{t \in \mathbb{N}}$, 
    and satisfy $N_a^{\pi_1\text{-s}}(t-1) \vee N_a^{\pi_0\text{-s}}(T) \leq N_a^{\pi_1\text{-s}}(t) \vee N_a^{\pi_0\text{-s}}(T)$.
    Therefore, by the Doob's optional sampling theorem (see Theorem~6.1.2 in \cite{kallenberg1997foundations}), we have
    \begin{align}\label{eq:pf-AR-I-optional-sampling}
        \E[I_{N_a^{\pi_1\text{-s}}(t) \vee N_a^{\pi_0\text{-s}}(T)}^{T, a}(B) - I_{N_a^{\pi_1\text{-s}}(t-1) \vee N_a^{\pi_0\text{-s}}(T)}^{T, a}(B) \mid \mathcal{F}_{N_a^{\pi_1\text{-s}}(t-1) \vee N_a^{\pi_0\text{-s}}(T)}^{T, a}] = 0.
    \end{align}
    Combining \eqref{eq:pf-AR-I-I}, \eqref{eq:pf-AR-tower-rule}
    and \eqref{eq:pf-AR-I-optional-sampling}, we conclude that \eqref{eq:pf-AR-P-E} holds. This completes the proof.


\subsection{Proof of Corollary~\ref{cor:mean-var-equal}}
\label{app:cor-mean-var-equal}
By Theorem~\ref{thm:AR-stack-same-dist}, we have that
\begin{align*}
    &(A_1^{\pi_0\text{-AR}}, R_1^{\pi_0\text{-AR}}, \dots, A_T^{\pi_0\text{-AR}}, R_T^{\pi_0\text{-AR}}, A_1^{\pi_1\text{-AR}}, R_1^{\pi_1\text{-AR}}, \dots, A_T^{\pi_1\text{-AR}}, R_T^{\pi_1\text{-AR}}) \\
    \stackrel{\text{d}}{=} 
    {}& (A_1^{\pi_0\text{-s}}, R_1^{\pi_0\text{-s}}, \dots, A_T^{\pi_0\text{-s}}, R_T^{\pi_0\text{-s}}, A_1^{\pi_1\text{-s}}, R_1^{\pi_1\text{-s}}, \dots, A_T^{\pi_1\text{-s}}, R_T^{\pi_1\text{-s}}). 
\end{align*}
Therefore, for each $i \in \{0, 1\}$, the following marginal distributions are identical:
\begin{equation}\label{eq:pf-AR-stack-same-dist-i}
\begin{aligned}
    (A_1^{\pi_i\text{-AR}}, A_2^{\pi_i\text{-AR}}, \dots, A_T^{\pi_i\text{-AR}}) &\stackrel{\text{d}}{=} (A_1^{\pi_i\text{-s}}, A_2^{\pi_i\text{-s}}, \dots, A_T^{\pi_i\text{-s}}), \\
    (R_1^{\pi_i\text{-AR}}, R_2^{\pi_i\text{-AR}}, \dots, R_T^{\pi_i\text{-AR}}) &\stackrel{\text{d}}{=} (R_1^{\pi_i\text{-s}}, R_2^{\pi_i\text{-s}}, \dots, R_T^{\pi_i\text{-s}}).
\end{aligned}    
\end{equation}
Similarly, it follows Lemma~\ref{lem:bnk-stack-same-dist} that
\begin{equation}\label{eq:pf-bnk-stack-same-dist-i}
\begin{aligned}
    (A_1^{\pi_i}, A_2^{\pi_i}, \dots, A_T^{\pi_i}) &\stackrel{\text{d}}{=} (A_1^{\pi_i\text{-s}}, A_2^{\pi_i\text{-s}}, \dots, A_T^{\pi_i\text{-s}}), \\
    (R_1^{\pi_i}, R_2^{\pi_i}, \dots, R_T^{\pi_i}) &\stackrel{\text{d}}{=} (R_1^{\pi_i\text{-s}}, R_2^{\pi_i\text{-s}}, \dots, R_T^{\pi_i\text{-s}}).
\end{aligned}
\end{equation}
Since the distributions of corresponding sequences in \eqref{eq:pf-AR-stack-same-dist-i} and \eqref{eq:pf-bnk-stack-same-dist-i} are identical, any measurable function of these sequences has the same expectation.  
Note that for each policy $\pi_i$ ($i \in \{0, 1\}$), we have
\begin{align*}
    N_a^{\pi_i}(T) &= \sum_{t=1}^T \mathbb{I}\{A_t^{\pi_i} = a\}, \quad
    \Var(N_a^{\pi_i}(T)) = \E[(N_a^{\pi_i}(T) - \E[N_a^{\pi_i}(T)])^2],
\end{align*}
and the same form applies to $N_a^{\pi_i\text{-AR}}(T)$ and $N_a^{\pi_i\text{-s}}(T)$ with the corresponding action sequences.
Hence, the desired equalities in Corollary~\ref{cor:mean-var-equal} follow immediately.


\section{Proofs of Statements in Section~\ref{sec:theoretical_properties}}
In this appendix, we prove the results presented in Section~\ref{sec:theoretical_properties}. We begin with the proof of Theorem~\ref{thm:AR-pull-times}. Next, we prove Corollary~\ref{cor:moment-eq} of Proposition~\ref{prop:stopping-time-martingale} (from Section~\ref{sec:analytical_framework}), which establishes a necessary foundation for the subsequent proof of Theorem~\ref{thm:asy_var}. Following the proof of Theorem~\ref{thm:asy_var}, we conclude with the proof of Proposition~\ref{prop:UCB-validity}.

\subsection{Proof of Theorem~\ref{thm:AR-pull-times}}
\label{app:thm-AR-pull-times}
    On the one hand, combining the definitions of $N^{\text{e-AR}}(T)$ and $N^{\text{r-AR}}(T)$ in \eqref{eq:def-num-env-interactions}--\eqref{eq:def-num-replay} with the identity $\max\{x, y\} = x + y - \min\{x, y\}$, we obtain
    \begin{align}
        N^{\text{e-AR}}(T) &= \sum_{a \in [K]} N_a^{\pi_0\text{-AR}}(T) + \sum_{a \in [K]}  N_a^{\pi_1\text{-AR}}(T) - N^{\text{r-AR}}(T) \nonumber  \\
        &= 2T - N^{\text{r-AR}}(T), \label{eq:Ne-2T-Nr}
    \end{align}
    where the second equality follows from the fact
    \begin{align}\label{eq:both-sum-to-T}
        \sum_{a \in [K]} N_a^{\pi_0\text{-AR}}(T) = \sum_{a \in [K]} N_a^{\pi_1\text{-AR}}(T) = T.
    \end{align}
    Taking expectations on both sides of equation~\eqref{eq:Ne-2T-Nr} yields
    \begin{align}\label{eq:ne-equal-2T-nr}
        n^{\text{e-AR}}(T) = 2T - n^{\text{r-AR}}(T).
    \end{align}
    
    On the other hand, by applying the identity $\max\{x, y\} = \frac{1}{2}(x+y + \vert x-y \vert)$ to the definition of $N^{\text{e-AR}}(T)$, we derive
    \begin{align*}
        N^{\text{e-AR}}(T) &= \frac{1}{2} \left[\sum_{a \in [K]} N_a^{\pi_0\text{-AR}}(T) + \sum_{a \in [K]}  N_a^{\pi_1\text{-AR}}(T) + \sum_{a \in [K]} \vert  N_a^{\pi_0\text{-AR}}(T) - N_a^{\pi_1\text{-AR}}(T) \vert \right] \\
        &= T + \frac{1}{2} \sum_{a \in [K]} \vert  N_a^{\pi_0\text{-AR}}(T) - N_a^{\pi_1\text{-AR}}(T) \vert.
    \end{align*}
    Denote $N_2^{\text{e-AR}}(T) \triangleq \frac{1}{2} \sum_{a \in [K]} \vert  N_a^{\pi_0\text{-AR}}(T) - N_a^{\pi_1\text{-AR}}(T) \vert$,
    which corresponds to the number of interactions with the real environment during Phase 2 of Algorithm~\ref{alg:AR}.
    Using the fact \eqref{eq:both-sum-to-T} again, $N_2^{\text{e-AR}}(T)$ can be represented in terms of the suboptimal arms only:
    \begin{align*}
    N_2^{\text{e-AR}}(T) 
    &= \frac{1}{2} \sum_{a \neq a^*,\, a \in [K]} \vert N_a^{\pi_0\text{-AR}}(T) - N_a^{\pi_1\text{-AR}}(T) \vert + \frac{1}{2} \vert N_{a^*}^{\pi_0\text{-AR}}(T) - N_{a^*}^{\pi_1\text{-AR}}(T) \vert \\
    &= \frac{1}{2} \sum_{ a\neq a^*,\, a \in [K]} \vert N_a^{\pi_0\text{-AR}}(T) - N_a^{\pi_1\text{-AR}}(T) \vert + \frac{1}{2} \bigg\vert \bigg(T - \sum_{a \neq a^*,\, a \in [K]} N_{a}^{\pi_0\text{-AR}}(T)\bigg) \\
    &\hspace{0.4cm}- \bigg(T - \sum_{a \neq a^*,\, a \in [K]} N_{a}^{\pi_1\text{-AR}}(T)\bigg) \bigg\vert \\
    &= \frac{1}{2} \sum_{ a\neq a^*,\, a \in [K]} \vert N_a^{\pi_0\text{-AR}}(T) - N_a^{\pi_1\text{-AR}}(T) \vert + \frac{1}{2} \bigg\vert \sum_{a \neq a^*, a \in [K]} (N_{a}^{\pi_0\text{-AR}}(T) - N_{a}^{\pi_1\text{-AR}}(T)) \bigg\vert.
\end{align*}
Next, we bound the expected value of $N_2^{\text{e-AR}}(T)$ as follows:
\begin{align*}
    \E[N_2^{\text{e-AR}}(T)] &\leq \E\bigg[\sum_{ a\neq a^*,\, a \in [K]} \vert N_a^{\pi_0\text{-AR}}(T) - N_a^{\pi_1\text{-AR}}(T) \vert \bigg]\nonumber \\
    &\leq \sum_{ a\neq a^*,\, a \in [K]} \E[ N_a^{\pi_0\text{-AR}}(T)] + \sum_{ a\neq a^*,\, a \in [K]} \E[N_a^{\pi_1\text{-AR}}(T)].
\end{align*}
By Corollary~\ref{cor:mean-var-equal} and the definition of $n^{\pi_i}(T)$ for $i \in \{0, 1\}$, it follows that 
\begin{equation*}
    \E[N_2^{\text{e-AR}}(T)] \leq n^{\pi_0}(T) + n^{\pi_1}(T).
\end{equation*}
Hence, we have
\begin{align}\label{eq:ne-T-n0-n1}
    n^{\text{e-AR}}(T) = \E[N^{\text{e-AR}}(T)] = T + \E[N_2^{\text{e-AR}}(T)] \leq T + n^{\pi_0}(T) + n^{\pi_1}(T).
\end{align}
Combining \eqref{eq:ne-T-n0-n1} and \eqref{eq:ne-equal-2T-nr}, we conclude that
\begin{equation*}
    n^{\text{e-AR}}(T) \leq \min\{T + n^{\pi_0}(T) + n^{\pi_1}(T),\, 2T-n^{\text{r-AR}}(T)\}.
\end{equation*}


\subsection{Proof of Corollary~\ref{cor:moment-eq}}
\label{app:cor-moment-eq}
   The results (i) and (ii) of Proposition~\ref{prop:stopping-time-martingale} allow us to directly apply the martingale stopping theorem (see Theorem~6.2.2 in \cite{ross1995stochastic}).
   We first apply the theorem to the martingale process $\{S_t^{T, a},\, t \in \mathbb{N}\}$ and obtain
    \begin{equation}\label{eq:lem-pf-S-equal-0}
        \begin{aligned}
        \E[ S_{N_a^{\pi_i\text{-s}}(T)}^{T, a}] &= \E[S_0^{T, a}] = 0 \quad \text{ for } i \in \{0, 1\}, \\
        \E[ S_{N_a^{\pi_0\text{-s}}(T) \wedge N_a^{\pi_1\text{-s}}(T)}^{T, a}] &= \E[S_0^{T, a}]  = 0.
    \end{aligned}
    \end{equation}
    Similarly, applying the martingale stopping theorem to the martingale process $\{V_t^{T, a},\, t \in \mathbb{N}\}$ gives
    \begin{equation}\label{eq:V-right-hand}
        \begin{aligned}
            \E[ V_{N_a^{\pi_i\text{-s}}(T)}^{T, a}] &= \E[V_0^{T, a}] = \E[(S_0^{T, a})^2] = 0 \quad \text{ for } i \in \{0, 1\}, \\
            \E[ V_{N_a^{\pi_0\text{-s}}(T) \wedge N_a^{\pi_1\text{-s}}(T)}^{T, a}]  &= \E[V_0^{T, a}] = \E[(S_0^{T, a})^2] = 0 .
        \end{aligned}
    \end{equation}
    On the other hand, according to the definition of $V_t^{T, a}$ in \eqref{eq:def-I-S-V}, we have 
    \begin{equation}\label{eq:V-left-hand}
        \begin{aligned}
            \E[ V_{N_a^{\pi_i\text{-s}}(T)}^{T, a}] &= \E[(S_{ N_a^{\pi_i\text{-s}}(T)}^{T, a})^2] - \sigma_a^2 \E[N_a^{\pi_i\text{-s}}(T)] \quad \text{ for } i \in \{0, 1\}, \\
            \E[ V_{N_a^{\pi_0\text{-s}}(T) \wedge N_a^{\pi_1\text{-s}}(T)}^{T, a}]  &= \E[(S_{N_a^{\pi_0\text{-s}}(T) \wedge N_a^{\pi_1\text{-s}}(T)}^{T, a})^2] - \sigma_a^2 \E[N_a^{\pi_0\text{-s}}(T) \wedge N_a^{\pi_1\text{-s}}(T)].
        \end{aligned}
    \end{equation}
    Combining equations \eqref{eq:V-right-hand} and \eqref{eq:V-left-hand}, and using $\E[ S_{N_a^{\pi_i\text{-s}}(T)}^{T, a}] = 0$ and $\E[ S_{N_a^{\pi_0\text{-s}}(T) \wedge N_a^{\pi_1\text{-s}}(T)}^{T, a}] = 0$ established in \eqref{eq:lem-pf-S-equal-0}, we obtain
    \begin{align*}
        \Var(S_{N_a^{\pi_i\text{-s}}(T)}^{T, a}) = \E[(S_{N_a^{\pi_i\text{-s}}(T)}^{T, a})^2] - (\E[S_{N_a^{\pi_i\text{-s}}(T)}^{T, a}])^2 = \E[(S_{N_a^{\pi_i\text{-s}}(T)}^{T, a})^2] = \sigma_a^2 \E[N_a^{\pi_i\text{-s}}(T)] \quad \text{ for } i \in \{0, 1\},
    \end{align*}    
    and similarly
    \begin{align*}
        \Var( S_{N_a^{\pi_0\text{-s}}(T) \wedge N_a^{\pi_1\text{-s}}(T)}^{T, a}) &= \E[(S_{N_a^{\pi_0\text{-s}}(T) \wedge N_a^{\pi_1\text{-s}}(T)}^{T, a})^2] - (\E[S_{N_a^{\pi_0\text{-s}}(T) \wedge N_a^{\pi_1\text{-s}}(T)}^{T, a}])^2 \\ &= \E[(S_{N_a^{\pi_0\text{-s}}(T) \wedge N_a^{\pi_1\text{-s}}(T)}^{T, a})^2] \\
        &= \sigma_a^2 \E[N_a^{\pi_0\text{-s}}(T) \wedge N_a^{\pi_1\text{-s}}(T)].
    \end{align*}
    This concludes the proof.


\subsection{Proof of Theorem~\ref{thm:asy_var}}
\label{app:thm-asy_var}
Corollary~\ref{cor:mean-var-equal} and Condition~\eqref{eq:var-assump} implies that
\begin{equation}\label{eq:var-assump-stack}
    \E[N_a^{\pi_i\text{-s}}(T)] = o(T) \quad \text{and} \quad
    \Var(N_a^{\pi_i\text{-s}}(T)) = o(T).
\end{equation}
By Theorem~\ref{thm:AR-stack-same-dist}, we have that
\begin{align*}
    &(A_1^{\pi_0\text{-AR}}, R_1^{\pi_0\text{-AR}}, \dots, A_T^{\pi_0\text{-AR}}, R_T^{\pi_0\text{-AR}}, A_1^{\pi_1\text{-AR}}, R_1^{\pi_1\text{-AR}}, \dots, A_T^{\pi_1\text{-AR}}, R_T^{\pi_1\text{-AR}}) \\
    \stackrel{\text{d}}{=} 
    {}& (A_1^{\pi_0\text{-s}}, R_1^{\pi_0\text{-s}}, \dots, A_T^{\pi_0\text{-s}}, R_T^{\pi_0\text{-s}}, A_1^{\pi_1\text{-s}}, R_1^{\pi_1\text{-s}}, \dots, A_T^{\pi_1\text{-s}}, R_T^{\pi_1\text{-s}}). 
\end{align*}
Therefore, the following marginal distributions are identical:
\begin{align*}
    &(R_1^{\pi_0\text{-AR}}, R_2^{\pi_0\text{-AR}}, \dots, R_T^{\pi_0\text{-AR}}, R_1^{\pi_1\text{-AR}}, R_2^{\pi_1\text{-AR}}, \dots, R_T^{\pi_1\text{-AR}}) \\
    \stackrel{\text{d}}{=} 
    {}& (R_1^{\pi_0\text{-s}}, R_2^{\pi_0\text{-s}}, \dots, R_T^{\pi_0\text{-s}}, R_1^{\pi_1\text{-s}}, R_2^{\pi_1\text{-s}}, \dots, R_T^{\pi_1\text{-s}}). 
\end{align*}
It then follows that for any $i,\, j \in \{0, 1\}$, 
\begin{align}\label{eq:pf-AR-stack-cov-equal}
    \operatorname{Cov}\bigg(\sum_{t=1}^{T} R_t^{\pi_i\text{-AR}},\,\sum_{t=1}^{T} R_t^{\pi_j\text{-AR}}\bigg) = \operatorname{Cov}\bigg(\sum_{t=1}^{T} R_t^{\pi_i\text{-s}},\,\sum_{t=1}^{T} R_t^{\pi_j\text{-s}}\bigg) .
\end{align}
We denote $C_{ij}(T) = \operatorname{Cov}(\sum_{t=1}^{T} R_t^{\pi_i\text{-s}},\,\sum_{t=1}^{T} R_t^{\pi_j\text{-s}})$.
By expressing the cumulative reward rewards $\sum_{t=1}^{T} R_t^{\pi_i\text{-s}}$ and $\sum_{t=1}^{T} R_t^{\pi_j\text{-s}}$ as the sum of the arm-specific rewards and rearranging the terms and then rearranging the terms, we can write $C_{ij}$ as
\begin{equation}\label{eq:C-cov-parts}
\begin{aligned}
    C_{ij}(T) &= \operatorname{Cov}\bigg(\sum_{a=1}^K \sum_{t=1}^{N_a^{\pi_i\text{-s}}(T)} X_{a,t},\,\sum_{a=1}^K \sum_{t=1}^{N_a^{\pi_j\text{-s}}(T)} X_{a,t}\bigg)\\
    &= \sum_{a=1}^K \sum_{b=1}^K \operatorname{Cov}( S_{N_a^{\pi_i\text{-s}}(T)}^{T, a} ,\,S_{N_b^{\pi_j\text{-s}}(T)}^{T, b})  + \sum_{a=1}^K \sum_{b=1}^K \mu_b \operatorname{Cov}(N_a^{\pi_i\text{-s}}(T),\,N_b^{\pi_j\text{-s}}(T)) \\
    &\hspace{1em} + \sum_{a=1}^K \sum_{b=1}^K \mu_a \operatorname{Cov}( N_a^{\pi_i\text{-s}}(T),\,S_{N_b^{\pi_j\text{-s}}(T)}^{T, b}) + \sum_{a=1}^K \sum_{b=1}^K \mu_b \operatorname{Cov}(S_{N_a^{\pi_i\text{-s}}(T)}^{T, a},\,N_b^{\pi_j\text{-s}}(T)),
\end{aligned}
\end{equation}
where $S_t^{T, a}$ is defined as in \eqref{eq:def-I-S-V}.

We aim to prove that $\lim_{T \rightarrow \infty} \frac{1}{T} C_{ij} (T) = \sigma_{a^*}^2$. To this end, we will establish the following three statements:
\begin{itemize}
    \item [(i)] For all $a,\, b \in [K]$ with $a \neq a^*$ or $b \neq a^*$, $\operatorname{Cov}( S_{N_a^{\pi_i\text{-s}}(T)}^{T, a} ,\,S_{N_b^{\pi_j\text{-s}}(T)}^{T, b}) = o(T)$. Moreover, $\lim_{T \rightarrow \infty }\frac{1}{T}\operatorname{Cov}(S_{N_{a^*}^{\pi_i\text{-s}}(T)}^{T, a^*} ,\,S_{N_{a^*}^{\pi_j\text{-s}}(T)}^{T, a^*}) = \sigma_{a^*}^2$.
    \item [(ii)] For all $a,\, b \in [K]$, $\operatorname{Cov}(N_a^{\pi_i\text{-s}}(T),\,N_b^{\pi_j\text{-s}}(T)) = o(T)$.
    \item [(iii)] For all $a,\, b \in [K]$,  $\operatorname{Cov}( N_a^{\pi_i\text{-s}}(T),\,S_{N_b^{\pi_j\text{-s}}(T)}^{T, b}) = o(T)$ and $\operatorname{Cov}(S_{ N_a^{\pi_i\text{-s}}(T)}^{T, a},\,N_b^{\pi_j\text{-s}}(T)) = o(T)$.
\end{itemize}
Before proving the above statements (i)-(iii) individually, we first show that $\Var(N_{a^*}^{\pi_i\text{-s}}(T)) = o(T)$ for $i \in \{0, 1\}$. Indeed, for the optimal arm $a^*$, using the identity
$N_{a^*}^{\pi_i\text{-s}}(T) = T - \sum_{a \neq a^*} N_a^{\pi_i\text{-s}}(T)$,
we obtain
\begin{align}\label{eq:pf-var-N-a^*}
    \Var(N_{a^*}^{\pi_i\text{-s}}(T)) =  \Var\bigg(T - \sum_{a\neq a^*} N_a^{\pi_i\text{-s}}(T)\bigg) = \Var\bigg(\sum_{a\neq a^*} N_a^{\pi_i\text{-s}}(T)\bigg) \leq (K-1) \sum_{a\neq a^*}\Var(N_{a}^{\pi_i\text{-s}}(T)),
\end{align}
where the last inequality follows from the Cauchy–Schwarz inequality on real numbers.
By \eqref{eq:var-assump-stack}, we have $\Var(N_a^{\pi_i\text{-s}}(T)) = o(T)$; hence, \eqref{eq:pf-var-N-a^*} implies that $\Var(N_{a^*}^{\pi_i\text{-s}}(T)) = o(T)$ for $i \in \{0, 1\}$. We next proceed to prove statements (i)-(iii). 
\begin{itemize}
    \item [(i)] 
    For any $a,\, b \in [K]$ with $a \neq a^*$ or $b \neq a^*$, assume without loss of generality that arm $a$ is suboptimal. By the Cauchy–Schwarz inequality together with Corollary~\ref{cor:moment-eq}, we have
    \begin{align}
        \vert\operatorname{Cov}( S_{N_a^{\pi_i\text{-s}}(T)}^{T, a},\, S_{N_b^{\pi_j\text{-s}}(T)}^{T, b} )\vert 
        &\leq \sqrt{\Var(S_{N_a^{\pi_i\text{-s}}(T)}^{T, a}) \Var(S_{ N_b^{\pi_j\text{-s}}(T)}^{T, b})} \nonumber\\
        & = \sqrt{\sigma_a^2 \sigma_b^2 \E[N_a^{\pi_i\text{-s}}(T)]\E[N_b^{\pi_j\text{-s}}(T)]} \nonumber \\
        & \leq \sqrt{\sigma_a^2 \sigma_b^2 T \E[N_a^{\pi_i\text{-s}}(T)]}\label{eq:pf-cov-S-S}.
    \end{align}
    Since by \eqref{eq:var-assump-stack}, we have $\E[N_a^{\pi_i\text{-s}}(T)] = o(T)$, then \eqref{eq:pf-cov-S-S} implies that  $\operatorname{Cov}( S_{N_a^{\pi_i\text{-s}}(T)}^{T, a},\, S_{ N_b^{\pi_j\text{-s}}(T)}^{T, b} ) = o(T)$.
    
    Moreover, by Corollary~\ref{cor:moment-eq}, we have $\E[ S_{N_{a^*}^{\pi_i\text{-s}}(T)}^{T, a^*}] = 0$ for $i \in \{0, 1\}$.
    Hence, according to the definition of covariance,
    \begin{align*}
        \operatorname{Cov}( S_{N_{a^*}^{\pi_i\text{-s}}(T)}^{T, a^*},\, S_{N_{a^*}^{\pi_j\text{-s}}(T)}^{T, a^*}) &= \E[ S_{N_{a^*}^{\pi_i\text{-s}}(T)}^{T, a^*} S_{N_{a^*}^{\pi_j\text{-s}}(T)}^{T, a^*}] - \E[ S_{N_{a^*}^{\pi_i\text{-s}}(T)}^{T, a^*}] \E[S_{N_{a^*}^{\pi_j\text{-s}}(T)}^{T, a^*}] \\
        &= \E[ S_{N_{a^*}^{\pi_i\text{-s}}(T)}^{T, a^*} S_{N_{a^*}^{\pi_j\text{-s}}(T)}^{T, a^*}].
    \end{align*}
    Note that $S_{N_{a^*}^{\pi_i\text{-s}}(T)}^{T, a^*}$ is $\mathcal{F}_{N_{a^*}^{\pi_i\text{-s}}(T)}^{T, a^*}$-measurable.
    Taking the conditional expectation with respect to
    $\mathcal{F}_{N_{a^*}^{\pi_i\text{-s}}(T)}^{T, a^*}$ and applying the Doob's optional sampling theorem (see Theorem~6.1.2 in \cite{kallenberg1997foundations}) to the martingale $\{S_t^{T, a^*},\, t \in \mathbb{N}\}$, we obtain
    \begin{align*}
        \operatorname{Cov}( S_{N_{a^*}^{\pi_i\text{-s}}(T)}^{T, a^*},\, S_{N_{a^*}^{\pi_j\text{-s}}(T)}^{T, a^*}) &= \E\big[ \E[ S_{N_{a^*}^{\pi_i\text{-s}}(T)}^{T, a^*} S_{N_{a^*}^{\pi_j\text{-s}}(T)}^{T, a^*} \mid \mathcal{F}_{N_{a^*}^{\pi_i\text{-s}}(T)}^{T, a^*}]\big] \\
        &= \E[ S_{N_{a^*}^{\pi_i\text{-s}}(T)}^{T, a^*} S_{N_{a^*}^{\pi_i\text{-s}}(T) \wedge N_{a^*}^{\pi_j\text{-s}}(T)}^{T, a^*} ].
    \end{align*}
    Since $S_{N_{a^*}^{\pi_i\text{-s}}(T)\wedge N_{a^*}^{\pi_j\text{-s}}(T)}^{T, a^*}$
    is $\mathcal{F}_{N_{a^*}^{\pi_i\text{-s}}(T)\wedge N_{a^*}^{\pi_j\text{-s}}(T)}^{T, a^*}$-measurable,
    taking the conditional expectation with respect to
    $\mathcal{F}_{N_{a^*}^{\pi_i\text{-s}}(T)\wedge N_{a^*}^{\pi_j\text{-s}}(T)}^{T, a^*}$
    and applying the Doob's optional sampling theorem once more yields
    \begin{align*}
        \operatorname{Cov}( S_{N_{a^*}^{\pi_i\text{-s}}(T)}^{T, a^*},\, S_{N_{a^*}^{\pi_j\text{-s}}(T)}^{T, a^*})  &= \E\big[ \E[ S_{N_{a^*}^{\pi_i\text{-s}}(T)}^{T, a^*} S_{N_{a^*}^{\pi_i\text{-s}}(T) \wedge N_{a^*}^{\pi_j\text{-s}}(T)}^{T, a^*} \mid \mathcal{F}_{N_{a^*}^{\pi_i\text{-s}}(T) \wedge N_{a^*}^{\pi_j\text{-s}}(T)}^{T, a^*}]\big] \\
        &= \E[(S_{N_{a^*}^{\pi_i\text{-s}}(T) \wedge N_{a^*}^{\pi_j\text{-s}}(T)}^{T, a^*})^2].
    \end{align*}
    Then, it follows from Corollary~\ref{cor:moment-eq} that
    \begin{align}
        \operatorname{Cov}( S_{N_{a^*}^{\pi_i\text{-s}}(T)}^{T, a^*},\, S_{N_{a^*}^{\pi_j\text{-s}}(T)}^{T, a^*}) &= \E[(S_{N_{a^*}^{\pi_i\text{-s}}(T) \wedge N_{a^*}^{\pi_j\text{-s}}(T)}^{T, a^*})^2] \nonumber\\
        &= \Var(S_{N_{a^*}^{\pi_i\text{-s}}(T) \wedge N_{a^*}^{\pi_j\text{-s}}(T)}^{T, a^*}) \nonumber\\
        &= \sigma_{a^*}^2 \E[N_{a^*}^{\pi_i\text{-s}}(T) \wedge N_{a^*}^{\pi_j\text{-s}}(T)] \label{eq:pf-cov-S-S-a^*}.
    \end{align}
    Note that
    \begin{align}\label{eq:pf-squeeze}
       \E[N_{a^*}^{\pi_i\text{-s}}(T)] + \E[N_{a^*}^{\pi_j\text{-s}}(T)] - T \leq \E[N_{a^*}^{\pi_i\text{-s}}(T) \wedge N_{a^*}^{\pi_j\text{-s}}(T)] \leq T.
    \end{align}
    Since $\E[N_{a}^{\pi_i\text{-s}}(T)] = o(T)$ for every suboptimal arm $a$ and $i\in \{0, 1\}$, it follows that $\E[N_{a^*}^{\pi_i\text{-s}}(T)] = T - o(T)$ for $i\in \{0, 1\}$, then we have 
    \begin{equation}\label{eq:pf-LHS-o}
        \E[N_{a^*}^{\pi_i\text{-s}}(T)] + \E[N_{a^*}^{\pi_j\text{-s}}(T)] - T = T- o(T).
    \end{equation}
    Hence, combining equations \eqref{eq:pf-cov-S-S-a^*}, \eqref{eq:pf-squeeze} and \eqref{eq:pf-LHS-o}, we conclude that
    \begin{equation*}
        \lim_{T\rightarrow \infty} \frac{1}{T} \operatorname{Cov}( S_{N_{a^*}^{\pi_i\text{-s}}(T)}^{T, a^*},\, S_{N_{a^*}^{\pi_j\text{-s}}(T)}^{T, a^*}) = \sigma_{a^*}^2 \quad \text{ for } i,\, j \in \{0, 1\}.
    \end{equation*}
    \item [(ii)]
    For any $a, \, b \in [K]$, 
    the Cauchy–Schwarz inequality implies that
    \begin{align*}
        \vert \operatorname{Cov}(N_a^{\pi_i\text{-s}}(T), \, N_b^{\pi_j\text{-s}}(T)) \vert \leq \sqrt{\Var(N_a^{\pi_i\text{-s}}(T)) \Var(N_b^{\pi_j\text{-s}}(T))}. 
    \end{align*}
    Since $\Var(N_a^{\pi_i\text{-s}}(T)) = o(T)$ for all $a \in [K]$ and $i \in \{0, 1\}$, it follows that $ \operatorname{Cov}(N_a^{\pi_i\text{-s}}(T), \, N_b^{\pi_j\text{-s}}(T)) = o(T)$ for $i,\, j \in \{0, 1\}$.
    \item [(iii)]
    For any $a, \, b \in [K]$, 
    by the Cauchy–Schwarz inequality together with Corollary~\ref{cor:moment-eq}, we obtain
    \begin{align*}
        \vert \operatorname{Cov}( N_a^{\pi_i\text{-s}}(T),\, S_{ N_b^{\pi_j\text{-s}}(T)}^{T, b}) \vert &\leq \sqrt{\Var(N_a^{\pi_i\text{-s}}(T)) \Var(S_{N_b^{\pi_j\text{-s}}(T)}^{T, b})} \\
        &= \sqrt{\sigma_b^2 \Var(N_a^{\pi_i\text{-s}}(T))\E[N_b^{\pi_j\text{-s}}(T)]} \\
        &\leq \sqrt{\sigma_b^2 T \Var(N_a^{\pi_i\text{-s}}(T))}.
    \end{align*}
    Given that $\Var(N_{a}^{\pi_i\text{-s}}(T)) = o(T)$ for each arm $a \in [K]$ and $i\in \{0, 1\}$, we conclude that $\operatorname{Cov}( N_a^{\pi_i\text{-s}}(T),\, S_{N_b^{\pi_j\text{-s}}(T)}^{T, b}) = o(T)$.
    Similarly, we can derive 
    \begin{align*}
        \vert \operatorname{Cov}( S_{N_a^{\pi_i\text{-s}}(T)}^{T, a}, \,N_b^{\pi_j\text{-s}}(T)) \vert \leq \sqrt{\sigma_a^2 T \Var(N_b^{\pi_j\text{-s}}(T))},
    \end{align*}
    which implies that $\operatorname{Cov}( S_{N_a^{\pi_i\text{-s}}(T)}^{T, a}, \,N_b^{\pi_j\text{-s}}(T)) = o(T)$ for $i,\, j \in \{0, 1\}$.
\end{itemize}
Having established statements (i)-(iii) above, we now combine them with equation \eqref{eq:C-cov-parts}, from which we immediately obtain
\begin{equation}\label{eq:pf-lim-C-ij}
    \lim_{T \rightarrow \infty} \frac{1}{T} C_{ij}(T) = \sigma_{a^*}^2 \quad \text{ for } i,\, j \in \{0, 1\}. 
\end{equation}
For the na\"{\i}ve estimator $\thetahat{b}(T)$, since $\Var(\thetahat{b}(T)) = \Var(\sum_{t=1}^T R_t^{\pi_0}) + \Var(\sum_{t=1}^T R_t^{\pi_1})$, it follows from Corollary~\ref{cor:mean-var-equal} that
\begin{align*}
    \Var(\thetahat{b}(T)) = \Var\bigg(\sum_{t=1}^T R_t^{\pi_0\text{-s}}\bigg) + \Var\bigg(\sum_{t=1}^T R_t^{\pi_1\text{-s}}\bigg) = C_{00}(T) + C_{11}(T).
\end{align*}
Therefore, by \eqref{eq:pf-lim-C-ij}, we have $\lim_{T\rightarrow \infty} \frac{1}{T}\Var(\thetahat{b}(T)) = 2\sigma_{a^*}^2$.
In contrast, for the proposed estimator, note that
$\Var(\thetahat{AR}(T)) = \Var(\sum_{t=1}^T R_t^{\pi_0\text{-AR}}) + \Var(\sum_{t=1}^T R_t^{\pi_1\text{-AR}}) - 2\operatorname{Cov}(\sum_{t=1}^T R_t^{\pi_0\text{-AR}}, \sum_{t=1}^T R_t^{\pi_1\text{-AR}})$.
By \eqref{eq:pf-AR-stack-cov-equal}, it follows that
\begin{align*}
    \Var(\thetahat{AR}(T)) &= \Var\bigg(\sum_{t=1}^T R_t^{\pi_0\text{-s}}\bigg) + \Var\bigg(\sum_{t=1}^T R_t^{\pi_1\text{-s}}\bigg) - 2\operatorname{Cov}\bigg(\sum_{t=1}^T R_t^{\pi_0\text{-s}}, \sum_{t=1}^T R_t^{\pi_1\text{-s}}\bigg) \\
    &= C_{00}(T) + C_{11}(T) - 2C_{01}(T).
\end{align*}
Applying $\eqref{eq:pf-lim-C-ij}$ yields 
$\lim_{T\rightarrow \infty} \frac{1}{T}\Var(\thetahat{AR}(T)) = 0$. This completes the proof.

\subsection{Proof of Proposition~\ref{prop:UCB-validity}}
\label{app:prop-UCB-validity}
\begin{itemize}

\item [(i)]
Fix an arbitrary suboptimal arm $a \in [K]\setminus \{a^*\}$ and denote $u_a(T) \triangleq \lceil \frac{4\alpha \log T}{\Delta_a^2} \rceil$ and event $D_t \triangleq \{A_t^{\pi}=a,\, N_a^{\pi}(t-1) \geq u_a(T)\}$ for each $t \in [T]$.
Following a similar argument as in the proof of Theorem~1 in \cite{auer2002finite}, except that the confidence term $c_{t,s} \triangleq \sqrt{(2 \log t) / s}$ is modified to $c_{t,s} \triangleq \sqrt{(\alpha \log t) / s}$, we can obtain  
\begin{align}\label{eq:P-N-leq}
    \Prob(D_t) \leq 2t^{2-2\alpha}.
\end{align}
Noting that the arm index $a$ in our notation corresponds to $i$ in theirs, 
the horizon $T$ to $n$, 
the action $A_t^{\pi}$ to $I_t$, 
and the number of pulls of arm $a$, $N_a^{\pi}(T)$, 
corresponds to $T_i(n)$ in their notation.

{Next, we can decompose $N_a^{\pi}(T)$ as
\begin{align}
N_a^{\pi}(T) &= \sum_{t=1}^T \mathbb{I}\{A_t^{\pi}=a,\, N_a^{\pi}(t-1) < u_a(T)\} + \sum_{t=1}^T \mathbb{I}\{A_t^{\pi}=a,\, N_a^{\pi}(t-1) \geq u_a(T)\} \nonumber \\
&\leq u_a(T) + \sum_{t=u_a(T)+1}^T \mathbb{I}\{D_t\}.
\label{eq:pf-N-a-leq-u-B}
\end{align}
We now bound the second moment $\E[(N_a^{\pi}(T))^2]$. Squaring \eqref{eq:pf-N-a-leq-u-B} and using $(x+y)^2 \leq 2x^2+2y^2$ yields $(N_a^{\pi}(T))^2 \leq 2(u_a(T))^2 + 2\big(\sum_{t=u_a(T)+1}^T \mathbb{I}\{D_t\}\big)^2$.
By taking expectations on both sides, we obtain
\begin{align}
    \E[(N_a^{\pi}(T))^2] &\leq 2 (u_a(T))^2 + 2 \sum_{t=u_a(T)+1}^T \Prob(D_t) + 4 \sum_{t=u_a(T)+1}^{T-1} \sum_{s=t+1}^T \Prob(D_t \cap D_s) \nonumber \\
    &\leq 2 (u_a(T))^2 + 4 \sum_{t=u_a(T)+1}^T t^{2-2\alpha} + 8 \sum_{t=u_a(T)+1}^{T-1} \sum_{s=t+1}^T s^{2-2\alpha}
    \label{eq:pf-E-N-sq}
\end{align}
where the last inequality follows from \eqref{eq:P-N-leq} and the fact  $\Prob(D_t \cap D_s) \leq \Prob(D_s)$.
Since the tail sum of a decreasing function 
can be bounded by the corresponding definite integral, and noting that $\alpha \geq 2$, we have
\begin{align}
    \sum_{t=u_a(T)+1}^T  t^{2-2\alpha}
    \leq \frac{1}{2 \alpha - 3}[(u_a(T))^{3-2\alpha} - T^{3-2\alpha}] \leq \frac{1}{u_a(T)}.
    \label{eq:pf-sum-u-T-alpha}
\end{align}
Similarly, we have $\sum_{s=t+1}^T  s^{2-2\alpha} \leq \frac{1}{t}$. It follows that
\begin{align}\label{eq:pf-sum-u-T-harm}
    \sum_{t=u_a(T)+1}^{T-1} \sum_{s=t+1}^T s^{2-2\alpha} \leq \sum_{t=u_a(T)+1}^{T-1} \frac{1}{t} & \leq \log(T-1) - \log(u_a(T)) \leq \log(T).
\end{align}
Substituting the bounds from \eqref{eq:pf-sum-u-T-alpha} and \eqref{eq:pf-sum-u-T-harm} into \eqref{eq:pf-E-N-sq}, we obtain
\begin{align*}
    \E[(N_a^{\pi}(T))^2] \leq 2 (u_a(T))^2 + \frac{4}{u_a(T)} + 8 \log(T).
\end{align*}
Since $u_a(T) = \lceil \frac{4\alpha \log T}{\Delta_a^2} \rceil$ with $\alpha \geq 2$, we have $\E[(N_a^{\pi}(T))^2] = O((\log T)^2)$. 
It then follows from $\Var(N_a^{\pi}(T)) \leq \E[ (N_a^{\pi}(T))^2 ]$ that, for UCB1 policy with $\alpha \geq 2$, $\Var(N_a^{\pi}(T)) = O((\log T)^2)$ holds.}

\item [(ii)]
{For any arm $a \in [K]$ and any integer $n \in [T]$, let $\hat{\mu}_{a,n} \triangleq \frac{1}{n} \sum_{k=1}^{n} X_{a,k}$ be the empirical mean of arm $a$ after $n$ pulls, and $X_{a,k}$ is the reward from the $k$-th pull. 
Let $c_{\delta, n} \triangleq \sqrt{2 \log (1 / \delta) / n}$ for $n \in [T]$. Without loss of generality, assume that arm 1 is optimal, i.e. $\mu_1 = \mu^*$.
Next, we fix an arbitrary suboptimal arm $a \in \{2, \ldots, K\}$ and denote $\theta_a \triangleq \frac{1}{2}(\mu_1 + \mu_a)$.
Note that for any integer $u \in [T]$, if
\begin{align*}
    \min_{n \in [T]} \hat{\mu}_{1,n} + c_{\delta, n} > \theta_a \quad \text{and} \quad \max_{n \in \{u, \ldots, T\}} \hat{\mu}_{a,n} + c_{\delta, n} < \theta_a,
\end{align*}
then arm $a$ would not be pulled at any step $t \in \{u,\ldots,T\}$ and hence $N_a^{\pi}(T) < u$.
Denote the events $E_1 \triangleq \{\min_{n \in [T]} \hat{\mu}_{1,n} + c_{\delta, n} \leq \theta_a\}$ and $E_a(u) \triangleq \{\max_{n \in \{u, \ldots, T\}} \hat{\mu}_{a,n} + c_{\delta, n} \geq \theta_a\}$.
It follows that for any integer $u \in [T]$,
\begin{align}\label{eq:pf-N-subset-E1-Eu}
    \{N_a^{\pi}(T) \geq u\} \subseteq E_1 \cup E_a(u).
\end{align}}

{We proceed to bound the probabilities of $E_1$ and $E_a(u)$. 
For $E_1$, 
\begin{align*}
    \Prob(E_1) \leq \sum_{n=1}^T \Prob(\hat{\mu}_{1,n} + c_{\delta, n} \leq \theta_a) \leq \sum_{n=1}^T \Prob\left(\hat{\mu}_{1,n} - \mu_1 \leq - \frac{\Delta_a}{2} - c_{\delta,n}\right).
\end{align*}
where $\Delta_a \triangleq \mu_1 - \mu_a$.
We further apply the Cram\'er--Chernoff method for $1$-sub-Gaussian variables (see \citet[Corollary 5.5]{lattimore2020bandit}) to obtain
\begin{align}
    \Prob(E_1) \leq \sum_{n=1}^T \exp\left\{-\frac{n}{2}\left(\frac{\Delta_a}{2} + c_{\delta, n}\right)^2\right\} \leq \delta \cdot \sum_{n=1}^T e^{- \Delta_a^2 n / 8} \leq \frac{\delta}{1 - e^{- \Delta_a^2 / 8}},
    \label{eq:pf-P-E1}
\end{align}
where the last inequality follows from the formula for a finite geometric series.
Next, consider $E_a(u)$. Define $u_a(\delta) \triangleq \lceil \frac{32 \log(1/\delta)}{\Delta_a^2} \rceil$.
Then, it holds that $c_{\delta, n} \leq \frac{\Delta_a}{4}$ for all $n > u_a(\delta
)$.
For any integer $u > u_a(\delta
)$, we derive
\begin{align*}
    \Prob(E_a(u)) &\leq
    \sum_{n=u}^T \Prob\left(\hat{\mu}_{a,n} - \mu_a > \frac{\Delta_a}{2} - c_{\delta,n}\right) \\
    &\leq
    \sum_{n=u}^T \Prob\left(\hat{\mu}_{a,n} - \mu_a > \frac{\Delta_a}{4}\right).
\end{align*}
Invoking the Cram\'er--Chernoff method and the formula for a finite geometric series again, we obtain that for any integer $u > u_a(\delta
)$,
\begin{align}
    \Prob(E_a(u)) \leq \sum_{n=u}^T e^{- \Delta_a^2 n / 32} \leq \frac{e^{- \Delta_a^2 u / 32}}{1 - e^{- \Delta_a^2 / 32}}.
    \label{eq:pf-P-Eu}
\end{align}
Combining \eqref{eq:pf-N-subset-E1-Eu}, \eqref{eq:pf-P-E1}, and \eqref{eq:pf-P-Eu} yields that for any integer $u > u_a(\delta
)$,
\begin{align}
    \Prob(N_a^{\pi}(T) \geq u) \leq \frac{\delta}{1 - e^{- \Delta_a^2 / 8}} + \frac{e^{- \Delta_a^2 u / 32}}{1 - e^{- \Delta_a^2 / 32}}.
    \label{eq:pf-UCB-delta-tail-prob}
\end{align}
Using the tail integral formula for $\E[(N_a^{\pi}(T))^2]$, we obtain
\begin{align}
    \E[(N_a^{\pi}(T))^2] &= \sum_{n=1}^T (2n-1) \Prob(N_a^{\pi}(T) \geq n) \nonumber \\ 
    &= \sum_{n=1}^{u_a(\delta)} (2n-1) \Prob(N_a^{\pi}(T) \geq n)  + \sum_{n=u_a(\delta) + 1}^T (2n-1) \Prob(N_a^{\pi}(T) \geq n).
    \label{eq:pf-N-square-leq}
\end{align}
For the first term in \eqref{eq:pf-N-square-leq},
\begin{align}
    \sum_{n=1}^{u_a(\delta)} (2n-1) \Prob(N_a^{\pi}(T) \geq n) \leq \sum_{n=1}^{u_a(\delta)} (2n-1) = (u_a(\delta))^2.
    \label{eq:pf-N-square-first-term}
\end{align}
For the second term in \eqref{eq:pf-N-square-leq}, it follows from \eqref{eq:pf-UCB-delta-tail-prob} that
\begin{align}
    \sum_{n=u_a(\delta) + 1}^T (2n-1) \Prob(N_a^{\pi}(T) \geq n) &\leq \frac{T^2 \delta}{1 - e^{- \Delta_a^2 / 8}} + \frac{2}{1 - e^{- \Delta_a^2 / 32}} \cdot \sum_{n=1}^{\infty} n e^{- \Delta_a^2 n / 32} \nonumber \\
    & \leq \frac{T^2 \delta}{1 - e^{- \Delta_a^2 / 8}} + \frac{2 e^{- \Delta_a^2 / 32}}{(1 - e^{- \Delta_a^2 / 32})^3},
    \label{eq:pf-N-square-second-term}
\end{align}
where the last inequality by applying the formula $\sum_{n=1}^{\infty} n x^n = \frac{x}{(1-x)^2}$ with $x = e^{- \Delta_a^2 / 32}$. 
Combining \eqref{eq:pf-N-square-leq}, \eqref{eq:pf-N-square-first-term}, and \eqref{eq:pf-N-square-second-term}, we have
\begin{align*}
    \E[(N_a^{\pi}(T))^2] \leq (u_a(\delta))^2 + \frac{T^2 \delta}{1 - e^{- \Delta_a^2 / 8}} + \frac{2 e^{- \Delta_a^2 / 32}}{(1 - e^{- \Delta_a^2 / 32})^3}.
\end{align*}
Since $u_a(\delta) = \lceil \frac{32 \log(1/\delta)}{\Delta_a^2} \rceil$ and $\delta = T^{-d}$ with $d \geq 2$, it holds that $(u_a(T))^2 = O((\log T)^2)$ and $T^2\delta = O(1)$.
Hence, $\E[(N_a^{\pi}(T))^2] = O((\log T)^2)$. It then follows from  
$\Var(N_a^{\pi}(T)) \leq \E[ (N_a^{\pi}(T))^2 ]$ that $\Var(N_a^{\pi}(T)) = O((\log T)^2)$.}
\end{itemize}

\end{document}